%% file: main.tex
\documentclass{article}

\newcommand{\companionsentence}{ Such score-trained heads compose across
observables with different likelihoods on a shared
backbone~\citep{Prochazka2026multi}.}
\input{preamble}

\input{tables/objective_prices}   

\title{Beliefs Beyond Posteriors:\\ The Predictive Cost of ELBO Binding}

\author{%
  Pavel Proch\'azka\\
  Cisco Inc.\\
  \texttt{paprocha@cisco.com}
}

\begin{document}

\maketitle

\input{sections/abstract}

\input{sections/intro}

\input{sections/objective}

\input{sections/binding}

\input{sections/method}

\input{sections/experiments}

\input{sections/related_work}

\input{sections/conclusion}

\clearpage
\bibliographystyle{iclr2027_conference}
\bibliography{refs}

\clearpage
\appendix
\input{tables/reading_guide}   

\input{sections/app_background}

\input{sections/app_statements}

\input{sections/app_proofs}

\input{sections/app_setup}

\input{sections/app_screen}

\FloatBarrier   
\input{sections/app_experiments}

\input{sections/app_beyond_gaussian}

\input{sections/app_beyond_last_layer}

\end{document}

%% file: preamble.tex
\usepackage[preprint]{neurips_2026}
\providecommand{\companionsentence}{}

\usepackage{xspace}
\newcommand{\FSC}{$F_\text{SC}$\xspace}   
\usepackage{placeins}   

\usepackage[utf8]{inputenc}
\usepackage[T1]{fontenc}
\usepackage{hyperref}
\usepackage{url}
\usepackage{booktabs}
\usepackage{array}
\usepackage{multirow}   
\usepackage{amsfonts}
\usepackage{dsfont}   
\usepackage{nicefrac}
\usepackage{microtype}
\usepackage{xcolor}
\usepackage{amsmath,amssymb,amsthm}
\usepackage{graphicx}
\usepackage{subcaption}
\usepackage{enumitem}
\usepackage[table]{xcolor}   
\usepackage{algpseudocode}
\usepackage{tikz}
\usetikzlibrary{positioning,calc,arrows.meta}
\usepackage{pgfplots}
\pgfplotsset{compat=1.18}
\usepgfplotslibrary{fillbetween,groupplots}

\newtheorem{theorem}{Theorem}
\newtheorem{definition}{Definition}
\newtheorem{proposition}{Proposition}
\newtheorem{lemma}{Lemma}
\newtheorem{corollary}{Corollary}
\newtheorem{remark}{Remark}

\newcounter{algorithm}

%% file: tables/objective_prices.tex
\newcommand{\priceCavity}{$0.01$--$0.05$}

%% file: sections/abstract.tex
\begin{abstract}
Posterior-seeking training need not produce the best predictive
distribution in its family under model misspecification. For a
last-layer weight belief, we locate the restriction responsible for
this gap and quantify its cost. At the last layer the negative ELBO
decomposes exactly into the belief's predictive log scores, the prior
KL, and a \emph{binding}: a sum of per-observation Jensen gaps that
constrains predictive dispersion through posterior covariance.
Removing the binding gives the established KL-regularised
form of predictive-score training; a parameterisation-dependent
factor-normaliser form performs similarly at the last layer. At
fixed predictive mean, the leading restricted direction is the variance
at each input, and its local value is the classical variance-component
score evaluated at the plug-in predictor. This yields both a
diagnostic and a price: in the flat-leverage limit, the Gaussian cost
is the Jensen gap of a representable residual-variance profile, whereas a
likelihood without a dispersion parameter pays the full expressible
improvement over the plug-in. Controlled experiments find
typically tenths of a nat for Gaussian regression but units to
tens of nats for overdispersed Poisson counts. In an exploratory, prospectively specified validation on 28
count sets, 27 unused elsewhere, a screen
computed from the point-prediction head alone agreed
with the measured outcome on 24 of 25 decided sets (an always-open rule: 18), and its price
proxy, once fixed, was within a factor of three of the predictive-score
gain of the fitted free profile on all 10 open sets tested afterwards.
Score training
changes the inferential target: it yields classical per-input
overdispersed heads whose covariance represents residual dispersion,
not posterior or epistemic uncertainty.
\end{abstract}

%% file: sections/intro.tex
\section{Introduction}
\label{sec:intro}

Bayesian deep learning trains a belief over weights to stand in for a
posterior, by the ELBO~\citep{Blundell2015}, by evidence maximisation,
or by a post-hoc Laplace fit~\citep{Daxberger2021}. That the ELBO
targets the evidence rather than the predictive score is known in
principle: under misspecification the posterior predictive is not the
best predictive in its family~\citep{Masegosa2020,Morningstar2022},
and training a belief's predictive directly is established at sparse
GPs~\citep{Sheth2020,Jankowiak2020}. Beyond that, what is missing is
\emph{which direction} a posterior-seeking belief cannot freely use,
\emph{how much} that costs, \emph{whether} a given problem will pay
it, and \emph{what} is trained when the restriction is lifted.

We answer by taking the objectives apart factor by factor. The Bethe
functional motivates a factor-local view; through it, the ELBO, the
Bethe expression at a factorised belief, and MAP share one
predictive-score core and differ in what they add to it.
At the last layer the negative ELBO is,
exactly, the belief's predictive log scores plus the prior KL plus a
\emph{binding}, a sum of per-observation Jensen gaps whose retention
makes the optimum the posterior projection. Per observation, to first
order around the point prediction, the binding penalises predictive
variance, and for canonical likelihoods it contributes the same
squared-residual term that enters the predictive score with the
opposite sign, leaving only model curvature. At fixed mean and
representation a posterior-seeking last-layer belief therefore sets its
covariance through curvature rather than residual magnitudes: a
leverage of the design, estimation uncertainty that shrinks as the data
grow, so residual dispersion not absorbed by an explicit likelihood
parameter, heteroscedastic noise or overdispersed counts, has nowhere
to go (Figure~\ref{fig:horizon_toy}). The restricted direction is the
variance at each input; its value under the predictive log score is the
classical variance-component (overdispersion) score.

The price follows from the likelihood. A Gaussian likelihood with
learned noise absorbs the average dispersion and pays only the Jensen
gap of the residual-variance profile, typically tenths of a nat; a
likelihood without a dispersion parameter
pays the whole expressible improvement over the point-prediction head,
units to tens of nats for Poisson counts; a well-specified head pays
nothing. For counts a dispersion parameter is therefore the first
remedy.

Because the score is evaluated at the point-prediction (\emph{plug-in})
head, that head alone carries a screen: for counts its Pearson
dispersion indicates whether a posterior-seeking belief will pay, and
a price proxy (Corollary~\ref{prop:plugin_test}(b)) for how much;
both held up in an exploratory validation on 28 count sets, 27 unused elsewhere
(Section~\ref{sec:experiments}). Training the belief by its score
lifts the restriction and yields the classical per-input overdispersed
heads, within a few hundredths of a nat of their direct counterparts;
what is lost is the covariance's meaning as posterior uncertainty.

\paragraph{Contributions.}
The Jensen gap, predictive-score training, the overdispersion scores,
and the heads are known (Table~\ref{tab:provenance}); new is their
reading as one restricted direction with a price.
\textbf{(i)}~The binding, the observation-dependent term that, with the
prior KL, turns predictive-score training into posterior-seeking ELBO
training, as an identity on any directed factor graph (Definition~\ref{def:binding},
Theorem~\ref{prop:binding_forms}, Corollary~\ref{cor:posterior_seeking}).
\textbf{(ii)}~The restricted direction and the residual-blind bound
covariance (Proposition~\ref{prop:overdispersion},
Corollary~\ref{cor:bound_covariance}).
\textbf{(iii)}~The price in the flat-leverage limit, for a general likelihood and in two instances
(Theorem~\ref{thm:price}, Corollaries~\ref{thm:pred_evidence},
\ref{prop:plugin_test}).
\textbf{(iv)}~The screen under a prospective protocol and a proxy
fixed before the later batches, with closure, fixed-mean, and HMC controls
(Definition~\ref{def:screen}, Section~\ref{sec:experiments}). The price
and the measurements are at the last layer with exchangeable data;
chronological splits and beliefs below the last layer are exploratory
(Table~\ref{tab:reading_guide}).

%% file: sections/objective.tex
\section{The Binding on a Directed Factor Graph}
\label{sec:general}

Consider a model over $N$ data plates, each coupling an observation
$y_n$ and input $x_n$ to an output $f_n$ through weights $\{w_l\}$
shared across plates: prior factors $p(w_l)$, observation factors
$p(y_n\mid f_n)$ (the likelihood), and intermediate factors $f_a$ for
whatever the architecture places between them, a directed factor
graph whose factors multiply to the joint (its last-layer instance:
Figure~\ref{fig:factor_graph}, Appendix~\ref{app:functionals}). A
\emph{belief} $q$ over the latent variables induces at every factor
$a$ its marginal $q_a$ on the factor's scope, the \emph{local normaliser}
$Z_a:=\int f_a\,q_a$, and the \emph{tilt} $b_a:=f_a\,q_a/Z_a$, the
marginal updated by that factor alone; at plate $n$,
$Z_n=\int p(y_n\mid f_n)\,q(f_n)\,\mathrm{d}f_n$ and
$b_n=p(y_n\mid f_n)\,q(f_n)/Z_n$. The Bethe functional
(Appendix~\ref{app:functionals}) motivates scoring each factor
locally; Theorem~\ref{prop:binding_forms} uses none of its consistency
constraints, and Proposition~\ref{prop:fsc_construction} evaluates it outside them.

\subsection{The Predictive-Score Objective}
\label{sec:objective}
\label{sec:bethe_opt}

\suppressfloats[t]
\input{figures/horizon_toy}
\begin{definition}[Predictive-score objective]
\label{def:fsc}
The predictive-score objective of a belief $q$ is the local-normaliser
sum over all factors: priors, intermediate factors
$\mathcal{A}_\text{int}$, and plates,
\begin{equation}
  F_\text{SC}(q)
    \;:=\; \sum_a (-\log Z_a)
    \;=\; \sum_l (-\log Z_{w_l})
    + \sum_{a\in\mathcal{A}_\text{int}} (-\log Z_a)
    + \sum_n (-\log Z_n).
  \label{eq:fsc_def}
\end{equation}
A deterministic coupling contributes nothing and is collapsed into the
variables it determines (Appendix~\ref{app:deterministic}), so every
factor is a positive density.
\end{definition}

The prior and intermediate factors form the model's law of its latent
variables,
$p_\text{lat}:=\prod_l p(w_l)\prod_{a\in\mathcal{A}_\text{int}}f_a$, a
normalised density in a directed model; their normalisers contain no
observation, and each plate term is the log score of the belief's
predictive, $-\log Z_n=-\log m_n(y_n)$ with
$m_n(y):=\int p(y\mid f)\,q(f_n)\,\mathrm{d}f$, so $F_\text{SC}$ is
proper-score training of $\{m_n\}$
(Proposition~\ref{prop:proper_score}, Appendix~\ref{app:statements}).
Definition~\ref{def:fsc} scores every plate with one $q$ fitted on all
the data, the \emph{shared cavity} (Definition~\ref{def:cavity},
Appendix~\ref{app:cavities}), so the sum is batchable; $q$ enters
only through its marginals.
\label{sec:relaxations}

%% file: figures/horizon_toy.tex
\begin{figure}[t]
\centering
\includegraphics[width=\linewidth]{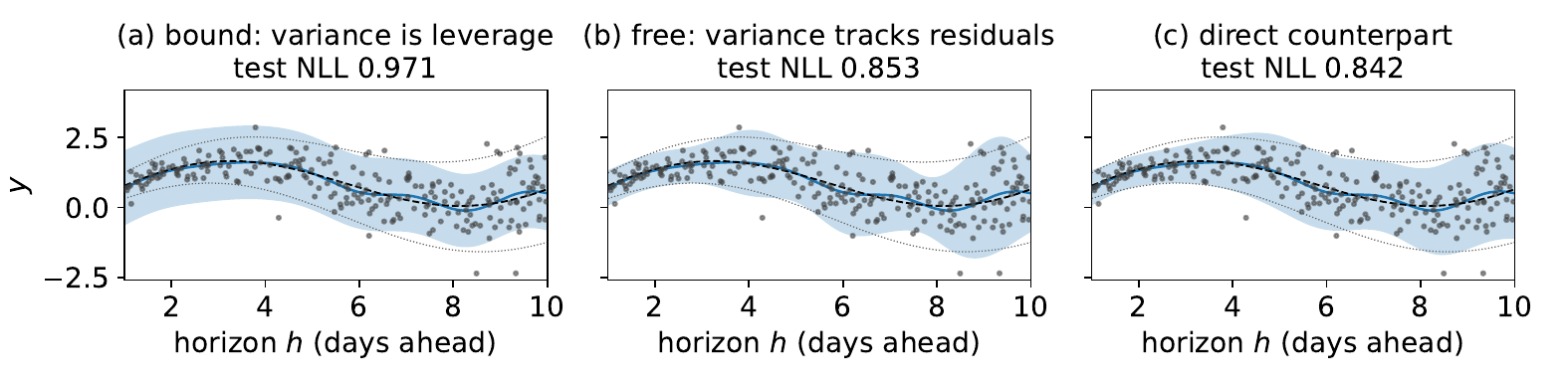}
\caption{\textbf{The binding on a toy.} Independent forecasts issued
$h$ days ahead, $y=f(h)+s(h)\varepsilon$, the noise growing with the
horizon (dots: training data; dashed: true mean; dotted: true $\pm2s$;
solid, shaded: predictive mean $\pm2$ sd); features and mean are fixed
and shared. (a)~Bound, the belief adds only
leverage~(Corollary~\ref{cor:bound_covariance}) to one noise level: a
flat band, blind to the residuals. (b)~Free, the band tracks $s(h)$; the
difference to (a) is the binding
gap~(Corollary~\ref{thm:pred_evidence}). (c)~A free log-variance head
(MVN) on the same features reaches a close profile and score.}
\label{fig:horizon_toy}
\end{figure}

%% file: sections/binding.tex
\subsection{Posterior Seeking and the Binding}
\label{sec:binding}

Write $\mathcal{Q}=\{q_\eta:\eta\in\mathcal{H}\}$ for the belief
family, $w$ for the latent variables, $\theta$ for the deterministic
parameters, prior scale $\alpha$ included, $\xi$ for the likelihood
hyperparameters, and
$\mathrm{ELBO}(q)=\mathbb{E}_q[\log p(y,w;\theta,\xi)-\log q]$,
\eqref{eq:elbo} with its KL expanded.

\begin{definition}[The binding]
\label{def:binding}
The binding of a belief is the sum of its plate couplings,
\begin{equation}
  B(q)\;:=\;\sum_{n=1}^{N}\mathrm{KL}\big(q(f_n)\,\big\|\,b_n\big).
  \label{eq:binding_def}
\end{equation}
\end{definition}

Each coupling is the per-plate Jensen gap of~\citet{Masegosa2020} in
its KL form,
\begin{equation}
  \mathrm{KL}\big(q(f_n)\,\big\|\,b_n\big)
  \;=\;\mathbb{E}_{q}\big[-\log p(y_n\mid f_n)\big]+\log Z_n\;\ge\;0,
  \label{eq:coupling}
\end{equation}
zero iff $p(y_n\mid\cdot)$ is $q(f_n)$-a.s.\ constant.

\begin{theorem}[Decomposition]
\label{prop:binding_forms}
Let the factors $a$ be those of Definition~\ref{def:fsc}, with
$\prod_af_a=p(y,w)$, $0<Z_a<\infty$, $\mathbb{E}_q\lvert\log f_a\rvert<\infty$,
and the joint entropy $H(q)$ finite. On any such graph,
\begin{equation}
  -\mathrm{ELBO}(q)
  \;=\; F_\text{SC}(q)
  \;+\; \sum_{a}\mathrm{KL}\big(q_a\,\|\,b_a\big)
  \;-\; H(q).
  \label{eq:elbo_identity}
\end{equation}
The non-observation factors then form a normalised latent law
$p_\text{lat}$ (Section~\ref{sec:objective}); the prior and intermediate terms and
the entropy sum to $\mathrm{KL}(q\,\|\,p_\text{lat})$, and
\begin{equation}
  -\mathrm{ELBO}(q)
  \;=\; \sum_{n=1}^{N}-\log Z_n \;+\; \mathrm{KL}\big(q\,\|\,p_\text{lat}\big)
  \;+\; B(q).
  \label{eq:elbo_identity_kl}
\end{equation}
\end{theorem}
The Bethe expression and MAP decompose around the same core: at a
belief that factorises within each factor's scope it is
$F_\text{SC}$ plus a belief-mismatch remainder $C_\text{mis}$
(Proposition~\ref{prop:fsc_construction}), and MAP is the point-belief
limit of $F_\text{SC}$ alone (Remark~\ref{rem:bethe_map});
Table~\ref{tab:objective_cells} collects the classical objectives this
way. $B$ collects every observation-dependent term of the ELBO beyond
the plate scores (Proposition~\ref{prop:proper_score}).
\input{tables/objective_cells}

\begin{definition}[Bound and free beliefs]
\label{def:routing}
Let the \emph{binding map} $\eta^\star(\theta,\xi)\in\arg\min_{\eta}
\mathrm{KL}\big(q_\eta\,\big\|\,p(w\mid\mathcal{D};\theta,\xi)\big)$
project the posterior onto $\mathcal{Q}$. For an objective $F$, the
belief is \emph{bound} when the map is imposed as a constraint,
$\min_{\theta,\xi}F(q_{\eta^\star(\theta,\xi)};\theta,\xi)$, and
\emph{free} when it is dropped,
$\min_{\theta,\xi,\eta}F(q_\eta;\theta,\xi)$.
\end{definition}

\begin{corollary}[Retaining the binding is posterior seeking]
\label{cor:posterior_seeking}
Fix $(\theta,\xi)$ with $0<p(y)<\infty$ and any family $\mathcal{Q}$. The right side
of~\eqref{eq:elbo_identity_kl}, the negative ELBO, is minimised over
$\mathcal{Q}$ at the binding map~\citep{Blei2017},
\begin{equation}
  \arg\min_{\eta\in\mathcal{H}}\,\big[-\mathrm{ELBO}(q_\eta)\big]
  \;=\;\arg\min_{\eta\in\mathcal{H}}\,
  \mathrm{KL}\big(q_\eta\,\big\|\,p(w\mid\mathcal{D};\theta,\xi)\big)
  \;\ni\;\eta^\star(\theta,\xi),
  \label{eq:elbo_binding}
\end{equation}
and without $B$ it is not, in general: $B$ is the term that binds the
belief to the posterior.
\end{corollary}
Minimised freely, $-\mathrm{ELBO}$ thus returns the bound belief: binding
is a consequence of retaining $B$, not a second definition.

\subsection{The Restricted Direction}
\label{sec:direction}

The \emph{plug-in belief} is the point mass at the fitted mean, $v=0$,
and its head the plug-in head. A plate sees the belief only through
its marginal $q(f_n)$; at fixed mean its variance $v$ is the leading
direction. Its predictive-score value is classical
(Appendix~\ref{app:heads_background}): with $\ell:=\log p(y\mid f)$ and
primes derivatives in $f$ at $f(x)$,
$\partial_v\mathbb{E}_{y\mid x}[-\log Z_n]\big|_{v=0^+}=-\tfrac12\kappa(x)$
with $\kappa(x):=\mathbb{E}_{y\mid x}[\ell''+\ell'^2]_{f=f(x)}$ the
overdispersion score~\eqref{eq:kappa}, zero for a well-specified
plug-in (\emph{closure}); the predictive score rewards variance wherever $\kappa(x)>0$,
the \emph{overdispersion channel}, and $\kappa\le0$ opens nothing at first order.
New is what the binding does to it:

\begin{proposition}[The binding on the overdispersion channel]
\label{prop:overdispersion}
In the setting of~\eqref{eq:kappa}, under the regularity listed in the proof
(Appendix~\ref{app:proofs_general}), with $\mathbb{E}_q$ the
expectation over the marginal and $B_n:=\mathrm{KL}(q(f_n)\,\|\,b_n)$
plate $n$'s summand of the binding, at $v=0^+$:
\textbf{(a)}~\emph{Penalty.} The binding penalises variance at first
order: for every $y$, $\partial_vB_n=\tfrac12\ell'^2\ge0$; averaged
over the data, $\partial_v\mathbb{E}_{y\mid x}[B_n]=\tfrac12\mathbb{E}_{y\mid x}[\ell'^2]$,
half the Fisher information of $f$ under closure.
\textbf{(b)}~\emph{Cancellation.} The penalty is the squared-score term
of the reward: for every $y$,
\begin{equation}
  \partial_v\,[-\log Z_n]=-\tfrac12\big(\ell''+\ell'^2\big),
  \qquad
  \partial_v\,\mathbb{E}_q[-\ell]=-\tfrac12\,\ell'',
  \label{eq:split}
\end{equation}
so the ELBO's data term reacts through the curvature alone~\citep{OpperArchambeau2009}.
\end{proposition}
$\kappa$ thus measures the value of the direction the binding
penalises, not the binding itself. The statement is local: a gain
further needs a direction the belief family can express
(Section~\ref{sec:method}).

%% file: tables/objective_cells.tex
\begin{table}[t]
\centering\scriptsize
\caption{Objectives as the predictive-score core $F_\text{SC}$
plus a remainder, each row an identity, exact corner, or limit proved
where cited; the remainders are priced in Section~\ref{sec:experiments} and Appendices~\ref{app:prior_ablation}, \ref{app:beyond_last_layer}.}
\label{tab:objective_cells}
\setlength{\tabcolsep}{4pt}
\begin{tabular}{@{}lll@{}}
\toprule
objective & decomposition & remainder \\
\midrule
$-\mathrm{ELBO}(q)$, any directed graph & $F_\text{SC}(q)+\sum_a\mathrm{KL}(q_a\,\|\,b_a)-H(q)$, \eqref{eq:elbo_identity} & factor couplings less the entropy \\
$-\mathrm{ELBO}(q)$, last-layer belief & $F_\text{SC}(q)+\mathrm{KL}(q\,\|\,p(w))+\log Z_w+B(q)$, \eqref{eq:elbo_identity_ll} & prior-term swap and the binding $B\ge0$ \\
Bethe expr., shared belief, factorised $q_a$ & $F_\text{SC}(q)+C_\text{mis}(q)$ & mismatch less node entropies (Prop.~\ref{prop:fsc_construction}) \\
MAP, $-\log p(y,w^\star)$ & $F_\text{SC}$ as $q\to\delta_{w^\star}$ & none in the limit (Remark~\ref{rem:bethe_map}) \\
evidence, Gaussian likelihood, last layer & $F_\text{SC}$ at the sequential cavity, bound & none, prior absorbed at $n{=}1$ (Prop.~\ref{prop:seq_evidence}) \\
PPGPR~\citep{Jankowiak2020} & $F_\text{SC}(q)+\mathrm{KL}(q\,\|\,p(w))+\log Z_w$ & prior-term swap (Appendix~\ref{app:prior_ablation}) \\
quadratic-variance head, no prior & $F_\text{SC}(q)+\log Z_w+\lambda\lVert\theta\rVert^2$, \eqref{eq:Lscroll} & weight decay as prior term (App.~\ref{app:prior_ablation}) \\
\bottomrule
\end{tabular}
\end{table}

%% file: sections/method.tex
\section{The Binding Price at the Last Layer and Its Test}
\label{sec:method}
\label{sec:classification}

\begin{proposition}[Reduced graph]
\label{rem:deterministic}
A deterministic activation factor has $-\log Z_a=0$ under its directed
normalisation (Appendix~\ref{app:deterministic}). With every weight
below the last layer a parameter in $\theta$ (no prior factor),
$F_\text{SC}$ contains only prior and observation terms and
$p_\text{lat}=p(w)$, the prior of the last layer's weights $w$, with
local normaliser $Z_w$: the \emph{reduced graph}.
\end{proposition}
On the reduced graph, \eqref{eq:elbo_identity_kl} reads
\begin{equation}
  -\mathrm{ELBO}(q)
  \;=\; \sum_{n=1}^{N}-\log Z_n \;+\; \mathrm{KL}\big(q(w)\,\|\,p(w)\big)
  \;+\; B(q).
  \label{eq:elbo_identity_ll}
\end{equation}

\begin{definition}[Gaussian belief at the last layer]
\label{def:scroll}
Restrict the belief to the final linear layer, $f_n=w^\top\!\psi_n$
with $\psi_n:=\mathrm{NN}_\theta(x_n)\in\mathbb{R}^H$ deterministic,
$q(w)=\mathcal{N}(\mu,\Sigma)$, $\Sigma\succeq0$, prior $p(w)=\mathcal{N}(0,\alpha^{-1}I)$.
The objective is
\begin{gather}
  q(f_n)=\mathcal{N}(m_n,v_n),\quad m_n=\mu^\top\!\psi_n,\quad
  v_n=\psi_n^\top\!\Sigma\,\psi_n,\quad
  Z_n=\int p(y_n\mid f_n;\,\xi)\,q(f_n)\,\mathrm{d}f_n,
  \label{eq:scroll_msg}\\
  \mathcal{L}(\mu,\Sigma,\theta,\xi)
   \;=\; -\log Z_w \;+\; \sum_{n=1}^N -\log Z_n,
  \qquad
  -\log Z_w = -\log\mathcal{N}(\mu;\,0,\,\Sigma+\alpha^{-1}I).
  \label{eq:Lscroll}
\end{gather}
\end{definition}
$\mathcal{L}$ is $F_\text{SC}(q)$ on the reduced graph
(Appendix~\ref{app:proofs_gaussian}); the $F_\text{SC}$ belief is this family
trained by $\mathcal{L}$, the ELBO belief the same family under the
binding (covariance families and training details in
Appendix~\ref{app:setup}). Three operations are kept apart: removing
$B$ from~\eqref{eq:elbo_identity_ll}, the PPGPR objective~\citep{Jankowiak2020,Sheth2020}; $\mathcal{L}$,
which differs from it only in the prior term $-\log Z_w$, learning
$\alpha$ in the pass at the cost of depending on the representation of
$w$ (Proposition~\ref{prop:bounded}, Appendix~\ref{app:gauge}); and the
feasible-set comparison of Definition~\ref{def:routing}, one objective
over the bound image and over the full family.

What the binding does to this belief's covariance is classical in
form~\citep{OpperArchambeau2009}: for a full covariance $\Sigma$ with
$\ell_n:=\log p(y_n\mid f)$ twice differentiable with $\ell_n,\ell_n',\ell_n''$
of at most exponential growth, every stationary point of the ELBO in
$\Sigma$, the bound belief among them, satisfies
\begin{equation}
  \Sigma^{-1}=\alpha I+\sum_n\mathbb{E}_{q(f_n)}\big[-\ell_n''\big]\,\psi_n\psi_n^\top ,
  \label{eq:bound_cov}
\end{equation}
the every-$v$ Gaussian form of Proposition~\ref{prop:overdispersion}(b)
summed over the plates against the prior KL. What it means for the
residuals is new:

\begin{corollary}[The bound covariance is residual-blind]
\label{cor:bound_covariance}
Fix $(\theta,\xi)$ and let the likelihood have a canonical link,
$\ell=(yf-b(f))/\phi+c(y,\phi)$.
\textbf{(a)}~$-\ell''=b''(f)/\phi$ contains no observation, so at
fixed mean the bound covariance~\eqref{eq:bound_cov} is a function of
the design and the mean profile alone: the inverse of the
belief-averaged Fisher information of $w$ plus the prior precision.
\textbf{(b)}~Its weighted plate variances are leverages,
$\sum_n\mathbb{E}_q[-\ell_n'']\,v_n=H-\alpha\operatorname{tr}\Sigma\le H$.
\end{corollary}
The Gaussian, Poisson log-link, and Bernoulli logit likelihoods are
canonical, and (a) is Proposition~\ref{prop:overdispersion}(b) at work:
the squared residual in the binding's penalty $\tfrac12\ell'^2$ cancels
the one in the score's reward, so the ELBO prices variance by the
model's curvature and not by the residuals. By~(b) the bound belief
carries estimation uncertainty, $O(H/N)$ per plate under balanced
leverage and curvature, and cannot carry residual dispersion, $O(1)$. In the
\emph{flat-leverage limit}, designs along which its plate variances vanish uniformly, its
plate predictive tends to the plug-in $p(y\mid f(x);\xi)$, with $\xi$
still free.

\subsection{The Price}
\label{sec:price}

\begin{definition}[Profile risk, binding gap, representation gap]
\label{def:gaps}
Fix $\theta$, the feature map $\psi$, and the mean profile $f(\cdot)$. For a plate-variance profile
$v(\cdot)$ and likelihood hyperparameters $\xi\in\Xi$ (none for Poisson,
$\sigma_\text{obs}^2$ for the Gaussian), the \emph{profile risk} and
the profiles the belief can express are
\begin{equation}
  R_\xi(v)=\mathbb{E}_x\,\mathbb{E}_{y\mid x}\big[-\log m\big(y;f(x),v(x);\xi\big)\big],
  \qquad
  \mathcal{V}_\psi=\{\psi(\cdot)^\top\Sigma\,\psi(\cdot):\Sigma\succeq0\},
  \label{eq:profile_risk_body}
\end{equation}
with $m(y;f,v;\xi)=\int p(y\mid t;\xi)\,\mathcal{N}(t;f,v)\,\mathrm{d}t$;
$R_\xi(0)$ is the plug-in risk. With
$R_\text{bound}:=\inf_{\xi\in\Xi}R_\xi(v^\text{bind}_\xi)$ the bound belief's
risk, $v^\text{bind}_\xi=\psi^\top\Sigma_\xi\psi$ for $\Sigma_\xi$ a bound
covariance~\eqref{eq:bound_cov} at $\xi$ and the fixed mean,
$R_\text{free}:=\inf_{\xi\in\Xi,\,v\in\mathcal{V}_\psi}R_\xi(v)$ the free
belief's, and $R^\star=\mathbb{E}_xH(p^\star(\cdot\mid x))$ that of the truth, the
\emph{binding gap} is $\Delta_\text{bind}:=R_\text{bound}-R_\text{free}$
and the \emph{representation gap} is $R_\text{free}-R^\star$.
\end{definition}

\begin{theorem}[The price of the binding]
\label{cor:price_general}
\label{thm:price}
In the setting of~\eqref{eq:kappa}, with $\Xi$ compact, under the
regularity listed in its proof (Appendix~\ref{app:proofs_gaussian}), in the flat-leverage
limit ($\sup_{x,\xi}v^\text{bind}_\xi(x)\to0$):
\textbf{(a)}~$R_\text{bound}\to\min_\Xi R_\xi(0)=R_{\xi^\star}(0)$, the
best plug-in, so $\Delta_\text{bind}\to\Delta_\text{bind}^\infty:=R_{\xi^\star}(0)-R_\text{free}\ge0$,
zero under closure.
\textbf{(b)}~$\Delta_\text{bind}^\infty>0$ whenever
$K=\mathbb{E}_x\big[\kappa_{\xi^\star}(x)\,\psi(x)\psi(x)^\top\big]$
has a positive eigenvalue.
\textbf{(c)}~With no $\xi$, $\Delta_\text{bind}^\infty=R(0)-R_\text{free}$,
the full expressible improvement.
\textbf{(d)}~For every $\hat\xi\in\Xi$ and expressible $\hat v$,
\begin{equation}
  R_{\xi^\star}(0)-R_{\hat\xi}(\hat v)\;\le\;\Delta_\text{bind}^\infty\;\le\;
  \mathbb{E}_x\,\mathrm{KL}\big(p^\star(\cdot\mid x)\,\big\|\,p(\cdot\mid f(x);\xi^\star)\big),
  \label{eq:price_bounds}
\end{equation}
the upper end attained iff expressible predictives approach the truth
in $\mathbb{E}_x\mathrm{KL}$.
\end{theorem}

\subsection{The Instances and the Test}
\label{sec:regression}

Under $\mathcal{N}(y_n;m_n,\sigma_\text{obs}^2)$ the plate predictive
is $\mathcal{N}(y;m_n,V_n)$, $V_n=\sigma_\text{obs}^2+v_n$~\eqref{eq:Lreg},
and the binding map of Definition~\ref{def:routing} is the conjugate posterior,
$\mu=\sigma_\text{obs}^{-2}\Sigma\Psi^\top y$ and, for a full
covariance,
\begin{equation}
    \Sigma = \big(\Psi^\top\!\Psi/\sigma_\text{obs}^2 + \alpha I\big)^{-1},
  \label{eq:R1}
\end{equation}
the Gaussian instance of~\eqref{eq:bound_cov}, over whose image the
bound belief trains while the free belief trains over the PSD cone.

\begin{corollary}[Gaussian instance: the residual-heteroscedasticity gap]
\label{thm:pred_evidence}
\label{cor:gap_closed_form}
Let $V^\star(x)=\mathbb{E}[(y-f(x))^2\mid x]$ with
$0<V^\star\le\bar V$ a.s., \emph{representable} when
$V^\star-\sigma^2\in\mathcal{V}_\psi$ for some $\sigma^2\in\Xi$; the
channel is
$\kappa=(V^\star-\sigma_\text{obs}^2)/\sigma_\text{obs}^4$.
\textbf{(a)}~\emph{(the limiting price)} For $\mathbb{E}[V^\star]\in\Xi$
the best plug-in has $\sigma_\text{obs}^2=\mathbb{E}[V^\star]$, and for
representable $V^\star$ the free belief
attains $V^\star$ and
\begin{equation}
  \Delta_\text{bind}^\infty
  \;=\; \tfrac12\Big(\log\mathbb{E}\big[V^\star\big]
        -\mathbb{E}\big[\log V^\star\big]\Big)
  \;\ge\; \frac{\operatorname{Var}\!\big[V^\star\big]}{4\,\bar V^2},
  \label{eq:routing_gap}
\end{equation}
the Jensen gap of the residual-variance profile, zero iff $V^\star$ is
a.s.\ constant.
\textbf{(b)}~\emph{(finite $N$, risk conditional on the design)} The
binding gap is strictly positive whenever no minimiser over the total
profiles $\sigma^2+v$, $v\in\mathcal{V}_\psi$, lies in the
leverage family~\eqref{eq:R1}, in particular whenever a
representable $V^\star$ lies outside it.
\end{corollary}
Freed, the split of $V_n$ and the handle on $\alpha$ are identified
only by the $O(1/N)$-per-plate prior term (Corollary~\ref{cor:pathologies}),
tamed by floors (Proposition~\ref{prop:bounded}); $q(w)$ parameterises predictive
marginals, not posterior uncertainty.

\begin{corollary}[Poisson instance: the whole overdispersion]
\label{prop:plugin_test}
Under the Poisson likelihood with log link the channel is
$\kappa=\mathrm{Var}(y\mid x)-\lambda(x)$ at matched mean, and there is
no $\xi$, so Theorem~\ref{thm:price}(c) applies and the
price is the full expressible improvement. Let $\hat\lambda$ be a Poisson plug-in, with each of the $M$ points
$(x_n,y_n)$ scored by a fit that did not see it, $\lambda_n=\mathbb{E}[y_n\mid x_n]$,
$\varsigma_n^2=\operatorname{Var}[y_n\mid x_n]$,
$\hat\kappa_n=(y_n-\hat\lambda_n)^2-\hat\lambda_n$ the per-point
score~\eqref{eq:kappa}, and
$\hat\phi=1+\tfrac1M\sum_n\hat\kappa_n/\hat\lambda_n$ the Pearson
dispersion.
\textbf{(a)}~$\mathbb{E}[\hat\kappa_n\mid x_n,\hat\lambda]=(\varsigma_n^2-\lambda_n)
+(\lambda_n-\hat\lambda_n)+(\lambda_n-\hat\lambda_n)^2$: the channel
plus the head's own mean error, which alone remains for a Poisson truth.
\textbf{(b)}~For $p^\star$ negative binomial with mean $\hat\lambda$
and variance $\phi\hat\lambda$ the upper end
of~\eqref{eq:price_bounds} is
$\Pi(\phi)=\mathbb{E}_x\mathrm{KL}\big(\mathrm{NB}(\hat\lambda,\phi)\,\|\,\mathrm{Poisson}(\hat\lambda)\big)$,
not attained for $\phi>1$.
\end{corollary}

\label{sec:test}
\begin{definition}[Certificate, channel screen, price proxy]
\label{def:screen}
Fit a plug-in head and score it on points it did not fit. Let
$\hat\kappa_n$ be its per-point score~\eqref{eq:kappa}.
\emph{Certificate.} Fit a free profile $\hat v$ at a fixed mean. The
certificate is $R_{\xi^\star}(0)-R_{\hat\xi}(\hat v)$ there, on held-out points,
the lower end of~\eqref{eq:price_bounds} estimated out of sample.
For counts, Corollary~\ref{prop:plugin_test} supplies the two
remaining objects.
\emph{Channel screen.} Compute the Pearson dispersion $\hat\phi$ and
the self-normalised score
$T=\sum_n\hat\kappa_n/\sqrt{\sum_n\hat\kappa_n^2}$. The screen calls
the channel open when $\hat\phi$ exceeds one by a chosen margin and
$T$ is significantly positive. The margin is chosen to cover, at bounded
rates, the head's own mean-error term of Corollary~\ref{prop:plugin_test}(a).
\emph{Price proxy.} Take a negative binomial with the plug-in's mean
and dispersion $\hat\phi>1$ as a stand-in for the truth. The proxy is
$\hat p=\Pi(\hat\phi)$, the upper end of~\eqref{eq:price_bounds} at that
stand-in.
\end{definition}
\label{sec:heads}
Table~\ref{tab:heads} sets the two instances side by side. Lifting the
binding leaves a likelihood-determined mixture head; its direct
counterpart keeps the data term, drops the prior term, and keeps the
profile class (Gaussian) or outputs the per-input parameter (counts).
\input{tables/instances}

%% file: tables/instances.tex
\begin{table}[t]
\centering
\scriptsize
\caption{The two instances of Theorem~\ref{thm:price} side by side:
what is trained when the binding is lifted, what it is compared with,
and how the price is sized and read. Formulas and
sources: Appendix~\ref{app:heads_background}.}
\label{tab:heads}
\setlength{\tabcolsep}{4pt}
\begin{tabular}{@{}>{\raggedright\arraybackslash}p{3.3cm}>{\raggedright\arraybackslash}p{5.0cm}>{\raggedright\arraybackslash}p{5.0cm}@{}}
\toprule
 & Gaussian likelihood & Poisson likelihood, log link \\
\midrule
the $F_\text{SC}$ head
  & quadratic heteroscedastic, $V_n=\sigma_\text{obs}^2+\psi_n^\top\Sigma\psi_n$
  & Poisson-lognormal, per-input dispersion from $\psi_n^\top\Sigma\psi_n$ \\
\addlinespace[2pt]
direct counterpart; comparator
  & quadratic variance head; free log-variance (MVN)
  & Poisson-lognormal with free per-input variance; negative binomial with per-input dispersion \\
\addlinespace[2pt]
bound belief
  & conjugate posterior~\eqref{eq:R1}, exact
  & ELBO, covariance~\eqref{eq:bound_cov}; or Laplace \\
\addlinespace[2pt]
channel $\kappa(x)$
  & $(V^\star-\sigma_\text{obs}^2)/\sigma_\text{obs}^4$
  & $\mathrm{Var}(y\mid x)-\lambda(x)$, matched mean \\
\addlinespace[2pt]
hyperparameter $\xi$
  & $\sigma_\text{obs}^2$, absorbs the average
  & none \\
\addlinespace[2pt]
price $\Delta_\text{bind}^\infty$ (flat leverage)
  & Jensen gap of the residual-variance profile~\eqref{eq:routing_gap}
  & the full expressible improvement $R(0)-R_\text{free}$ \\
\addlinespace[2pt]
read at the plug-in
  & heteroscedasticity score test; fitted log-variance Jensen gap (not run)
  & screen on $\hat\phi$ and $T$; proxy $\hat p=\Pi(\hat\phi)$ (Definition~\ref{def:screen}) \\
\addlinespace[2pt]
measured (Section~\ref{sec:experiments})
  & typically tenths of a nat
  & units to tens of nats \\
\bottomrule
\end{tabular}
\end{table}

%% file: sections/experiments.tex
\section{Experiments}
\label{sec:experiments}

The controlled measurements are at the last layer with exchangeable
data: Gaussian (eight UCI sets, three large tabular sets, two frozen
deep embeddings) and Poisson (three generators, six real count sets);
data, splits, protocols, and arms in Appendix~\ref{app:setup}.
The Poisson generator, without overdispersion, is the \emph{synthetic closure
control}, $\kappa\equiv0$; abalone, length of stay, and claim counts
are \emph{empirical near-closure sets}; the two time series, bike and
fluview (weekly influenza counts), are also split chronologically.
Every arm is validation-selected per seed over four backbones,
references also over $\lambda$; 20 seeds; paired $t$-tests. The free
arm is the $F_\text{SC}$ belief, the common term of
Table~\ref{tab:objective_cells} and the object of study, not a
competitor, which the prior-term ablations tie to
the PPGPR form in NLL (Section~\ref{sec:method}). The fixed-mean control isolates the binding; for counts it also swaps
the prior term, which the near-zero prior-swap column prices (Table~\ref{tab:binding_gap}).
End-to-end posterior-seeking arms may add approximation,
optimisation, early-stopping, and feature-adaptation effects.
\input{tables/screen_macros}

\input{tables/table_binding_gap}

\paragraph{How much does the binding cost?}
Table~\ref{tab:binding_gap} isolates it at fixed mean and features
(\emph{frozen, bound}). Under the Gaussian likelihood it is typically
tenths of a nat: the gap predicted from the two fitted profiles tracks the measured
one (correlation $0.91$), the complete conjugate map at the same
features moves it by at most $0.13$ on every set but naval, and the
exact corner, exact inference in the neural-linear model at its own
features, sits $0.05$ above the free belief
(Appendix~\ref{app:gaussian_full}). Under the Poisson likelihood the
same control gives $2$--$10$ nats on the overdispersed generators,
bike, and fluview and $48$ on news (article shares), each below its
certificate (loosely on fluview and news), and every closure row, synthetic or empirical, is zero
within $0.01$. On the generators and bike a belief bound under the
negative-binomial likelihood, whose dispersion absorbs the average,
costs $0.01$--$0.05$ where the Poisson-bound belief costs $2$--$9$
(Theorem~\ref{thm:price}(a),(c)): for counts a likelihood with a dispersion
parameter is the first remedy.

\paragraph{What does a trained ELBO belief pay?}
Less than the control, because early stopping keeps it from its own
optimum. The early-stopped ELBO belief loses $-0.03$ to $2.1$ nats
under random splits ($3.5$ on fluview under the chronological one):
on the generators and news validation keeps it near its initial
belief, as its validation NLL worsens from the first checks on
(Figure~\ref{fig:elbo_curves}); on bike and fluview it stops late. Trained to the budget it moves toward
the Poisson plug-in, and the HMC posterior predictive at the same
features sits there, $12$ and $5$ nats above the free belief on the
generators and equal to it on the closure control.

\paragraph{Can the price be read in advance?}
The channel screen (Definition~\ref{def:screen}) calls the channel
open when $\hat\phi\ge1.35$ and the self-normalised score $T>1.645$, a
margin above a fitted head's own excess on the closure generator, and
$\hat p$ counts as right within a factor of three of the
certificate. Both were run, relu only, on \screenSets{}
count data sets, all but wine quality (a UCI set) new to the paper, the screen prospectively
specified, changes disclosed in Appendix~\ref{app:screen}, $\hat p$ derived on the first
\screenBOneN{} and fixed before the other \screenTestN{}, which test
it, the last \screenBThreeN{} chosen to be rich in closed sets. Of the \screenValid{} valid sets
\screenOpenValid{} are open by the measured binding gap and an
always-open rule scores \screenOpenDecided{} of \screenDecided{},
so the evidence is the closed side: \screenClosedCalledClosed{} called
closed, one false open, no false closed; overall it agrees with the
measured outcome on \screenCorrect{} of \screenDecided{},
\screenTies{} abstentions. On the
\screenTestN{} later sets $\hat p$ was within the factor on
\screenTestBand{} of \screenTestOpen{} open sets (Spearman
\screenRhoPTest{}). The screen's outcome is the binding gap and the
proxy's target the certificate; on the six valid sets of the second batch the gap is
$0.82$--$0.98$ of the certificate on every set, the regime of
Theorem~\ref{thm:price}(c), and it falls below where a learned prior
scale absorbs the constant direction (Seoul bikes: $0.4$ of $43$
nats). End to end, where features and optimisation also move, the free
belief or the global-dispersion NB head was best in NLL on every open set, every
head tied within $0.07$ on the closed ones, and the ELBO belief lost
where the screen called open and on both abstentions
(Table~\ref{tab:appendix_screen_endtoend}).

\input{tables/table_heads_counts}

\paragraph{What is trained instead?}
In Table~\ref{tab:table_heads_counts} the derived count head is within
$0.02$ nats of its direct counterpart on every set but fluview, the
smallest ($N{=}1{,}343$; $0.04$); it leads per-input negative binomial
by $0.27$ on news and $0.37$ on fluview and the global-dispersion
default by $0.04$--$0.27$ on the generators, bike, and news, conceding
fluview to it by $0.14$. The likelihood change recovers most of the Poisson price; the
per-input profile adds these hundredths to tenths of a nat.
The Gaussian belief stands the same way against its counterpart
(Appendix~\ref{app:gaussian_full}). In NLL the $F_\text{SC}$ belief is
a matched head: the PPGPR form ties it, and the regulariser's
strength is immaterial (the counterpart at any single
$\lambda$ within $0.002$ of its tuned value, the belief at any fixed
prior precision within $0.02$) except on fluview
(Appendices~\ref{app:prior_ablation}, \ref{app:beyond_gaussian}). It
differs in carrying its dispersion in the quadratic form of the mean's
own features rather than in a free output, which shows only at small
$N$ and under temporal shift (Appendix~\ref{app:beyond_gaussian});
where that pays is open.

\paragraph{Secondary views.}
The mismatch arm, trained on $F_\text{SC}+C_\text{mis}$
(Proposition~\ref{prop:fsc_construction}), carries a weaker partial binding
of the same sign (last column of Table~\ref{tab:binding_gap}; zero on
closure rows); the leave-one-out cavity (ten seeds, relu+LN) lands \priceCavity{} nats from
the shared one on the data-rich UCI sets but naval
(Proposition~\ref{prop:cavity_invariance}). Below the last layer,
exploratory (ten seeds) and under the moment-matched construction of
Appendix~\ref{app:beyond_last_layer}, the picture repeats at every
depth: on the overdispersed generators the score-trained belief is
within $0.07$ nats of the oracle, the ELBO belief $0.1$--$0.4$ above it
early-stopped (where hidden layers carry a belief, within the
construction's approximation error) and $2.8$--$10$ at the budget.

\paragraph{Where it ends.}
The freed covariance loses the leverage meaning: as a shift score it
has AUROC $0.58$--$0.74$ against the bound belief's $0.88$--$0.95$,
and a conjugate refit on frozen trained features restores most
($0.75$--$0.96$, Appendix~\ref{app:ood}). The expressibility condition shows on
tanh backbones: $\kappa>0$, yet the fitted profile is uncorrelated with
the true one ($r=0.00$) and the per-input profile buys nothing. Near-noiseless data
concede to the exact corner; the screen reads available
overdispersion, not expressibility, is unstable below mean count one,
and void where training diverges (metro traffic).

%% file: tables/screen_macros.tex
\newcommand{\screenSets}{twenty-eight}
\newcommand{\screenDecided}{twenty-five}
\newcommand{\screenCorrect}{twenty-four}
\newcommand{\screenTies}{two}
\newcommand{\screenBOneN}{thirteen}
\newcommand{\screenBTwoN}{seven}
\newcommand{\screenBTwoBand}{six}
\newcommand{\screenBTwoOpen}{six}
\newcommand{\screenBOneBand}{nine}
\newcommand{\screenBOneOpen}{ten}
\newcommand{\screenRhoS}{0.81}
\newcommand{\screenRhoPOne}{0.99}
\newcommand{\screenRhoPTwo}{0.83}
\newcommand{\screenSeeds}{20}
\newcommand{\screenBaseline}{sixteen}
\newcommand{\screenBTwoRatio}{0.84}
\newcommand{\screenClosureSeeds}{60}
\newcommand{\screenClosureClosed}{55}
\newcommand{\screenClosureSLo}{0.07}
\newcommand{\screenClosureSHi}{0.13}
\newcommand{\screenOpenDecided}{eighteen}
\newcommand{\screenValid}{twenty-seven}
\newcommand{\screenOpenValid}{twenty}
\newcommand{\screenBThreeN}{eight}
\newcommand{\screenTestN}{fifteen}
\newcommand{\screenTestBand}{ten}
\newcommand{\screenTestOpen}{ten}
\newcommand{\screenRhoPTest}{0.95}
\newcommand{\screenTestRatio}{0.84}
\newcommand{\screenAccLo}{0.80}
\newcommand{\screenAccHi}{0.99}
\newcommand{\screenBandLo}{0.72}
\newcommand{\screenOpenCalledOpen}{eighteen}
\newcommand{\screenOpenCalledClosed}{none}
\newcommand{\screenOpenAbstained}{two}
\newcommand{\screenClosedCalledOpen}{one}
\newcommand{\screenClosedCalledClosed}{six}

%% file: tables/table_binding_gap.tex
\begin{table}[t]
\centering\scriptsize
\setlength{\tabcolsep}{2.5pt}
\caption{The binding gap across likelihoods: paired test-NLL difference to the $F_\text{SC}$ belief (positive: the arm is worse), means over 20 seeds (Gaussian frozen control and prediction: medians over seeds); the Gaussian rows are medians across data sets (per data set: Appendices~\ref{app:gaussian_full}, \ref{app:beyond_gaussian}). Price column: for the Gaussian rows the gap predicted from the two fitted profiles (Corollary~\ref{thm:pred_evidence}); for counts the certificate $R(0)-R(\hat v)$ at the selected backbone, a lower bound on the \emph{limiting} price (Theorem~\ref{thm:price}(d)), which the finite-$N$ bound belief undercuts by keeping some leverage. The frozen control compares, for the Gaussian rows, one objective over feasible sets nested up to the free side's $\varepsilon I$ floor, and for counts an ELBO refit with a score refit, which also swaps the prior term; the \emph{prior swap} column prices that swap alone, the free belief trained with $\mathrm{KL}(q\,\|\,p)$ in place of $-\log Z_w$, the KL form minus the binding. ELBO: for the Gaussian rows VBLL, prior KL tempered (Appendix~\ref{app:setup}).}
\label{tab:binding_gap}
\begin{tabular}{@{}lrrrrrrr@{}}
\toprule
 & \multicolumn{2}{c}{exact / controlled} & price & prior swap & \multicolumn{3}{c@{}}{posterior-seeking arms} \\
\cmidrule(lr){2-3}\cmidrule(lr){4-4}\cmidrule(lr){5-5}\cmidrule(l){6-8}
 & exact corner & frozen, bound & pred.\,/\,cert. & free, KL & ELBO & Laplace & $+C_\text{mis}$ \\
\midrule
\multicolumn{8}{@{}l}{\emph{Gaussian}}\\
UCI (median of 8) & $+0.05$ & $+0.23$ & $+0.22$ & $+0.00$ & $+0.03$ & -- & $+0.08$ \\
at scale (median of 5) & -- & $+0.16$ & -- & -- & $+0.14$ & -- & -- \\
\multicolumn{8}{@{}l}{\emph{counts}}\\
PLN generator ($N{=}2$k) & -- & $+9.39$ & $+15.4$ & $-0.00$ & $+1.15$ & $+10.2$ & $+0.36$ \\
NB generator ($N{=}2$k) & -- & $+8.31$ & $+9.58$ & $-0.00$ & $+1.11$ & $+4.57$ & $+0.11$ \\
bike & -- & $+2.40$ & $+5.97$ & $-0.01$ & $+2.10$ & $+5.12$ & $+0.49$ \\
fluview & -- & $+9.76$ & $+128.8$ & $+0.01$ & $-0.03$ & $+75.2$ & $+0.41$ \\
news & -- & $+48.2$ & $+3205.2$ & $+0.00$ & $+0.61$ & $+2325.7$ & $+2.28$ \\
Poisson, closure control & -- & $+0.00$ & $+0.00$ & $-0.00$ & $+0.00$ & $-0.00$ & $-0.01$ \\
abalone, near-closure & -- & $+0.01$ & $-0.00$ & $+0.00$ & $-0.00$ & $+0.00$ & $+0.00$ \\
stay, near-closure & -- & $+0.00$ & $-0.01$ & $-0.00$ & $-0.00$ & $-0.01$ & $-0.00$ \\
claims, near-closure & -- & $+0.01$ & $+0.01$ & $-0.00$ & $+0.00$ & $+0.00$ & $-0.00$ \\
\multicolumn{8}{@{}l}{\emph{counts, chronological}}\\
bike, ordered split & -- & $+11.6$ & -- & $+0.01$ & $+1.93$ & $+17.7$ & $+0.69$ \\
fluview, ordered split & -- & $+8.51$ & -- & $+0.05$ & $+3.51$ & $+243.6$ & $+0.24$ \\
\bottomrule
\end{tabular}
\end{table}

%% file: tables/table_heads_counts.tex
\begin{table}[ht]
\centering
\scriptsize
\setlength{\tabcolsep}{3pt}%
\caption{Count likelihood: the derived Poisson-lognormal head (per-input dispersion carried by the shared $\Sigma$) against its direct predictive-family counterpart ($\dagger$: same family, dispersion output directly) and the community comparators ($\ddagger$), test NLL (20 seeds; every row validation-selected over four backbones, the $\dagger$ and $\ddagger$ heads also over a weight-decay grid $\lambda$, the \FSC rows trained in one pass with nothing tuned). \FSC-Full and -Diag within $0.05$ of each other. Bold = best; italic = not separated from the best (paired $t$-test, $p\ge0.05$).}
\label{tab:table_heads_counts}
\begin{tabular}{lcccccccc}
\toprule
 & PLN 2k & NB 2k & bike & fluview & news & abalone & stay & claims \\
\midrule
\FSC-Full (Poisson-lognormal) & 2.987 & 2.455 & 4.881 & \textit{8.643} & 8.750 & \textit{2.250} & 1.629 & 0.216 \\
\FSC-Diag & 2.991 & 2.456 & \textbf{4.833} & \textit{8.673} & \textbf{8.730} & \textbf{2.250} & 1.629 & \textbf{0.215} \\
PLN, free per-input dispersion$^\dagger$ & \textbf{2.969} & 2.451 & 4.884 & 8.687 & \textit{8.734} & \textit{2.251} & \textbf{1.615} & 0.218 \\
NB, per-input dispersion$^\ddagger$ & 3.030 & \textbf{2.423} & \textit{4.851} & 9.010 & 9.016 & \textit{2.251} & \textit{1.615} & \textit{0.215} \\
NB, global dispersion$^\ddagger$ & 3.115 & 2.573 & 4.919 & \textbf{8.505} & 9.021 & 2.255 & 1.619 & \textit{0.215} \\
Poisson$^\ddagger$ & 15.222 & 7.587 & 10.054 & 121.878 & 2836.033 & \textit{2.250} & 1.616 & 0.217 \\
\bottomrule
\end{tabular}
\end{table}

%% file: sections/related_work.tex
\section{Related Work}
\label{sec:related}

\paragraph{The ELBO and the predictive score.}
That the posterior maximises the ELBO, which bounds the evidence and
not the predictive score, underlies generalised
Bayes~\citep{Bissiri2016,Knoblauch2022} and the PAC-Bayes account of
the misspecified regime~\citep{Masegosa2020,Morningstar2022}, and
bears on the cold-posterior effect~\citep{Wenzel2020}. Training a belief's predictive by its
score is established at sparse GPs as direct loss
minimisation~\citep{Sheth2020,Wei2021} and PPGPR~\citep{Jankowiak2020},
the form~\eqref{eq:elbo_identity_ll} without $B$ at a shared belief
whose parameters are free, the uncollapsed choice of sparse-GP
inference~\citep{Hensman2013}, carried to deep
GPs~\citep{Jankowiak2020dspp}. The two objectives are the ends of one known family, the power-EP and
black-box $\alpha$ energies~\citep{Minka2004,HernandezLobato2016}
(VFE to FITC for sparse GPs~\citep{Bui2017,Bauer2016}): with each
site's cavity replaced by the belief~\citep{LiGal2017}, $\alpha\to0$
is the ELBO and $\alpha{=}1$ is~\eqref{eq:elbo_identity_ll} without
$B$, while the Bethe expression at the shared belief is
$F_\text{SC}+C_\text{mis}$ (Proposition~\ref{prop:fsc_construction}).
What is added is the term that separates the two ends
(Table~\ref{tab:provenance}, Appendix~\ref{app:background}).

\paragraph{The restricted direction and its test.}
The statistics behind Proposition~\ref{prop:overdispersion} and the
screen are classical: the small-dispersion expansion and overdispersion
score~\citep{Cox1983,Chesher1984,DeanLawless1989,CameronTrivedi1990,Lin1997},
the Pearson dispersion and heteroscedasticity score
tests~\citep{McCullaghNelder1989,BreuschPagan1979},
and the Gaussian-VI fixed
point~\eqref{eq:bound_cov}~\citep{OpperArchambeau2009}; new
is their joint reading: the binding cancels the squared-residual part
of the score's reward, so the score values the restricted direction,
sized by the limiting price.

\paragraph{Heads and last-layer BNNs.}
The Gaussian head is a mean--variance network~\citep{NixWeigend1994}
(stabilised in~\citealp{Seitzer2022,Stirn2023,Immer2023}) whose variance
is quadratic in the features, as in random-coefficient
regression~\citep{HildrethHouck1968}; count regression
ships Poisson and negative-binomial heads with a global dispersion,
Poisson-lognormal and mixed models~\citep{Hinde1982,CameronTrivedi1998},
and neural heads with per-input dispersion~\citep{Young2025}.\companionsentence{}
Last-layer BNNs pair a backbone
with a Bayesian final layer~\citep{Kristiadi2020,CalvoOrdonez2026}: Laplace
Redux~\citep{Daxberger2021}, neural-linear empirical
Bayes~\citep{Snoek2015,Ober2019} at the exact corner,
VBLL~\citep{Harrison2024} free in the covariance but bound in the
objective (Corollary~\ref{cor:posterior_seeking}); deep kernel
methods~\citep{Liu2020,vanAmersfoort2021} read epistemic
uncertainty from the bound belief's leverage, which the free belief
gives up (Appendix~\ref{app:ood}).

%% file: sections/conclusion.tex
\section{Conclusion}
\label{sec:conclusion}

For a canonical link at fixed mean and representation, a
posterior-seeking last-layer belief sets its covariance by the model's
curvature, with no residual term. The binding, the ELBO's
per-observation Jensen gaps, restricts the variance at each input, a
direction valued by the classical overdispersion score, at a price set
by the likelihood: typically tenths of a nat for a Gaussian with
learned noise, units to tens for overdispersed Poisson counts, nothing
for a well-specified head. For counts a dispersion parameter is the
first remedy, and the point-prediction head says in advance, in an
exploratory validation on new count data (Section~\ref{sec:experiments}),
whether a posterior-seeking belief will pay (its cross-fitted Pearson
dispersion) and roughly how much (a price proxy,
Corollary~\ref{prop:plugin_test}(b)). Training the belief by its score yields the classical
per-input overdispersed heads, whose covariance has no posterior meaning. The
two roles of the covariance need not share one matrix. Freed, it
carries the residual dispersion and loses its leverage meaning; bound,
it keeps the leverage and cannot carry the dispersion. At the last
layer the split works: the free belief gives the predictive,
and a conjugate refit at the same features recovers most of the shift score
(Appendix~\ref{app:ood}). Whether it survives below the last layer,
where a change of belief moves the mean, the dispersion, and the
representation together, is the open question. The
analysis assumes the last layer, exchangeable data, an expressible
direction, and a fitted plug-in; near-noiseless data concede to exact
inference, the factor-normaliser form depends on the weights'
parameterisation, and exact deep marginals remain open.

%% file: tables/reading_guide.tex
\begin{table}[t]
\centering\scriptsize
\caption{Reading guide: every statement and measurement of the paper
by scope and evidential status.}
\label{tab:reading_guide}
\setlength{\tabcolsep}{4pt}
\begin{tabular}{@{}>{\raggedright\arraybackslash}p{3.6cm}>{\raggedright\arraybackslash}p{2.9cm}>{\raggedright\arraybackslash}p{6.6cm}@{}}
\toprule
scope & status & statements and measurements \\
\midrule
known objects (Appendix~\ref{app:background}) & definitions with sources & Definitions~\ref{def:cavity}, \ref{def:score_training} \\
\addlinespace[2pt]
directed factor graph (Section~\ref{sec:general}) & exact identities; first order at $v=0^+$ & Definitions~\ref{def:fsc}--\ref{def:routing}, Theorem~\ref{prop:binding_forms}, Corollary~\ref{cor:posterior_seeking}; Proposition~\ref{prop:overdispersion} \\
\addlinespace[2pt]
Gaussian belief at the last layer (Section~\ref{sec:method}) & exact & Proposition~\ref{rem:deterministic}, Definition~\ref{def:scroll}, Corollary~\ref{cor:bound_covariance} \\
\addlinespace[2pt]
the price and its instances (Section~\ref{sec:method}) & exact; the price statements in the flat-leverage limit & Definition~\ref{def:gaps}, Theorem~\ref{thm:price}, Corollaries~\ref{cor:gap_closed_form}, \ref{prop:plugin_test}, Definition~\ref{def:screen} \\
\addlinespace[2pt]
extended statements (Appendix~\ref{app:statements}) & exact; Remark~\ref{rem:bethe_map} and Proposition~\ref{prop:routing_selection} as limits, Lemma~\ref{lem:shared_loo} and Proposition~\ref{prop:cavity_invariance} with order terms & Propositions~\ref{prop:proper_score}, \ref{prop:fsc_construction}, \ref{prop:seq_evidence}, \ref{prop:bounded}, \ref{prop:cavity_invariance}, \ref{prop:routing_selection}, Corollaries~\ref{cor:cw_form}, \ref{cor:pathologies}, Lemmas~\ref{lem:reg_perturbation}, \ref{lem:shared_loo}, Remark~\ref{rem:bethe_map} \\
\addlinespace[2pt]
frozen last-layer controls (Section~\ref{sec:experiments}) & controlled measurement at fixed features, and at fixed mean except the conjugate map and HMC & the frozen and price columns of Table~\ref{tab:binding_gap}; conjugate map, pinned noise, predicted gap and selection sweep (Figure~\ref{fig:routing_selection_measured}), NB-bound refits, cavity cell, HMC (Appendices~\ref{app:gaussian_full}, \ref{app:beyond_gaussian}) \\
\addlinespace[2pt]
end-to-end arms (Section~\ref{sec:experiments}) & empirical; approximation, optimisation, early stopping, and feature adaptation enter & the exact-corner, prior-swap, and posterior-seeking columns of Table~\ref{tab:binding_gap}; Tables~\ref{tab:table_heads_counts}, \ref{tab:headline_nll}, \ref{tab:table_scale_nll}; prior-term ablations (Appendices~\ref{app:prior_ablation}, \ref{app:beyond_gaussian}); count probes and backbone profile (Appendix~\ref{app:beyond_gaussian}); ELBO dynamics; OOD (Appendix~\ref{app:ood}) \\
\addlinespace[2pt]
chronological splits (Section~\ref{sec:experiments}, Appendix~\ref{app:beyond_gaussian}) & exploratory & the ordered rows of Table~\ref{tab:binding_gap} \\
\addlinespace[2pt]
count screen (Section~\ref{sec:experiments}, Appendix~\ref{app:screen}) & exploratory, prospectively specified & 28 count sets, screen and price proxy; outside the protocol, the calibration study and the mean-misfit and capacity checks \\
\addlinespace[2pt]
below the last layer (Appendix~\ref{app:beyond_last_layer}) & exploratory, moment-matched construction & depth audit and count channel at depth \\
\bottomrule
\end{tabular}
\end{table}

%% file: sections/app_background.tex
\section{Background}
\label{app:background}

This appendix collects the known objects the body builds on, each
with its source: the two functionals whose decomposition the body
states (Appendix~\ref{app:functionals}), the heads and the
overdispersion score behind the channel and the screen
(Appendix~\ref{app:heads_background}), the cavities with the leverage
and site-form identities (Appendix~\ref{app:cavities}), and proper
scoring (Appendix~\ref{app:proper_scoring}).

\subsection{The Two Functionals}
\label{app:functionals}

Both functionals are stated on the model's own graph
(Figure~\ref{fig:factor_graph}): weights $w$, one output $f_n$ per
plate, the prior factor $p(w)$, the observation factors
$p(y_n\mid f_n)$, and the couplings $f_n=w^\top\psi(x_n)$; notation
kept throughout: $p$ (or $f_a$ for a generic factor) for model
factors, $q$ or $b$ for beliefs.

\input{figures/factor_graph}

\paragraph{Evidence lower bound.}
Variational inference posits one tractable \emph{belief} $q(w)$ over
the weights of the joint $p(y,w)=p(y\mid w)\,p(w)$ and maximises the
ELBO, a lower bound on the log-evidence and the standard training
objective for weight-space Bayesian neural
networks~\citep{Blundell2015}:
\begin{equation}
  \log p(y) \;\geq\; \mathbb{E}_{q(w)}\!\left[\log p(y\mid w)\right]
  - \mathrm{KL}(q(w)\|p(w))
  \;=:\; \mathrm{ELBO}(q),
  \label{eq:elbo}
\end{equation}
obtained by Jensen's inequality applied to
$\log p(y)=\log\mathbb{E}_{q}[p(y,w)/q(w)]$. Its negative is the
variational free energy
$-\mathbb{E}_{q}[\log p(y\mid w)]-\mathbb{E}_{q}[\log p(w)]-H(q)$:
one belief, one joint expectation, one joint entropy.

\paragraph{The Bethe free energy.}
The Bethe functional keeps the same factors but scores them
\emph{locally}: a factor belief $b_a(x_a)$ for every factor $f_a$ over
its variables $x_a$ (here the prior, $a{=}0$, with $x_0=\{w\}$, and
one factor per plate, $a{=}n$, the observation factor composed with its
coupling, $p(y_n\mid w^\top\psi_n)$ with $x_n=\{w\}$), a variable
belief $b_i(x_i)$ for every variable, and, for a general joint
$p(x)=\frac1Z\prod_a f_a(x_a)$,
\begin{equation}
  F_\text{Bethe}(b)
   = -\sum_a \mathbb{E}_{b_a}\!\big[\log f_a\big] - H_\text{Bethe}(b),
  \qquad
  H_\text{Bethe}(b) = \sum_a H(b_a) - \sum_i (d_i-1)\,H(b_i),
  \label{eq:bethe_general}
\end{equation}
with $d_i$ the degree of variable $i$ ($d_w=N{+}1$ here: the prior and
the $N$ plates), subject to normalisation and \emph{local consistency}:
each factor belief must marginalise to the belief of every neighbouring
variable, $\sum_{x_a\setminus x_i} b_a = b_i$~\citep{Yedidia2001}.
Against~\eqref{eq:elbo} the joint expectation has become a sum of
per-factor expectations and the joint entropy the Bethe entropy, a
sum of factor entropies corrected for the variables they share. The
functional is defined on \emph{any} factor graph, so the framework
is not tied to the last layer; the full network's graph is loopy in
general, since every stochastic weight layer beyond the last and the
backbone's couplings add cycles. With deterministic backbone layers,
Proposition~\ref{rem:deterministic} collapses the intermediate factors
to constraints, and what remains, $w$ coupled to each plate through the
feature map, is the star of Figure~\ref{fig:factor_graph}. On loopy graphs the stationary
points of~\eqref{eq:bethe_general} coincide with the fixed points of
belief propagation, and those of EP under moment-matching
constraints~\citep{Yedidia2001,Minka2001}; the body uses
the functional only for its local view:
Proposition~\ref{prop:fsc_construction}
evaluates the Bethe expression, \eqref{eq:bethe_general} without its
constraints, at the shared belief and reads
$F_\text{SC}$ off as its local-normaliser part
(Definition~\ref{def:fsc}).

\paragraph{Gaussian message convolution.}
Marginalising a Gaussian likelihood $\mathcal{N}(y;f,\sigma_\text{obs}^2)$
over a Gaussian message $\mathcal{N}(f;\mu_f,v_f)$ is a closed-form
convolution,
\begin{equation}
  \int \mathcal{N}(y;f,\sigma_\text{obs}^2)\,\mathcal{N}(f;\mu_f,v_f)\,\mathrm{d}f
  = \mathcal{N}(y;\mu_f,\,\sigma_\text{obs}^2+v_f),
  \label{eq:gauss_conv}
\end{equation}
the standard GP regression predictive~\citep{Williams2006}. Under the
Gaussian likelihood it makes~\eqref{eq:Lscroll} explicit,
\begin{equation}
  \mathcal{L}_\text{Gauss}
    = -\log Z_w + \sum_n\!\left[\frac{(y_n-m_n)^2}{2\,V_n}
        + \tfrac{1}{2}\log(2\pi V_n)\right],
  \qquad
  V_n = \sigma_\text{obs}^2 + v_n.
  \label{eq:Lreg}
\end{equation}
Applied plate by plate
against the running conjugate posterior $(\mu_{<n},\Sigma_{<n})$ on
plates $1,\dots,n-1$ (the prior at $n=1$), the same convolution gives the
prequential form of the Gaussian evidence~\citep{FongHolmes2020},
\begin{equation}
  -\log p(y\mid X)
    = \sum_{n=1}^N\!\left[\frac{(y_n-\mu_{<n}^\top\!\psi_n)^2}{2\,V_n^{\mathrm{seq}}}
        + \tfrac{1}{2}\log\!\big(2\pi V_n^{\mathrm{seq}}\big)\right],
  \qquad
  V_n^{\mathrm{seq}} = \sigma_\text{obs}^2 + \psi_n^\top\Sigma_{<n}\,\psi_n ,
  \label{eq:T11seq}
\end{equation}
the exact neural-linear empirical-Bayes model, order-invariant by the
chain rule.

\subsection{Heads and the Overdispersion Score}
\label{app:heads_background}

\paragraph{The heads, and where they are known.}
With the belief moments $(m,v)$ of~\eqref{eq:scroll_msg}, the per-plate
predictive $Z_n$ of Definition~\ref{def:scroll} is, per factor:
Gaussian, $\mathcal{N}(y;m,\sigma_\text{obs}^2+v)$ by~\eqref{eq:gauss_conv},
the heteroscedastic head of the mean--variance
network~\citep{NixWeigend1994} with a variance quadratic in the
features; Poisson with log link,
$\int\mathrm{Poisson}(y\mid e^{f})\,\mathcal{N}(f;m,v)\,\mathrm{d}f$
with $\mathrm{Var}[y\mid x]=\lambda+\lambda^2(e^{v}-1)$,
$\lambda=e^{m+v/2}$, the
Poisson-lognormal head, which as a regression with a Gaussian effect
on the log rate is the compound
Poisson regression of~\citet{Hinde1982} in the overdispersed-count
family of~\citet{CameronTrivedi1998}, whose classical form carries one
global dispersion (counterparts, comparators and bound beliefs:
Table~\ref{tab:heads}, Appendix~\ref{app:setup}).

\paragraph{The overdispersion channel.}
Fix the mean profile $f(\cdot)$ and let the plate marginal have mean
$f(x)$, variance $v$, and third absolute central moment $o(v)$; let
$p(y\mid\cdot)\in C^3$ with bounded derivatives, $p(y\mid f(x))>0$
a.s., and the difference quotient of $-\log Z_n$ at $v=0^+$ be
dominated by a $y$-integrable function; statements that expand $\ell$
itself (Proposition~\ref{prop:overdispersion}) add the regularity listed
in their proof. With $\ell:=\log p(y\mid f)$
and primes derivatives in $f$ at $f(x)$, the small-dispersion
expansion~\citep{Cox1983,Chesher1984} gives
\begin{equation}
  \partial_v\,\mathbb{E}_{y\mid x}[-\log Z_n]\big|_{v=0^+}=-\tfrac12\,\kappa(x),
  \qquad
  \kappa(x):=\mathbb{E}_{y\mid x}\big[\ell''+\ell'^2\big]_{f=f(x)} ,
  \label{eq:kappa}
\end{equation}
with $\kappa$ the overdispersion score of~\citet{DeanLawless1989}
and~\citet{CameronTrivedi1990} (its GLMM form:~\citealp{Lin1997}); for
a well-specified plug-in, $y\sim p(\cdot\mid f(x))$, $\kappa\equiv0$
by the information-matrix equality~\citep{White1982}. Under a likelihood
$p(y\mid f;\xi)$, $\kappa_\xi$ is~\eqref{eq:kappa} with
$\ell=\log p(y\mid f;\xi)$. Its two instances in the paper:
Gaussian, $\ell'=r/\sigma_\text{obs}^2$ and
$\ell''=-1/\sigma_\text{obs}^2$, so
$\kappa=(\mathbb{E}[r^2\mid x]-\sigma_\text{obs}^2)/\sigma_\text{obs}^4$;
Poisson with $\lambda=e^f$, $\ell'=y-\lambda$ and $\ell''=-\lambda$, so
$\kappa=\mathrm{Var}(y\mid x)-\lambda(x)$ at matched mean;
\eqref{eq:kappa} is the average over $y\mid x$ of the per-$y$
identity~\eqref{eq:split}.

\paragraph{Overdispersion diagnostics.}
The objects the screen and proxy of Definition~\ref{def:screen} are
built from are classical. For a Poisson fit with rates $\hat\lambda_n$, the Pearson
dispersion $\hat\phi=\tfrac1M\sum_n(y_n-\hat\lambda_n)^2/\hat\lambda_n
=1+\tfrac1M\sum_n\hat\kappa_n/\hat\lambda_n$ over $M$ points (up
to degrees of freedom) is the quasi-likelihood estimate of the
dispersion factor~\citep{Wedderburn1974,McCullaghNelder1989}; the score test of
$\mathrm{Var}[y\mid x]=\lambda+\tau\lambda^2$ against $\tau=0$
is~\citet{DeanLawless1989} and~\citet{CameronTrivedi1990}, whose
per-point term has conditional mean $\kappa$ of~\eqref{eq:kappa}; for a Gaussian fit
the score test of heteroscedasticity along prespecified directions,
which is also the test of random coefficient variation,
is~\citet{BreuschPagan1979}.
What the paper adds, the reading of these statistics as the value of
the direction the binding restricts and, as its size, the upper end of
the price bounds~\eqref{eq:price_bounds} at the plug-in under a
negative-binomial stand-in matched to $\hat\phi$
(Corollary~\ref{prop:plugin_test}(b); Gaussian analogues in
Table~\ref{tab:heads}, not evaluated here), is set beside each source
in Table~\ref{tab:provenance}.
\input{tables/provenance}

\subsection{Cavities and Leverage}
\label{app:cavities}

\paragraph{Cavities.}
Figure~\ref{fig:cavity} draws the three cavities:

\begin{definition}[Cavity]
\label{def:cavity}
The cavity is the belief that scores plate $n$.
\emph{Sequential}, $q_{<n}(w)$
(predecessors only; the prior is plate $1$'s cavity, so no separate
prior term remains): the prequential score~\citep{Dawid1984},
telescoping exactly to the evidence $-\log Z$
(Proposition~\ref{prop:seq_evidence})---order-dependent per plate, not
batchable. \emph{Leave-one-out},
$q_{(-n)}(w)$ (all plates but $n$): the symmetric LOO predictive
$-\sum_n\log p(y_n\mid y_{\setminus n})$ (GP LOO-CV).
\emph{Shared}: the single belief $q(w)$ for every plate---plate $n$'s
own message treated as flat; the body's objective
(Definition~\ref{def:fsc}). Under a non-shared cavity the binding map
of Definition~\ref{def:routing} projects the posterior under the
cavity's data, one map per plate.
\end{definition}

\input{figures/cavity_graph}

\paragraph{Leverage, leave-one-out, and the site form.}
For the conjugate posterior~\eqref{eq:R1} at fixed
$(\alpha,\sigma_\text{obs})$, write $r_n=y_n-\mu^\top\!\psi_n$ for the
in-sample residual and
$h_n=\sigma_\text{obs}^{-2}\,\psi_n^\top\Sigma\,\psi_n\in[0,1)$ for the
leverage, the hat-matrix diagonal, with
$\sum_nh_n=d_\text{eff}\le H$ the effective number of
parameters~\citep{MacKay1992interp}. The leave-one-out predictive is closed
form~\citep{Sundararajan2001,Williams2006}: residual $r_n/(1-h_n)$ and
variance $\sigma_\text{obs}^2/(1-h_n)$, against the in-sample
$V_n=\sigma_\text{obs}^2(1+h_n)$.
For a free belief whose $\Sigma^{-1}-\alpha I$ lies in the cone spanned
by $\{\psi_n\psi_n^\top\}$, leaving a plate out uses a site form (any
fixed one when several exist) of
expectation propagation~\citep{Minka2001},
$q(w)\propto t_0(w)\prod_n t_n(w)$, with
$t_0(w)\propto\exp(-\tfrac{\alpha}{2}\lVert w-\mu\rVert^2)$ carrying the
prior's precision at the belief's mean (so that $q$ keeps mean $\mu$)
and one Gaussian site per plate at the shared mean,
$t_n(w)\propto\exp\!\big[-(w^\top\!\psi_n-m_n)^2/(2\tau_n)\big]$, so that
\begin{equation}
  \Sigma^{-1}=\alpha I+\sum_n \frac{\psi_n\psi_n^\top}{\tau_n},
  \qquad
  v_{\setminus n}=\frac{v_n}{1-g_n},\quad
  g_n=\frac{v_n}{\tau_n}\in[0,1),\quad
  \sum_n g_n = H-\alpha\operatorname{tr}\Sigma\le H,
  \label{eq:site_form}
\end{equation}
with $v_{\setminus n}$ the plate variance of the leave-one-out cavity
$q_{(-n)}\propto q/t_n$ (Sherman--Morrison) and $g_n$ the site's share
of plate $n$'s belief variance; the bound belief has the covariance of
$\tau_n\equiv\sigma_\text{obs}^2$ (its sites sit at $y_n$, not at $m_n$,
so its leave-one-out mean moves), where $g_n$ is the leverage $h_n$.

\subsection{Proper Scoring}
\label{app:proper_scoring}

\paragraph{Proper scoring rules and proper-score training.}
A scoring rule $S(m,y)$ for a predictive density $m$ is \emph{proper} if
its expectation under $y\sim p$ is minimised at $m=p$, and \emph{strictly}
proper if that minimiser is unique~\citep{Gneiting2007}. The log score
$S(m,y)=-\log m(y)$ is strictly proper: with $m(\cdot\mid x)$ the
predictive, the population data risk decomposes as
\begin{equation}
  \mathbb{E}_{x}\,\mathbb{E}_{y\sim p^\star(\cdot\mid x)}\!\big[-\log m(y\mid x)\big]
  = \mathbb{E}_{x}\big[H(p^\star(\cdot\mid x))\big]
    + \mathbb{E}_{x}\big[\mathrm{KL}\!\left(p^\star(\cdot\mid x)\,\|\,m(\cdot\mid x)\right)\big],
  \label{eq:proper_score}
\end{equation}
Gibbs' inequality, minimised over predictives that can represent the
true conditional $p^\star$ at $m=p^\star$.

\begin{definition}[Proper-score training]
\label{def:score_training}
An objective is \emph{proper-score training} of the predictives
$m_n$ if it has the form $\sum_n S(m_n,y_n)+\Omega$ with $S$ strictly
proper and $\Omega$ free of the observations $y_{1:N}$
(instances: MAP, the sequential-cavity evidence, PPGPR, and the
quadratic-variance head of Table~\ref{tab:objective_cells}). Propriety is a property of the
pointwise loss: at the shared cavity the belief is fitted on the plates
it scores, so proper-score training here is empirical risk
minimisation with a proper loss, not a prequentially proper forecast
of each training plate (Proposition~\ref{prop:cavity_invariance}
bounds its difference from the leave-one-out score at a Gaussian last
layer).
\end{definition}

%% file: figures/factor_graph.tex

\begin{figure}[t]
\centering
\resizebox{0.98\linewidth}{!}{%
\begin{tikzpicture}[
  var/.style={circle, draw=black, very thick, minimum size=3ex},
  fac/.style={rectangle, draw=black, fill=black, minimum size=2ex, inner sep=0pt},
  font=\small
]


\node[fac] (prior)  at ( 0.0, 0) {};
\node[var] (w)      at ( 2.0, 0) {$w$};
\node[fac] (det)    at ( 4.3, 0) {};
\node[var] (fn)     at ( 6.8, 0) {$f_n$};
\node[fac] (lik)    at ( 9.3, 0) {};
\node[var] (yn)     at (10.9, 0) {$y_n$};
\node[fac] (delta)  at (12.5, 0) {};

\draw[very thick] (prior)--(w)--(det)--(fn)--(lik)--(yn)--(delta);


\node[above=10pt of prior, align=center]
  {$p(w)$\\[-1pt]\footnotesize$\mathcal{N}(0,\alpha^{-1}I)$};
\node[above=10pt of det, align=center]
  {$f_n = w^\top\!\psi_n$};
\node[above=10pt of lik, align=center]
  {$p(y_n|f_n)$};
\node[above=10pt of delta, align=center]
  {$\delta(y_n\!-\!y_n^{\rm obs})$};


\draw[dashed, gray!55, rounded corners=5pt, thick]
  (3.2, -1.4) rectangle (13.3, 1.6);
\node[anchor=south east, font=\footnotesize, gray!55] at (13.3, -2)
  {$n = 1,\ldots,N$};


\node[font=\footnotesize] (xn) at (4.3, -1.2) {$x_n$};
\draw[->, dashed, thick, >=stealth] (xn)
  -- node[right=2pt, font=\footnotesize] {$\psi_n\!=\!NN(x_n)$}
  (det.south);


\node[above=7pt, font=\scriptsize, blue!85]
  at ($(det.east)!0.5!(fn.west)$)
  {$\longrightarrow\;\mathcal{N}(\mu^\top\!\psi_n,\,v_n)$};

\node[below=9pt, font=\scriptsize, red!85]
  at ($(fn.east)!0.5!(lik.west)$)
  {$\propto\,p(y_n|f_n)\;\longleftarrow$};

\node[below=9pt, font=\scriptsize, black]
  at ($(yn.east)!0.5!(delta.west)$)
  {$\delta(y_n\!-\!y_n^{\rm obs})\;\longleftarrow$};

\node[above=7pt, font=\scriptsize, black!75]
  at ($(w.east)!0.35!(det.west)$)
  {$\longrightarrow\;q(w)$};

\node[below=7pt, font=\scriptsize, black!75]
  at ($(w.east)!0.35!(det.west)$)
  {$\beta_n(w)\;\longleftarrow$};

\end{tikzpicture}%
}

\caption{%
  Factor graph of the last-layer model on which $F_\text{SC}$
  (Definition~\ref{def:fsc}) is evaluated. Circles: variable nodes;
  filled squares: factor nodes; dashed box: plate $n=1,\ldots,N$. The
  belief $q(w)=\mathcal{N}(\mu,\Sigma)$ sends the forward message
  $\mathcal{N}(\mu^\top\!\psi_n,v_n)$, $v_n=\psi_n^\top\!\Sigma\,\psi_n$,
  to each $f_n$ (\textbf{blue}); the likelihood $p(y_n|f_n)$, Gaussian or
  Poisson, returns $\propto p(y_n|f_n)$ (\textbf{red}); $\beta_n(w)$ is
  plate $n$'s site in the site form~\eqref{eq:site_form}. The graph is a
  tree, and the plate integral
  $Z_n=\int\!\mathcal{N}(f_n;\mu^\top\!\psi_n,v_n)\,p(y_n|f_n)\,\mathrm{d}f_n$
  is one-dimensional: the Gaussian convolution~\eqref{eq:gauss_conv} in
  closed form, and a Gauss--Hermite quadrature for Poisson
  (Appendix~\ref{app:setup}).
}
\label{fig:factor_graph}
\end{figure}

%% file: tables/provenance.tex
\begin{table}[t]
\centering
\scriptsize
\setlength{\tabcolsep}{3pt}
\caption{Provenance: what each object is known as, and what this paper
adds, with the statement that carries the increment.}
\label{tab:provenance}
\begin{tabular}{@{}>{\raggedright\arraybackslash}p{2.3cm}>{\raggedright\arraybackslash}p{5.0cm}>{\raggedright\arraybackslash}p{5.9cm}@{}}
\toprule
object & known as & added here \\
\midrule
the data term $-\log Z_n$ & proper-score training of a belief's predictive at sparse GPs: direct loss minimisation~\citep{Sheth2020,Wei2021}, PPGPR~\citep{Jankowiak2020} & stated on any directed factor graph with every factor scored by its normaliser (Def~\ref{def:fsc}, Prop~\ref{prop:proper_score}); the ELBO's complement written as per-plate couplings to the tilted marginals (Thm~\ref{prop:binding_forms}); MAP and the evidence as its limit and sequential instance (Table~\ref{tab:objective_cells}) \\
\addlinespace[2pt]
the binding & the ELBO's Jensen gap; predictive suboptimality under misspecification~\citep{Masegosa2020,Morningstar2022} & the per-plate couplings isolated as the observation-dependent remainder (Def~\ref{def:binding}) and identified as the term that binds the belief to the posterior (Cor~\ref{cor:posterior_seeking}); priced in general (Thm~\ref{thm:price}), in closed form for the Gaussian instance and, for a negative-binomial truth, bounded above by an explicit KL for the Poisson (Cors~\ref{cor:gap_closed_form}, \ref{prop:plugin_test}) \\
\addlinespace[2pt]
the bound covariance & the Gaussian-VI fixed point~\citep{OpperArchambeau2009} & residual-blind at fixed mean, so the direction the binding restricts is the variance at each input (Cor~\ref{cor:bound_covariance}) \\
\addlinespace[2pt]
the score $\kappa$ & the variance-component score at zero~\citep{DeanLawless1989,CameronTrivedi1990,Lin1997} & the binding cancels exactly this channel at first order (Prop~\ref{prop:overdispersion}); positive on an expressible direction, it makes the price strictly positive for any likelihood under the regularity of Theorem~\ref{thm:price} (Thm~\ref{thm:price}(b)) \\
\addlinespace[2pt]
the screen and the price proxy & the Pearson dispersion and the overdispersion score tests of count regression~\citep{McCullaghNelder1989,DeanLawless1989}; heteroscedasticity score test~\citep{BreuschPagan1979} & read at the plug-in as a screen with an empirically calibrated margin, the upper end of the price bounds at the plug-in under a negative-binomial stand-in as the proxy (Def~\ref{def:screen}, Cor~\ref{prop:plugin_test}(b)); an exploratory, prospectively specified validation on 28 count sets, 27 unused elsewhere (Section~\ref{sec:experiments}, Appendix~\ref{app:screen}) \\
\addlinespace[2pt]
the heads & mean--variance network~\citep{NixWeigend1994}; Poisson-lognormal / negative-binomial regression~\citep{Hinde1982,CameronTrivedi1998}; heteroscedastic GP~\citep{LazaroGredillaTitsias2011} & per-input dispersion $\psi^\top\Sigma\psi$ from one shared belief, the head fixed by the likelihood (Table~\ref{tab:heads}); measured against the direct predictive-family counterparts and the community comparators (Section~\ref{sec:experiments}) \\
\addlinespace[2pt]
the cavities & GP leave-one-out~\citep{Sundararajan2001}; the prequential evidence~\citep{Dawid1984,FongHolmes2020} & the evidence as the sequential instance of $F_\text{SC}$ (Prop~\ref{prop:seq_evidence}); shared and leave-one-out scores agree to $O(H/N)$ per plate at a Gaussian last layer (Prop~\ref{prop:cavity_invariance}), measured \\
\addlinespace[2pt]
$C_\text{mis}$ & Bethe and EP energies~\citep{Yedidia2001,Minka2001} & the Bethe expression at the shared belief as $F_\text{SC}$ plus a mismatch term (Prop~\ref{prop:fsc_construction}), its last-layer closed form (Cor~\ref{cor:cw_form}), and its price, a milder binding, measured \\
\bottomrule
\end{tabular}
\end{table}

%% file: figures/cavity_graph.tex

\begin{figure}
\centering
\begin{tikzpicture}[
  wvar/.style   ={circle, draw=black, very thick, minimum size=3.4ex, font=\small},
  plate/.style  ={draw=black!70, rounded corners=2pt, minimum size=3.2ex, font=\footnotesize},
  scored/.style ={draw=blue!80, very thick, fill=blue!8, rounded corners=2pt, minimum size=3.2ex, font=\footnotesize},
  future/.style ={draw=black!25, text=black!35, rounded corners=2pt, minimum size=3.2ex, font=\footnotesize},
  inmsg/.style  ={->, >=stealth, thick, black!70},
  fwd/.style    ={->, >=stealth, thick, blue!80},
  nomsg/.style  ={->, >=stealth, dashed, black!25},
  ttl/.style    ={font=\footnotesize},
  sub/.style    ={font=\footnotesize, text=black!45},
]
\def\s{0.95}   

\begin{scope}[shift={(0,0)}]
  \node[wvar]   (w)  at (0,1.9)     {$w$};
  \node[plate]  (p1) at (-2*\s,0)   {$1$};
  \node[plate]  (p2) at (-1*\s,0)   {$2$};
  \node[scored] (p3) at ( 0,0)      {$n$};
  \node[future] (p4) at ( 1*\s,0)   {$4$};
  \node[future] (p5) at ( 2*\s,0)   {$5$};
  \draw[inmsg] (p1) -- (w);
  \draw[inmsg] (p2) -- (w);
  \draw[nomsg] (p4) -- (w);
  \draw[nomsg] (p5) -- (w);
  \draw[fwd]   (w)  -- (p3);
  \node[ttl] at (0,-0.95) {(a) sequential $q_{<n}$};
  \node[sub] at (0,-1.45) {plates $<n$ only};
\end{scope}

\begin{scope}[shift={(4.5,0)}]
  \node[wvar]   (w)  at (0,1.9)     {$w$};
  \node[plate]  (p1) at (-2*\s,0)   {$1$};
  \node[plate]  (p2) at (-1*\s,0)   {$2$};
  \node[scored] (p3) at ( 0,0)      {$n$};
  \node[plate]  (p4) at ( 1*\s,0)   {$4$};
  \node[plate]  (p5) at ( 2*\s,0)   {$5$};
  \draw[inmsg] (p1) -- (w);
  \draw[inmsg] (p2) -- (w);
  \draw[inmsg] (p4) -- (w);
  \draw[inmsg] (p5) -- (w);
  \draw[fwd]   (w)  -- (p3);
  \node[ttl] at (0,-0.95) {(b) leave-one-out $q_{(-n)}$};
  \node[sub] at (0,-1.45) {all plates but $n$};
\end{scope}

\begin{scope}[shift={(9,0)}]
  \node[wvar]   (w)  at (0,1.9)     {$w$};
  \node[plate]  (p1) at (-2*\s,0)   {$1$};
  \node[plate]  (p2) at (-1*\s,0)   {$2$};
  \node[scored] (p3) at ( 0,0)      {$n$};
  \node[plate]  (p4) at ( 1*\s,0)   {$4$};
  \node[plate]  (p5) at ( 2*\s,0)   {$5$};
  \draw[inmsg] (p1) -- (w);
  \draw[inmsg] (p2) -- (w);
  \draw[inmsg] (p4) -- (w);
  \draw[inmsg] (p5) -- (w);
  \draw[inmsg] (p3) -- (w);                 
  \draw[fwd]   (w)  to[bend left=32] (p3);  
  \node[ttl] at (0,-0.95) {(c) shared $q(w)$};
  \node[sub] at (0,-1.45) {all plates ($\beta_n\!\approx\!1$)};
\end{scope}

\end{tikzpicture}
\caption{%
  The three cavities of Definition~\ref{def:cavity}: which plates'
  messages $\beta_i(w)$ (black arrows into $w$) form the belief on the
  last-layer weight $w$ that \emph{scores} plate $n$; \textbf{blue} marks
  the scored plate and the forward message
  $\mathcal{N}(\mu^\top\!\psi_n,v_n)$ it receives.
  \textbf{(a)}~\emph{Sequential}: predecessors $1,\dots,n-1$ only; under
  the Gaussian likelihood the per-plate scores sum exactly to $-\log Z$
  (Proposition~\ref{prop:seq_evidence}), each depends on the order, and
  the pass is not batchable.
  \textbf{(b)}~\emph{Leave-one-out}: all plates but $n$, the
  leave-one-out predictive of GP model selection.
  \textbf{(c)}~\emph{Shared}: one belief $q(w)$, formed from all plates
  including $n$, scores every plate, plate $n$'s own message treated as
  flat ($\beta_n\approx1$); the choice of $F_\text{SC}$
  (Definition~\ref{def:fsc}), batchable in one pass and within $O(H/N)$
  per plate of (b) at a Gaussian last layer
  (Lemma~\ref{lem:shared_loo}, Proposition~\ref{prop:cavity_invariance}),
  but not the evidence of (a).
}
\label{fig:cavity}
\end{figure}

%% file: sections/app_statements.tex
\section{Extended Statements}
\label{app:statements}

The formal statements the body omits, in the body's order; proofs in
Appendix~\ref{app:proofs}, measurements with the experiment
appendices.

\subsection{The Binding on a Directed Factor Graph (Section~\ref{sec:general})}
\label{app:statements_objective}

The data-score bound behind proper-score training
(Definition~\ref{def:score_training}) is elementary:

\begin{lemma}[Regulariser perturbation]
\label{lem:reg_perturbation}
Write the objective as $S_N+\Omega$ with $S_N=\sum_n S(m_n,y_n)$ the
data term and $\Omega$ the regulariser, over a feasible set on which
$\Omega_{\inf}:=\inf\Omega>-\infty$, and let $\Delta_\Omega$ be the
infimum over $S_N$-optimal points of $\Omega$ minus $\Omega_{\inf}$.
Then every minimiser $\hat q$ of $S_N+\Omega$ satisfies
$0\le S_N(\hat q)-\inf S_N\le\Delta_\Omega$: its data score is within
$\Delta_\Omega/N$ per plate of the data optimum. Without the lower bound
on $\Omega$, or without an $S_N$-optimal point, the statement is void.
\end{lemma}
\begin{proof}[Proof of Lemma~\ref{lem:reg_perturbation}]
For $q^\star$ $S_N$-optimal,
$S_N(\hat q)\le S_N(\hat q)+\Omega(\hat q)-\Omega_{\inf}\le S_N(q^\star)+\Omega(q^\star)-\Omega_{\inf}$,
and $S_N(\hat q)\ge S_N(q^\star)$; take the infimum over $q^\star$.
\end{proof}

\begin{proposition}[$F_\text{SC}$ is proper-score training of the belief's predictive]
\label{prop:proper_score}
For any observation factor $p(y\mid f)$, with predictive
$m_n(y):=\int p(y\mid f)\,q(f_n)\,\mathrm{d}f$:
\textbf{(i)}~$-\log Z_n=-\log m_n(y_n)$ for every plate;
\textbf{(ii)}~the prior and intermediate terms contain no
observation. Hence $F_\text{SC}$ is proper-score training of
$\{m_n\}$ (Definition~\ref{def:score_training};
Appendix~\ref{app:proper_scoring}).
\textbf{(iii)}~In $-\mathrm{ELBO}$ the binding $B$ collects every
observation-dependent term beyond the plate scores, and the plate term
of $-\mathrm{ELBO}$ itself,
$\mathbb{E}_q[-\log p(y_n\mid f_n)]=-\log Z_n+B_n$, is not a proper
score of $m_n$.
\end{proposition}

On the Bethe side, the same core appears when $F_\text{Bethe}$ is
evaluated at the shared-belief configuration $b(q)$, in general off
the local-consistency polytope:

\begin{proposition}[The Bethe expression at the shared belief]
\label{prop:fsc_construction}
Let $q$ factorise over the variable nodes within each factor's scope,
$q_a=\prod_{i\in a}q_i$ with $q_i$ the belief at node $i$, and let
$b(q)$ be the shared-belief configuration of~\eqref{eq:bethe_general},
$b_i=q_i$ and $b_a=f_aq_a/Z_a$, with $0<Z_a<\infty$ and the expectations
and entropies below finite. Then
\[
  F_\text{Bethe}\big(b(q)\big)=F_\text{SC}(q)+C_\text{mis}(q),
  \qquad
  C_\text{mis}(q)=\sum_a\big(\mathbb{E}_{b_a}-\mathbb{E}_{q_a}\big)[\log q_a]
   -\sum_iH(q_i),
\]
the belief-mismatch terms, one per factor, minus the node entropies.
\end{proposition}
The factorisation holds on the reduced graph, where every factor
touches the single node $w$, and for any mean-field belief, the
construction of Appendix~\ref{app:beyond_last_layer} included. For a
belief correlated within a factor's scope the Bethe configuration uses
the product $\prod_{i\in a}q_i$ where Theorem~\ref{prop:binding_forms}
uses the marginal $q_a$, and the two read different cores.

\begin{remark}[MAP limit]
\label{rem:bethe_map}
Let the factors be bounded and continuous at $w^\star$ and let the
belief concentrate, $q\to\delta_{w^\star}$ weakly. Then
$F_\text{SC}(q)\to-\log p(y,w^\star)$, the MAP objective.
\end{remark}

\begin{corollary}[The remainder at the last layer]
\label{cor:cw_form}
On the reduced graph of Proposition~\ref{rem:deterministic}, for a
Gaussian belief $q=\mathcal{N}(\mu,\Sigma)$, $\Sigma\succ0$, under the prior
$\mathcal{N}(0,\alpha^{-1}I)$, with $\Gamma=(I+\alpha\Sigma)^{-1}$,
$z=\Gamma\mu$, and $\tilde b_n\propto p(y_n\mid f)\,\mathcal{N}(f;m_n,v_n)$,
\begin{equation}
  C_\text{mis}=-\tfrac12\big(\alpha^2 z^\top\Sigma z+\operatorname{tr}\Gamma\big)
  -\tfrac12\log\det(2\pi\Sigma)
  -\tfrac12\sum_n\Big[\frac{\mathbb{E}_{\tilde b_n}[(f-m_n)^2]}{v_n}-1\Big],
  \label{eq:cmis_last_layer}
\end{equation}
and flat tilts give $C_\text{mis}=-H(q)$.
\end{corollary}
The mismatch arm of the experiments trains $F_\text{SC}+C_\text{mis}$ in
this form, analytic for the Gaussian likelihood and a one-dimensional
quadrature for Poisson.

\subsection{The Last Layer: Cavities, Non-Identification, Boundedness (Section~\ref{sec:method})}
\label{app:shared_loo}

\begin{proposition}[The exact corner]
\label{prop:seq_evidence}
Under the Gaussian likelihood and Gaussian belief, let plate $n$ be
scored by its sequential cavity (Definition~\ref{def:cavity}), the conjugate posterior
$(\mu_{<n},\Sigma_{<n})$ on plates $1,\dots,n-1$ (the prior at $n=1$).
Then $F_\text{SC}$ is exactly the negative log-evidence
$-\log p(y\mid X)$, term by term the prequential
predictive~\eqref{eq:T11seq}~\citep{Dawid1984,FongHolmes2020}.
\end{proposition}

At the bound belief both scores are closed
form~\citep{Sundararajan2001}; for the free belief
Proposition~\ref{prop:cavity_invariance} compares them at fixed mean and
Appendix~\ref{app:routing_selection} measures the in-sample optimism.

\begin{lemma}[In-sample vs.\ leave-one-out, closed case]
\label{lem:shared_loo}
Fix $(\alpha,\sigma_\text{obs})$ and let $q=\mathcal{N}(\mu,\Sigma)$ be the
exact Gaussian posterior on all $N$ plates, with in-sample residual
$r_n$ and leverage $h_n$ as in Appendix~\ref{app:cavities}. Then,
per plate, exactly,
\begin{equation}
  -\log p(y_n\mid y_{\setminus n})
  \;-\;\bigl(-\log Z_n\bigr)
  \;=\;
  \frac{r_n^2}{\sigma_\text{obs}^2}\,\frac{h_n}{1-h_n^2}
  \;-\;\tfrac{1}{2}\log\bigl(1-h_n^2\bigr)
  \;\ge\;0 .
  \label{eq:shared_loo_gap}
\end{equation}
Under balanced leverage ($h_n\le C\,d_\text{eff}/N$) and at bounded
residual scale ($r_n^2=O(\sigma_\text{obs}^2)$) the per-plate gap is
$O(1/N)$ and the summed gap $O(d_\text{eff})=O(H)$.
\end{lemma}
Both qualifiers are active: a plate with outlying leverage or residual
pays the exact price~\eqref{eq:shared_loo_gap}, unbounded as
$h_n\!\to\!1$.

Freed, the belief loses what the binding determined; the
unboundedness in (i) is the variance collapse of mean--variance
heads~\citep{NixWeigend1994} reported by~\citet{Seitzer2022,Stirn2023}:

\begin{corollary}[Two non-identifications]
\label{cor:pathologies}
For~\eqref{eq:Lreg} with the belief free:
\textbf{(i)}~only the totals $V_n$ are identified by the data term and,
when a constant shift is expressible ($\psi_n^\top A\psi_n\equiv1$ for
some $A\succeq0$),
their split by the $O(1/N)$-per-plate prior term alone, and
$\mathcal{L}_\text{Gauss}$ is unbounded below along $V_n\to0$ at
interpolation;
\textbf{(ii)}~$\alpha$ survives only in $-\log Z_w$, its gradient is
$O(\alpha^{-2})$ once $\lambda_{\min}(\Sigma)\gg\alpha^{-1}$, and for
large $\alpha$ the prior term decreases in $\alpha$ when
$\operatorname{tr}\Sigma^{-1}>\mu^\top\Sigma^{-2}\mu$ and has a finite
stationary point when the inequality is reversed.
\end{corollary}
A bias feature or a constant $\lVert\psi\rVert$ makes the shift of~(i)
expressible. The floor $\Sigma\succeq\varepsilon I$ closes~(i), and under it the
drift of~(ii) is harmless, $-\log Z_w\to-\log\mathcal{N}(\mu;0,\Sigma)$;
boundedness below, not joint
attainment, is what is claimed:

\begin{proposition}[Boundedness and the representation gauge]
\label{prop:bounded}
The rescaling $(\psi,\mu,\Sigma,\alpha)\mapsto(c\psi,\mu/c,\Sigma/c^2,c^2\alpha)$,
$c>0$, leaves every plate term and every prediction of~\eqref{eq:Lreg}
invariant and lowers $-\log Z_w$ by $H\log c$. On
$\Sigma\succeq\varepsilon I$, $\lVert\psi_n\rVert^2\ge\nu>0$,
$\mathcal{L}_\text{Gauss}\ge\tfrac{H}{2}\log(2\pi\varepsilon)+\tfrac{N}{2}\log(2\pi\varepsilon\nu)$
for all $(\psi,\mu,\alpha)$, so Lemma~\ref{lem:reg_perturbation}
applies with finite $\Delta_\Omega$ whenever an $S_N$-optimal point
exists. In the KL form, \eqref{eq:elbo_identity_ll} without $B$,
the objective is gauge-invariant and bounded by the plate floor alone.
\end{proposition}
The floor on $\lVert\psi\rVert$ is met by non-affine LayerNorm for
non-degenerate samples and contained in practice otherwise
(Appendix~\ref{app:gauge}); in the KL form the prior scale is the
EM form of the type-II update,
$\alpha^\star=H/(\operatorname{tr}\Sigma+\lVert\mu\rVert^2)$~\citep{Tipping2001}.

\paragraph{The shared cavity is not load-bearing.}
For the bound belief the shared and leave-one-out scores differ by
Lemma~\ref{lem:shared_loo}; for the free belief the site
form~\eqref{eq:site_form} gives the leave-one-out cavity in closed form.

\begin{proposition}[Shared and leave-one-out scores at a Gaussian last layer]
\label{prop:cavity_invariance}
Under the Gaussian likelihood at fixed mean, with $r_n^2=O(V_n)$ and
balanced shares $g_n\le CH/N$, uniformly over the feasible set:
\textbf{(a)}~\emph{(free, in the site form)} the free belief's shared and leave-one-out
scores differ per plate by exactly
\begin{equation}
  \tfrac12\log\!\big(1+\delta_n/V_n\big)
  -\frac{r_n^2}{2}\,\frac{\delta_n}{V_n\,(V_n+\delta_n)},
  \qquad \delta_n=v_{\setminus n}-v_n=\frac{v_n g_n}{1-g_n},
  \label{eq:site_loo_gap}
\end{equation}
which is $O(g_n)$, so the two criteria differ by $O(H)$ in total;
\textbf{(b)}~hence the shared and leave-one-out criteria have
minimisers within $O(H/N)$ per plate of each other's infimum, and both
share the population optimum $V^\star$ of
Corollary~\ref{thm:pred_evidence} when it is representable.
\end{proposition}
The statement is about criteria, not parameters: close objective values
do not imply close beliefs, which is why Appendix~\ref{app:ledger}
compares test scores.

\subsection{A Conjugate Model of the Price}
\label{app:routing_selection_theory}

A conjugate linear model separates the binding from the other
candidate for the bound belief's loss, that its hyperparameters are
selected by the evidence and not by the predictive (measured on learned
features in Appendix~\ref{app:routing_selection}). Data follow
\begin{equation}
  y = \beta^\top x + s(x)\,\varepsilon, \qquad
  x\sim P_x \text{ on } \mathbb{R}^H,\quad
  \varepsilon\sim\mathcal{N}(0,1) \perp x,
  \label{eq:rsel_model}
\end{equation}
with $0<a\le s^2(x)\le b$ a.s., $P_x$ of bounded support with a nonsingular second
moment, identity features, and prior $w\sim\mathcal{N}(0,\alpha^{-1}I)$.
For $\zeta=(\alpha,\sigma_\text{obs}^2)$ let $m_\zeta$ be the bound
predictive; $\zeta_\text{ev}$ maximises the evidence and $\zeta_\text{loo}$
minimises the closed-form leave-one-out NLL. The \emph{selection gap}
moves only the criterion within the bound family,
\begin{equation}
  \Delta_\text{sel}
   = \mathbb{E}\big[-\log m_{\zeta_\text{ev}}(y\mid x)\big]
    - \mathbb{E}\big[-\log m_{\zeta_\text{loo}}(y\mid x)\big],
  \label{eq:sel_gap}
\end{equation}
and the binding gap $\Delta_\text{bind}$ of Definition~\ref{def:gaps}
moves only the variance profile, at the fixed mean
$\mu_{\zeta_\text{ev}}^\top x$.

\begin{proposition}[Selection is second-order; binding is first-order]
\label{prop:routing_selection}
In the model~\eqref{eq:rsel_model} with fixed $H$ and $N\to\infty$:
\emph{(i)} the normalised evidence and
leave-one-out criteria converge to the \emph{same} population functional
$L(\zeta)=\mathbb{E}\big[-\log\mathcal{N}(y;\beta^\top x,
\sigma_\text{obs}^2)\big]$; $L$ depends on $\zeta$ through $\sigma_\text{obs}^2$
alone, so with $\zeta$ in a compact set whose noise range contains
$\mathbb{E}[s^2]$ both selected noises converge to $\mathbb{E}[s^2]$,
$\alpha$ is not identified at this order, and $\Delta_\text{sel}\to0$,
quadratically in the selected noises' error since $L$ is stationary at
$\mathbb{E}[s^2]$.
\emph{(ii)} The bound plate variances vanish uniformly,
$v_\zeta(x)=O_P(H/N)$, so $\Delta_\text{bind}\to\Delta_\text{bind}^\infty\le
J:=\tfrac12\big(\log\mathbb{E}[s^2]-\mathbb{E}[\log s^2]\big)$,
the Jensen gap~\eqref{eq:routing_gap} at $V^\star=s^2$, positive unless
$s^2$ is a.s.\ constant, with equality
for representable $s^2$ (Corollary~\ref{cor:gap_closed_form}(a));
$\Delta_\text{bind}^\infty>0$ whenever
$\mathbb{E}[(s^2-\mathbb{E}s^2)\,xx^\top]$ has a positive eigenvalue
(Theorem~\ref{thm:price}(b)).
\end{proposition}
At the selected variance the channel of
Corollary~\ref{thm:pred_evidence} is
$\kappa=\bigl(s^2-\mathbb{E}[s^2]\bigr)/\mathbb{E}[s^2]^2$, open where
the noise is above average.
Figure~\ref{fig:routing_selection_measured} plots a Gaussian-design
instance ($s=0.3+h\lvert x_1\rvert$, outside the bounded hypotheses)
beside the benchmarks, its diamonds at bound minus true profile, within
$0.04$ nats of $\Delta_\text{bind}^\infty$.

%% file: sections/app_proofs.tex
\section{Proofs}
\label{app:proofs}

Proofs of the numbered claims of the main text and of
Appendix~\ref{app:statements}, grouped by section; each uses only
Appendices~\ref{app:background}--\ref{app:statements}, the main
text, and earlier proofs. Throughout, $r_n=y_n-\mu^\top\psi_n$,
$V_n=\sigma_\text{obs}^2+\psi_n^\top\Sigma\psi_n$, and
$-\log Z_w=\tfrac12\log\det P+\tfrac12\mu^\top P^{-1}\mu+\text{const}$ with
$P=\Sigma+\alpha^{-1}I$ (Equation~\eqref{eq:Lscroll}).

\subsection{The Binding on a Directed Factor Graph (Section~\ref{sec:general})}
\label{app:proofs_general}
\label{app:proofs_objective}

\begin{proof}[Proof of Proposition~\ref{prop:proper_score} (Proper-score training)]
(i)~By Definition~\ref{def:fsc},
$Z_n=\int p(y_n\mid f_n)\,q(f_n)\,\mathrm{d}f_n=m_n(y_n)$: each plate
term is the log score of the plate predictive at the observation.
(ii)~The prior terms $-\log Z_{w_l}$ and the intermediate terms
$-\log Z_a$ involve the belief and the model factors only. Strict
propriety is Gibbs' inequality~\eqref{eq:proper_score};
Lemma~\ref{lem:reg_perturbation} applies with $\Omega$ the prior and
intermediate terms, bounded below on the shipped last-layer feasible set by
Proposition~\ref{prop:bounded} (on the open gauge $\Omega_{\inf}=-\infty$
and the lemma is void), so the
bound holds at every minimiser. The intermediate terms vanish on
deterministic backbones (Appendix~\ref{app:deterministic}).
(iii)~By~\eqref{eq:elbo_identity_kl}, plate $n$ enters $-\mathrm{ELBO}$
as $-\log m_n(y_n)+\mathrm{KL}(q(f_n)\,\|\,b_n)$, and
$\mathrm{KL}(q\,\|\,p_\text{lat})$ contains no observation. Kept, the
plate term is $\mathbb{E}_{q}[-\log p(y_n\mid f_n)]$, linear in
$q(f_n)$, so its expectation under any law of $y_n$ is minimised by a
point mass on a minimiser of $f\mapsto\mathbb{E}_{y_n}[-\log p(y_n\mid f)]$,
whatever the truth: a belief whose predictive is the truth but whose
marginal is not concentrated is beaten, strictly when that minimiser is
unique (as for the Gaussian and Poisson log-link factors), so the kept
term is not a proper score of $m_n$.
\end{proof}

\begin{proof}[Proof of Equation~\eqref{eq:coupling} (The plate coupling)]
By the tilt of Section~\ref{sec:general}, $\log b_n=\log p(y_n\mid f_n)+\log q(f_n)-\log Z_n$,
so the log-ratio of the coupling is $-\log p(y_n\mid f_n)+\log Z_n$ and
its expectation under $q(f_n)$ is the stated form. Nonnegativity is
Jensen's inequality on the concave logarithm,
$\log Z_n=\log\mathbb{E}_q[p(y_n\mid f_n)]\ge\mathbb{E}_q[\log p(y_n\mid f_n)]$,
with equality iff $p(y_n\mid f_n)$
is $q(f_n)$-a.s.\ constant. On any scope $s$ of which $f_n=f_n(s)$ is a
deterministic function ($s=w$ and $f_n=w^\top\psi_n$ at the last
layer), the tilt is $p(y_n\mid f_n(s))\,q(s)/Z_n$ with the same $Z_n$,
and the log-ratio $-\log p(y_n\mid f_n(s))+\log Z_n$ depends on $s$
only through $f_n$, so its expectation under $q(s)$ equals the
expectation under the pushforward $q(f_n)$.
\end{proof}

\begin{proof}[Proof of Theorem~\ref{prop:binding_forms} (Decomposition)]
\emph{Step 1, the entropy.} Expanding the KL in~\eqref{eq:elbo},
$\mathrm{ELBO}(q)=\mathbb{E}_q[\log p(y\mid w)]+\mathbb{E}_q[\log p(w)]-\mathbb{E}_q[\log q]
=\mathbb{E}_q[\log p(y,w)-\log q]$, the form used in
Section~\ref{sec:binding}; with $p(y,w)=\prod_a f_a$,
\[
  -\mathrm{ELBO}(q)
  \;=\;\mathbb{E}_q\big[-\log p(y,w)\big]+\mathbb{E}_q[\log q]
  \;=\;\sum_a\mathbb{E}_q\big[-\log f_a\big]\;-\;H(q):
\]
the entropy is the second term of the ELBO, and it is the exact joint
entropy of $q$, so nothing about the graph's structure is used; every
term is finite by $\mathbb{E}_q\lvert\log f_a\rvert<\infty$ and
$\lvert H(q)\rvert<\infty$.
\emph{Step 2, one factor.} $f_a$ depends only on its scope, so
$\mathbb{E}_q[-\log f_a]=\mathbb{E}_{q_a}[-\log f_a]$, and with the tilt
$b_a=f_aq_a/Z_a$,
\[
  \mathrm{KL}(q_a\,\|\,b_a)
  =\mathbb{E}_{q_a}\big[\log q_a-\log f_a-\log q_a+\log Z_a\big]
  =\mathbb{E}_{q_a}[-\log f_a]+\log Z_a ,
\]
finite by $0<Z_a<\infty$ and $\mathbb{E}_q\lvert\log f_a\rvert<\infty$;
hence $\mathbb{E}_q[-\log f_a]=-\log Z_a+\mathrm{KL}(q_a\,\|\,b_a)$,
with $\mathrm{KL}(q_a\,\|\,b_a)=0$ iff $f_a$ is constant on the support
of $q_a$. Substituting into Step~1 and summing gives~\eqref{eq:elbo_identity}.
\emph{Step 3, the binding form.} Split the sum into plates and
non-plate factors. For the plates, $\mathrm{KL}(q_a\,\|\,b_a)$ is the
coupling~\eqref{eq:coupling}: with deterministic couplings the plate
factor depends on its scope only through $f_n$, and the proof
of~\eqref{eq:coupling} replaces $q_a$ by the pushforward $q(f_n)$ at
any depth. Their sum is $B(q)$.
For the non-plate factors, Step~2 in the other direction gives
$\sum_{a\notin\text{plates}}\big(-\log Z_a+\mathrm{KL}(q_a\,\|\,b_a)\big)-H(q)
 =\mathbb{E}_q\big[-\log\textstyle\prod_{a\notin\text{plates}}f_a\big]-H(q)
 =\mathrm{KL}(q\,\|\,p_\text{lat})$
with $p_\text{lat}=\prod_{a\notin\text{plates}}f_a$. It is normalised:
integrating $p(y,w)=\prod_af_a$ over $y$, each plate factor
$p(y_n\mid f_n)$ integrates to one, leaving $\int p_\text{lat}=1$; and
$q\ll p_\text{lat}$ since every factor is positive
(Definition~\ref{def:fsc}). A model given by unnormalised factors
satisfies the hypothesis once its normaliser is absorbed into a prior
factor. This is~\eqref{eq:elbo_identity_kl}. On the reduced graph the only non-plate
factor is the prior, $p_\text{lat}=p(w)$, and the term reads
$-\log Z_w+\mathrm{KL}(q(w)\,\|\,b_0)-H(q)=\mathrm{KL}(q(w)\,\|\,p(w))$
with $b_0=p(w)q(w)/Z_w$.
\end{proof}

\begin{proof}[Proof of Proposition~\ref{prop:fsc_construction} (The Bethe expression at the shared belief)]
At $b(q)$ the tilt is $b_a=f_aq_a/Z_a$, so
$\log b_a-\log f_a=\log q_a-\log Z_a$ and each factor term
of~\eqref{eq:bethe_general} collapses,
\[
  -\mathbb{E}_{b_a}[\log f_a] - H(b_a)
  \;=\; \mathbb{E}_{b_a}\big[\log b_a - \log f_a\big]
  \;=\; -\log Z_a+\mathbb{E}_{b_a}[\log q_a] .
\]
Summing over factors and adding the node terms
$\sum_i(d_i{-}1)H(b_i)=\sum_i(d_i{-}1)H(q_i)$ of $-H_\text{Bethe}$,
\[
  F_\text{Bethe}\big(b(q)\big)
  \;=\; F_\text{SC}(q)+\sum_a\mathbb{E}_{b_a}[\log q_a]+\sum_i(d_i-1)H(q_i).
\]
Writing $\mathbb{E}_{b_a}[\log q_a]=(\mathbb{E}_{b_a}-\mathbb{E}_{q_a})[\log q_a]-H(q_a)$
and using the factorisation, $H(q_a)=\sum_{i\in a}H(q_i)$, so that
$\sum_aH(q_a)=\sum_id_iH(q_i)$, the entropy terms combine to
$-\sum_iH(q_i)$, the stated $C_\text{mis}$. Under the factorisation
$q_a$ is both the marginal of Definition~\ref{def:fsc} and the product
$\prod_{i\in a}q_i$ of the Bethe configuration, so $Z_a$ is the same
object in both. $b(q)$ is not locally consistent in general (the
marginal of $b_a$ equals $q_i$ only where the factor's message is
flat), so the evaluation is the algebraic
expression~\eqref{eq:bethe_general}, not a point of the constrained
problem.
\end{proof}

\begin{proof}[Proof of Corollary~\ref{cor:cw_form} (The remainder at the last layer)]
The deterministic-node contributions cancel
(Appendix~\ref{app:deterministic}, the soft-factor limit), leaving the
$N{+}1$ factors (the prior $a{=}0$ and the plates) over the single node
$w$, so Proposition~\ref{prop:fsc_construction} gives
$C_\text{mis}=\sum_{a=0}^{N}(\mathbb{E}_{b_a}-\mathbb{E}_q)[\log q]-H(q)$.
For a Gaussian $q$, $\log q$ is quadratic in
$w$, so
$(\mathbb{E}_{b_a}-\mathbb{E}_q)[\log q]
=-\tfrac12\big(\mathbb{E}_{b_a}[(w-\mu)^\top\Sigma^{-1}(w-\mu)]-H\big)$.
The prior tilt is $\mathcal{N}(z,\Sigma\Gamma)$ with $z-\mu=-\alpha\Sigma z$,
giving $-\tfrac12(\alpha^2z^\top\Sigma z+\operatorname{tr}\Gamma-H)$. A
plate tilt acts on $f=\psi_n^\top w$ alone and leaves $w\mid f$ as
under $q$; under $q$ the quadratic form is $(f-m_n)^2/v_n$ plus an
independent $\chi^2_{H-1}$, so the plate contributes
$-\tfrac12\big(\mathbb{E}_{\tilde b_n}[(f-m_n)^2]/v_n-1\big)$. Adding
$-H(q)=-\tfrac12\log\det(2\pi\Sigma)-\tfrac H2$
gives~\eqref{eq:cmis_last_layer}.
\end{proof}

\begin{proof}[Proof of Remark~\ref{rem:bethe_map} (MAP limit)]
Every local normaliser is an expectation of its factor under the
belief: $Z_n=\mathbb{E}_q\big[p(y_n\mid f_n(w))\big]$ and
$Z_w=\mathbb{E}_q[p(w)]$ on the reduced graph, and likewise for every
intermediate factor. For a bounded factor continuous at $w^\star$,
weak convergence $q\to\delta_{w^\star}$ gives
$Z_n\to p(y_n\mid f_n(w^\star))$ and $Z_w\to p(w^\star)$, so
$F_\text{SC}(q)=\sum_a(-\log Z_a)\to-\log p(w^\star)-\sum_n\log p(y_n\mid f_n(w^\star))
=-\log p(y,w^\star)$.
\end{proof}

\begin{proof}[Proof of Corollary~\ref{cor:posterior_seeking} (Retaining the binding is posterior seeking)]
The minimum defining $\eta^\star$ is assumed attained, and $\eta^\star$
is any of the minimisers.
\emph{With $B$.} Write $\mathcal{D}=(X,y)$ and let $(\theta,\xi)$ be
fixed, suppressed below, with $0<p(y)<\infty$. By Bayes' rule,
$p(w\mid\mathcal{D})=p(y,w)/p(y)$, so for every $q_\eta\in\mathcal{Q}$
\begin{align*}
  \mathrm{KL}\big(q_\eta\,\|\,p(w\mid\mathcal{D})\big)
  &=\mathbb{E}_{q_\eta}\big[\log q_\eta-\log p(y,w)\big]+\log p(y)\\
  &=-\mathrm{ELBO}(q_\eta)+\log p(y),
\end{align*}
the last step being the ELBO in the form of Step~1 of the proof of
Theorem~\ref{prop:binding_forms}. The evidence $\log p(y)$ does not
depend on $\eta$, so $-\mathrm{ELBO}(q_\eta)$ and the KL to the
posterior have the same minimisers over $\mathcal{H}$, among them
$\eta^\star(\theta,\xi)$ by Definition~\ref{def:routing}; for every
$q_\eta$ satisfying the hypotheses of Theorem~\ref{prop:binding_forms},
$-\mathrm{ELBO}(q_\eta)$ is the right side
of~\eqref{eq:elbo_identity_kl}. No model structure, depth, or
parametrisation of $\mathcal{Q}$ was used.

\emph{Without $B$.} The right side of~\eqref{eq:elbo_identity_kl}
minus $B$ is $S_N(q_\eta)+\Omega(q_\eta)$ with $S_N=\sum_n-\log Z_n$ the plate
scores and $\Omega=\mathrm{KL}(q_\eta\,\|\,p_\text{lat})$, which contains no
observation. Let $\eta^\star$ be interior to $\mathcal{H}$ and the three
terms differentiable there. Since $\eta^\star$ minimises $S_N+\Omega+B$,
$\nabla_\eta(S_N+\Omega)=-\nabla_\eta B$ at $\eta^\star$, so
$\eta^\star$ is a stationary point of $S_N+\Omega$, a necessary
condition for minimising it, iff $\nabla_\eta B(q_{\eta^\star})=0$.
This fails in general. Let $\mathcal{Q}$ contain a direction that
scales every plate variance at fixed plate means, $v_n\mapsto(1+t)v_n$
(at the last layer, $\Sigma\mapsto(1+t)\Sigma$ at fixed $\mu$); along
it $\mathrm{d}B/\mathrm{d}t|_{t=0}=\sum_nv_n\,\partial_{v}B_n$. Under
the Gaussian likelihood,
$\partial_vB_n=r_n^2/(2V_n^2)+v_n/(2\sigma_\text{obs}^2V_n)$ at every
$v_n\ge0$, positive for $v_n>0$; at the conjugate posterior
$\Sigma\succ0$, so $\mathrm{d}B/\mathrm{d}t>0$ whenever some
$\psi_n\neq0$, and the binding map is never stationary for the
objective without $B$.
\end{proof}

\begin{proof}[Proof of Proposition~\ref{prop:overdispersion} (The binding on the overdispersion channel)]
\emph{Regularity conditions.} Those of~\eqref{eq:kappa}
(Appendix~\ref{app:heads_background}), and for the identities in $\ell$ one
of: $\ell=\log p(y\mid\cdot)\in C^3$ with $\ell'''$ bounded; or the
marginal Gaussian and $\ell$ in $C^2$ with $\ell,\ell',\ell''$ of at
most exponential growth, $\lvert g(t)\rvert\le Ce^{c\lvert t\rvert}$.
The Gaussian likelihood meets both; the Poisson log link
($\ell'''=-e^{f}$) only the second. For the averaged claim of~(a),
also the difference quotient of $\mathbb{E}_{q_v}[-\ell]$ at $v=0^+$ is
dominated by a $y$-integrable function, and
$\mathbb{E}_{y\mid x}[\ell'^2]<\infty$. ``Every $y$'' means every $y$
with $p(y\mid f(x))>0$, where $\ell$ is defined.

Fix a plate, its input $x$, and an observation $y$; write $\bar f=f(x)$
for the fixed mean, $q_v$ for the plate marginal with mean $\bar f$ and
variance $v$, and let $\mathbb{E}_{q_v}|t-\bar f|^3=o(v)$. Two
expectations appear and are kept apart: $\mathbb{E}_{q_v}$ integrates
the output $t$ over the belief's marginal at fixed $y$;
$\mathbb{E}_{y\mid x}$ integrates the observation over the data law and
enters only in the averaged claim of~(a).

\emph{One expansion.} For $g\in C^3$ with $|g'''|\le M_3$, Taylor's
theorem about $\bar f$ gives
$g(t)=g(\bar f)+g'(\bar f)(t-\bar f)+\tfrac12g''(\bar f)(t-\bar f)^2+\rho$,
$|\rho|\le\tfrac16M_3|t-\bar f|^3$. Under $\mathbb{E}_{q_v}$ the linear
term vanishes because the mean is $\bar f$, the quadratic term is
$\tfrac12g''(\bar f)\,v$ because the variance is $v$, and the remainder
is $o(v)$:
\begin{equation}
  \mathbb{E}_{q_v}[g(t)]=g(\bar f)+\tfrac12\,g''(\bar f)\,v+o(v),
  \qquad
  \partial_v\,\mathbb{E}_{q_v}[g]\big|_{v=0^+}=\tfrac12\,g''(\bar f).
  \label{eq:var_expansion}
\end{equation}
A derivative in the variance is thus a second derivative in the output,
at the mean. For a Gaussian marginal the same holds for $g\in C^2$
with $g,g',g''$ of at most exponential growth: the kernel solves the
heat equation
$\partial_v\mathcal{N}(t;\bar f,v)=\tfrac12\partial_t^2\mathcal{N}(t;\bar f,v)$,
two integrations by parts, whose boundary terms vanish against the
Gaussian tails, give
$\partial_v\mathbb{E}_{q_v}[g]=\tfrac12\mathbb{E}_{q_v}[g'']$ at every
$v>0$, and dominated convergence makes the right side continuous at
$v=0^+$, with limit $\tfrac12g''(\bar f)$.

\emph{(b) Cancellation.} With $g=\ell=\log p(y\mid\cdot)$,
\eqref{eq:var_expansion} gives $\partial_v\mathbb{E}_q[-\ell]=-\tfrac12\ell''$.
With $g=p(y\mid\cdot)$, $Z_n(v)=\mathbb{E}_{q_v}[p(y\mid t)]
=p+\tfrac12p''v+o(v)$ and $p>0$, so
$\partial_v[-\log Z_n]=-\tfrac12\,p''/p=-\tfrac12(\ell''+\ell'^2)$, since
$p=e^\ell$ gives $p'=\ell'p$ and $p''=(\ell''+\ell'^2)\,p$. Both hold
for every $y$, and the reward's squared-score term $-\tfrac12\ell'^2$
is absent from the data term's derivative. For a likelihood
log-concave in $f$, $\ell''\le0$ and the data term rewards variance
nowhere.
\emph{(a) Penalty.} By~\eqref{eq:coupling}, $B_n=\mathbb{E}_q[-\ell]+\log Z_n$,
so from~(b) $\partial_vB_n=-\tfrac12\ell''+\tfrac12(\ell''+\ell'^2)=\tfrac12\ell'^2\ge0$
for every $y$; averaging over $y\mid x$, with $\partial_v$ and
$\mathbb{E}_{y\mid x}$ exchanged under the two dominations of the
regularity conditions, gives the expected penalty,
and under closure $y\sim p(\cdot\mid\bar f)$, so
$\mathbb{E}_{y\mid x}[\ell'^2]$ is the Fisher information of $f$.
\end{proof}

\subsection{The Binding Price at the Last Layer and Its Test (Section~\ref{sec:method})}
\label{app:proofs_gaussian}

\input{sections/app_deterministic}

\begin{proof}[Proof that $\mathcal{L}=F_\text{SC}$, Equation~\eqref{eq:Lscroll}]
Plate terms: $f_n=w^\top\psi_n$ is linear, so
$q(f_n)=\mathcal{N}(\mu^\top\psi_n,\;\psi_n^\top\Sigma\psi_n)$ exactly
and $Z_n$ is the convolution~\eqref{eq:scroll_msg} for any factor.
Intermediate terms vanish on the deterministic backbone
(Proposition~\ref{rem:deterministic}). Prior term: the single
probabilistic layer gives
$Z_w=\int q(w)\,p(w)\,\mathrm{d}w=\mathcal{N}(\mu;\,0,\,\Sigma+\alpha^{-1}I)$
by the multivariate form of~\eqref{eq:gauss_conv} in $w$. Summing
gives~\eqref{eq:Lscroll}, finite on all of $\Sigma\succeq0$ since
$\Sigma+\alpha^{-1}I\succ0$; the KL form~\eqref{eq:elbo_identity_ll}
needs $\Sigma\succ0$.
\end{proof}

\begin{proof}[Proof of Corollary~\ref{cor:bound_covariance} (The bound covariance is residual-blind)]
\emph{Proposition~\ref{prop:overdispersion} at the last layer.} Under
the Gaussian belief the plate marginal is
$q(f_n)=\mathcal{N}(m_n,v_n)$ with $v_n=\psi_n^\top\Sigma\psi_n$, so,
under the growth condition of Section~\ref{sec:method}, the identity of
Proposition~\ref{prop:overdispersion}(b) holds at every $v_n>0$, not
only at $0^+$ (the heat-equation form of~\eqref{eq:var_expansion}):
$\partial_{v_n}\mathbb{E}_q[-\ell_n]=-\tfrac12\mathbb{E}_q[\ell_n'']$.
With $\partial v_n/\partial\Sigma=\psi_n\psi_n^\top$, the ELBO's data
term has $\partial_\Sigma\sum_n\mathbb{E}_q[-\ell_n]
=\tfrac12\sum_n\mathbb{E}_q[-\ell_n'']\,\psi_n\psi_n^\top$, and the
prior KL has $\partial_\Sigma\mathrm{KL}(q\,\|\,p)=\tfrac12(\alpha I-\Sigma^{-1})$;
their sum vanishes exactly at~\eqref{eq:bound_cov}, the condition
of~\citet{OpperArchambeau2009}. The bound belief minimises
$-\mathrm{ELBO}$ over the family (Corollary~\ref{cor:posterior_seeking}),
and a minimiser exists at an interior $\Sigma\succ0$ whenever every
$\ell_n$ is bounded above in $f$, as for the Gaussian, Poisson and
Bernoulli likelihoods at every observation. The data term is then
bounded below, $\mathrm{KL}(q(w)\,\|\,p(w))\to\infty$ as $\Sigma$
approaches the boundary of the PSD cone or
$\operatorname{tr}\Sigma+\lVert\mu\rVert^2\to\infty$, and
$-\mathrm{ELBO}$ is continuous on $\Sigma\succ0$ under the growth
condition, so its sublevel sets are compact. Hence the bound belief
satisfies~\eqref{eq:bound_cov}.
\emph{(a)} $\ell'=(y-b'(f))/\phi$ and $\ell''=-b''(f)/\phi$, free of
$y$. At fixed $\mu$ the observations enter $\sum_n\mathbb{E}_q[-\ell_n]$
only through $-\sum_ny_nm_n/\phi$, which does not involve $\Sigma$, so
the $\Sigma$-dependent part of $-\mathrm{ELBO}$ is
$U(\Sigma)=\sum_n\mathbb{E}_q[b(f_n)]/\phi
+\tfrac12(\alpha\operatorname{tr}\Sigma-\log\det\Sigma)$,
a function of the design $\{\psi_n\}$ and the means $\{m_n\}$ alone.
In the Cholesky factor $\Sigma=LL^\top$ ($L$ lower triangular with
positive diagonal, a diffeomorphism onto $\Sigma\succ0$), $U$ is strictly
convex: $b''=\mathrm{Var}(y)/\phi\ge0$ makes
$\mathbb{E}_{z}[b(m_n+\psi_n^\top Lz)]$, $z\sim\mathcal{N}(0,I)$, convex in $L$,
and $\lVert L\rVert_F^2$ and $-2\sum_i\log L_{ii}$ are strictly convex.
$U$ is also coercive, since $\mathbb{E}_q[b(f_n)]\ge b(m_n)$ by Jensen's inequality.
So~\eqref{eq:bound_cov} at fixed mean has exactly one solution, the
minimiser of $U$, a function of the design and the mean profile alone. In the terms
of Proposition~\ref{prop:overdispersion}(a),(b),
$-\tfrac12\ell''=\tfrac12\ell'^2-\tfrac12(\ell''+\ell'^2)$, the
binding's penalty minus the score's reward, and the two depend on $y$
through the same term $(y-b'(f))^2/2\phi^2$, which cancels. Under the
model $\mathbb{E}_y[\ell'^2]=b''(f)/\phi$ (Bartlett), so
$\sum_nb''(w^\top\psi_n)\,\psi_n\psi_n^\top/\phi$ is the Fisher
information of $w$ and~\eqref{eq:bound_cov} is its average under $q$
plus $\alpha I$.
\emph{(b)} With $\omega_n=\mathbb{E}_q[-\ell_n'']$,
$\sum_n\omega_nv_n=\operatorname{tr}\big(\Sigma\sum_n\omega_n\psi_n\psi_n^\top\big)
 =\operatorname{tr}\big(\Sigma(\Sigma^{-1}-\alpha I)\big)
 =H-\alpha\operatorname{tr}\Sigma\le H$,
the site-share identity of~\eqref{eq:site_form} with
$\tau_n=1/\omega_n$, the leverages being the shares $g_n=\omega_nv_n$;
hence $v_n=O(H/N)$ uniformly when the shares are
balanced and the weights $\omega_n$ bounded below. For the Gaussian
likelihood $\omega_n=1/\sigma_\text{obs}^2$ and~\eqref{eq:bound_cov}
is~\eqref{eq:R1}.
\end{proof}

\begin{proof}[Proof of Theorem~\ref{cor:price_general} (The price of the binding)]
Notation as in Definition~\ref{def:gaps}.
\emph{Hypotheses.} $\Xi$ is compact and each $p(\cdot\mid f;\xi)$,
$\xi\in\Xi$, is a density (for the Gaussian,
$\Xi\subset[\sigma_-^2,\infty)$ with $\sigma_-^2>0$); $\lVert\psi\rVert$
is bounded; $R^\star>-\infty$, and $\xi\mapsto R_\xi(0)$ is finite and
continuous on $\Xi$. The \emph{flat-leverage limit} is taken along a sequence of
growing training designs under a fixed feature map $\psi$ and profile
class $\mathcal{V}_\psi$, so that $R_\text{free}$ is fixed, along which
$\sup_{x,\,\xi\in\Xi}v^\text{bind}_\xi(x)\to0$ for any selection of
solutions of~\eqref{eq:bound_cov}: only the bound profile varies along
the sequence
(as $\lVert\psi\rVert$ is bounded, $\lambda_{\max}(\Sigma_\xi)\to0$
uniformly on $\Xi$ suffices; Corollary~\ref{cor:bound_covariance}(b)
gives the plate-level rate). The domination
of~\eqref{eq:kappa} holds uniformly: for some $v_0>0$ and $G$ with
$\mathbb{E}_x\mathbb{E}_{y\mid x}G<\infty$,
$\big|\log m(y;f(x),v;\xi)-\log p(y\mid f(x);\xi)\big|\le v\,G(x,y)$
for all $x$, $\xi\in\Xi$ and $v\in[0,v_0]$. For the Gaussian,
$G=(1+r^2/\sigma_-^2)/(2\sigma_-^2)$ with $r=y-f(x)$ serves at every
$v$ once $\mathbb{E}r^2<\infty$.
\emph{Sign and finiteness at every $N$.} Each
$v^\text{bind}_\xi=\psi^\top\Sigma_\xi\psi$ with $\Sigma_\xi\succeq0$ is
expressible and both infima run over $\Xi$, so
$R_\text{bound}\ge R_\text{free}$; and
$R^\star\le R_\text{free}\le\min_\Xi R_\xi(0)<\infty$
by~\eqref{eq:proper_score}.
\emph{(a) Limit.} By the domination,
$\sup_{\xi\in\Xi}\lvert R_\xi(v^\text{bind}_\xi)-R_\xi(0)\rvert
\le\sup_{x,\xi}v^\text{bind}_\xi(x)\,\mathbb{E}G\to0$, and an infimum
moves by at most the uniform distance, so
$R_\text{bound}\to\inf_\Xi R_\xi(0)$, attained at some $\xi^\star$ by
continuity on the compact $\Xi$; hence
$\Delta_\text{bind}\to\Delta_\text{bind}^\infty=R_{\xi^\star}(0)-R_\text{free}\ge0$.
\emph{Zero under closure.} If
$p^\star(\cdot\mid x)=p(\cdot\mid f(x);\xi_0)$ a.s.\ for some
$\xi_0\in\Xi$, then
$R_{\xi_0}(0)=R^\star\le R_\text{free}\le R_{\xi^\star}(0)\le R_{\xi_0}(0)$
by~\eqref{eq:proper_score}, so all are equal and
$\Delta_\text{bind}^\infty=0$.
\emph{(b) Positive when $K$ has a positive eigenvalue.} Let $u$ be a unit
eigenvector with eigenvalue $\lambda>0$, $A=uu^\top$, and
$v_\varepsilon(x)=\varepsilon\,\psi(x)^\top A\psi(x)
=\varepsilon(u^\top\psi(x))^2\in\mathcal{V}_\psi$, below $v_0$ for small
$\varepsilon$ as $\psi$ is bounded. The difference quotients
$[R_{\xi^\star}(v_\varepsilon)-R_{\xi^\star}(0)]/\varepsilon$ have
$x$-integrands, after $\mathbb{E}_{y\mid x}$, bounded by
$(u^\top\psi)^2\,\mathbb{E}_{y\mid x}G$, integrable, and converging
to $-\tfrac12\kappa_{\xi^\star}(x)(u^\top\psi(x))^2$
by~\eqref{eq:kappa} (the same bound makes $K$ finite); by dominated
convergence
$\frac{d}{d\varepsilon}R_{\xi^\star}(v_\varepsilon)\big|_{0^+}
=-\tfrac12u^\top Ku=-\lambda/2<0$, so
$R_\text{free}\le R_{\xi^\star}(v_\varepsilon)<R_{\xi^\star}(0)$ for small
$\varepsilon$.
\emph{(c) Size.} With no $\xi$, $\Xi$ is a point and~(a) reads
$\Delta_\text{bind}^\infty=R(0)-R_\text{free}$.
\emph{(d) Bounds.} The lower inequality is
$R_\text{free}\le R_{\hat\xi}(\hat v)$. For the upper one,
by~\eqref{eq:proper_score} with $R^\star$ finite,
$R_\xi(v)-R^\star=\mathbb{E}_x\mathrm{KL}\big(p^\star(\cdot\mid x)\,\|\,m(\cdot;f(x),v(x);\xi)\big)\ge0$
for every profile, so $R_\text{free}\ge R^\star$ and
$\Delta_\text{bind}^\infty\le R_{\xi^\star}(0)-R^\star
=\mathbb{E}_x\mathrm{KL}\big(p^\star(\cdot\mid x)\,\|\,p(\cdot\mid f(x);\xi^\star)\big)$,
with equality iff $R_\text{free}=R^\star$, that is, iff
$\inf_{\xi\in\Xi,\,v\in\mathcal{V}_\psi}\mathbb{E}_x\mathrm{KL}(p^\star\|m)=0$.
\end{proof}

\begin{proof}[Proof of Corollary~\ref{thm:pred_evidence} (Gaussian instance), with the Gaussian case of Definition~\ref{def:gaps}]
\emph{Reduction.}
The data term of~\eqref{eq:Lreg} is $S_N=\sum_n d(r_n,V_n)$ with
$d(r,V)=\tfrac{r^2}{2V}+\tfrac12\log V$.

\emph{Population optimum.}
Minimising $\mathbb{E}_{x,y}[d(y-f(x),V(x))]$ over a
variance function is pointwise in $x$, at
$V^\star(x)=\mathbb{E}[(y-f(x))^2\mid x]$, with risk
$R(V^\star):=\tfrac12\mathbb{E}\log(2\pi eV^\star)$, the least of any
Gaussian predictive with mean $f$; $V^\star$ is
constant iff the conditional residual variance is homoscedastic.

\emph{The bound profile.}
Under the binding~\eqref{eq:R1},
$v_n=\psi_n^\top(\Psi^\top\Psi/\sigma_\text{obs}^2+\alpha I)^{-1}\psi_n$
(for Diag, the KL-projection
$\sigma_d^2=1/(\Psi^\top\Psi/\sigma_\text{obs}^2+\alpha I)_{dd}$) is
the leverage variance: a function of $(\alpha,\sigma_\text{obs},\Psi)$
alone, independent of the residuals, and $O(1/N)$ as $\Psi^\top\Psi$
grows. Thus $V_n^{\mathrm{bind}}=\sigma_\text{obs}^2+v_n$ is
asymptotically homoscedastic and equals $V^\star$ only when $V^\star$
is constant.

\emph{Part (b).} Fix $\alpha$ and the design, write $V^\star_n=V^\star(x_n)$,
and take as risk conditional on the design
$R_D(V):=\sum_n\mathbb{E}[d(y_n-f(x_n),V_n)\mid x_n]
=\sum_n\big[V^\star_n/(2V_n)+\tfrac12\log V_n\big]$, each summand
uniquely minimised at $V_n=V^\star_n$. The bound totals
$\mathcal{V}_{\mathrm{bind},\alpha}=\{\sigma_\text{obs}^2+v_n:\ \sigma_\text{obs}^2\in\Xi\}$,
$v_n$ as above, are the continuous image of the compact $\Xi$, hence
compact, and lie in the free totals
$\mathcal{V}_{\psi}^{\mathrm{tot}}=\{\sigma_\text{obs}^2+v:\ v\in\mathcal{V}_\psi,\ \sigma_\text{obs}^2\in\Xi\}$.
So $\min_{\mathcal{V}_{\mathrm{bind},\alpha}}R_D\ge\inf_{\mathcal{V}_{\psi}^{\mathrm{tot}}}R_D$,
with equality only if the attained minimiser over
$\mathcal{V}_{\mathrm{bind},\alpha}$ also minimises over
$\mathcal{V}_{\psi}^{\mathrm{tot}}$: strict whenever no minimiser over
$\mathcal{V}_{\psi}^{\mathrm{tot}}$ lies in $\mathcal{V}_{\mathrm{bind},\alpha}$.
A representable $V^\star$ is the unique minimiser over
$\mathcal{V}_{\psi}^{\mathrm{tot}}$, so the gap is positive whenever
$V^\star\notin\mathcal{V}_{\mathrm{bind},\alpha}$; outside the topological closure of
$\bigcup_\alpha\mathcal{V}_{\mathrm{bind},\alpha}$ (whose limits
$\alpha\to\infty$, constant profiles, and $\alpha\to0$ are
residual-independent leverage curves too) it lies outside every
$\mathcal{V}_{\mathrm{bind},\alpha}$, the generic case under
heteroscedasticity at finite $N$ and the guaranteed one as $N$ grows.

\emph{The gaps of Definition~\ref{def:gaps}.} The representation gap
splits into $R_\text{free}-R(V^\star)$, zero for representable
$V^\star$, and $R(V^\star)-R^\star
=\mathbb{E}_x\mathrm{KL}\big(p^\star(\cdot\mid x)\,\|\,\mathcal{N}(f(x),V^\star(x))\big)$
by~\eqref{eq:proper_score}, zero iff the truth is Gaussian with mean
$f$; the free belief removes only the binding gap.

\emph{Part (a), the closed form.}
By Theorem~\ref{thm:price}(a), in its flat-leverage limit and with
$\Xi\subset(0,\infty)$ compact and containing $\mathbb{E}[V^\star]$
(its Gaussian domination holds as $V^\star\le\bar V$),
$R_\text{bound}\to\min_\Xi R_c(0)$, and the plug-in risk
$\tfrac{\mathbb{E}[V^\star]}{2c}+\tfrac12\log(2\pi c)$ is minimised at
$c^\star=\mathbb{E}[V^\star]$ with value
$\tfrac12\log(2\pi e\,\mathbb{E}[V^\star])$, while for representable
$V^\star$ the free belief attains $R(V^\star)$, so
$R_\text{free}=R(V^\star)$; the
difference is~\eqref{eq:routing_gap}, the Jensen gap of the strictly
concave $\log$, zero iff $V^\star$ is a.s.\ constant. For the lower
bound, $V^\star\le\bar V$ gives $(\log)''\le-\bar V^{-2}$ on the support, so
$v\mapsto\log v+v^2/(2\bar V^2)$ is concave and Jensen yields
$\mathbb{E}[\log V^\star]\le\log\mathbb{E}[V^\star]
-\operatorname{Var}[V^\star]/(2\bar V^2)$.
\end{proof}

\begin{proof}[Proof of Corollary~\ref{prop:plugin_test} (Poisson instance)]
\emph{(a)} Write
$(y_n-\hat\lambda_n)^2=(y_n-\lambda_n)^2+2(y_n-\lambda_n)(\lambda_n-\hat\lambda_n)+(\lambda_n-\hat\lambda_n)^2$.
Because the fit scoring $(x_n,y_n)$ did not see it, the conditional
expectation given $x_n$ and the fit leaves $\varsigma_n^2$, zero, and
$(\lambda_n-\hat\lambda_n)^2$; subtracting
$\hat\lambda_n=\lambda_n-(\lambda_n-\hat\lambda_n)$ gives the claim,
and under closure $\varsigma_n^2=\lambda_n$. Dividing by $\hat\lambda_n$ and
averaging,
$\mathbb{E}[\hat\phi]-1=\tfrac1M\sum_n(\varsigma_n^2-\lambda_n)/\hat\lambda_n
+\tfrac1M\sum_n(\delta_n+\hat\lambda_n\delta_n^2)$ with
$\delta_n=(\lambda_n-\hat\lambda_n)/\hat\lambda_n$ the head's relative
error: first order where $\delta_n$ does not average out, and second
order times the rate, so a fixed margin on $\hat\phi$ covers the mean
error only at bounded rates.
\emph{(b)} The Poisson likelihood has no $\xi$, so the upper end
of~\eqref{eq:price_bounds} is
$\mathbb{E}_x\mathrm{KL}(p^\star(\cdot\mid x)\,\|\,\mathrm{Poisson}(\hat\lambda(x)))$,
attained iff expressible Poisson-lognormal predictives approach
$p^\star$ in $\mathbb{E}_x\mathrm{KL}$; the negative-binomial case is a
substitution. For negative-binomial truth with $\phi>1$ it is not
attained: the negative binomial is the gamma mixture of Poissons and
Poisson mixing laws are identifiable, so
$\mathrm{KL}(\mathrm{NB}\,\|\,m(\cdot;f,v))>0$ at every $v\ge0$; it is
continuous in $v$ and diverges as $v\to\infty$, so its infimum over $v$
is positive and $\Pi(\phi)$ strictly exceeds the limiting price. The flat-leverage
hypotheses of Theorem~\ref{thm:price} hold for the log link under
balanced shares when the rates are bounded below, since then
$\omega_n=\mathbb{E}_q[e^{f_n}]$ in
Corollary~\ref{cor:bound_covariance}(b) is. The domination holds with
$G=e^{f(x)}+\tfrac12(y-e^{f(x)})^2$ for $v\le2$ whenever both have finite
expectation: with $s=t-f(x)$, Jensen gives
$\log m-\log p\ge-e^{f}(e^{v/2}-1)\ge-e^{f}v$, and $e^{s}\ge1+s$ gives
$\log m-\log p\le\log\mathbb{E}\,e^{(y-e^{f})s}=\tfrac12(y-e^{f})^2v$.
\end{proof}

\subsection{The Extended Statements (Appendix~\ref{app:statements})}
\label{app:proofs_extended}

\begin{proof}[Proof of Proposition~\ref{prop:seq_evidence} (The exact corner)]
At the sequential cavity the prior is plate~$1$'s cavity, so no
separate prior term remains (Definition~\ref{def:cavity}). By the chain
rule $-\log p(y\mid X)=\sum_n-\log p(y_n\mid y_{<n})$, and
under the Gaussian likelihood each predictive is
$\mathcal{N}\big(\mu_{<n}^\top\psi_n,\,V_n^{\mathrm{seq}}\big)$
by~\eqref{eq:gauss_conv} against the conjugate posterior
$(\mu_{<n},\Sigma_{<n})$ on plates $1,\dots,n-1$, the binding map under
the sequential cavity's data (the prior at $n=1$), term for
term~\eqref{eq:T11seq}.
\end{proof}

\begin{proof}[Proof of Lemma~\ref{lem:shared_loo} (In-sample vs.\ leave-one-out)]
Rank-one downdate of the precision and Sherman--Morrison give the
standard identities~\citep{Williams2006,Sundararajan2001}: LOO residual
$r_n/(1-h_n)$, LOO variance $\sigma_\text{obs}^2/(1-h_n)$, against the
in-sample $V_n=\sigma_\text{obs}^2(1+h_n)$. Substituting both Gaussian
log-densities yields~\eqref{eq:shared_loo_gap}, nonnegative by
$h_n\in[0,1)$; $h_n\le C\,d_\text{eff}/N$ bounds each term by $O(h_n)$
and $\sum_n h_n=d_\text{eff}$.
\end{proof}

\begin{proof}[Proof of Corollary~\ref{cor:pathologies} and Proposition~\ref{prop:bounded} (Non-identification and boundedness)]
\emph{(i) The split:} the data term depends on
$(\Sigma,\sigma_\text{obs})$ only through the totals $\{V_n\}$, so any
reallocation preserving every $V_n$ changes only $-\log Z_w$, $O(1/N)$
per plate; on homoscedastic data $\Sigma=0$,
$\sigma_\text{obs}^2=V^\star$ attains the data optimum, so the
$\Sigma\to0$ drift there is score-optimal. \emph{The collapse:}
$v_n\ge0$ gives $\sigma_\text{obs}^2\le\min_n V_n$; for fixed
$r_n\neq0$ the plate term $d(r_n,V_n)$ (proof of
Corollary~\ref{thm:pred_evidence}) is bounded below:
$\partial_Vd=\tfrac{V-r^2}{2V^2}$ vanishes at $V=r^2$ with
$\partial_V^2d=\tfrac{1}{2r^4}>0$, the unique minimiser. The collapse
therefore requires $r_n\to0$, where the pointwise optimum
$d(r_n,r_n^2)=\tfrac12+\tfrac12\log r_n^2$ decreases without bound. On
near-noiseless data the jointly trained mean can drive $r_n\to0$, and
$\sigma_\text{obs}^2\le\min_nV_n\to0$ follows, the heteroscedastic-NLL
degeneracy~\citep{Seitzer2022,Stirn2023}.
\emph{(ii) Prior precision:} under the binding $\alpha$ enters every
plate variance through $\Sigma(\alpha)$; free, only $-\log Z_w$ depends
on it, through $P=\Sigma+\alpha^{-1}I$ with
$\partial_\alpha P=-\alpha^{-2}I$, so
\begin{equation}
  \frac{\partial\mathcal{L}_\text{Gauss}}{\partial\alpha}
  =-\tfrac12\alpha^{-2}\big(\operatorname{tr}P^{-1}-\mu^\top P^{-2}\mu\big).
  \label{eq:dL_dalpha}
\end{equation}
Once $\lambda_{\min}(\Sigma)\gg\alpha^{-1}$, $P\to\Sigma$ and the
gradient is $O(\alpha^{-2})\to0$ with sign fixed by
$T_\alpha=\operatorname{tr}\Sigma^{-1}-\mu^\top\Sigma^{-2}\mu$: for
large $\alpha$ the prior term decreases along $\alpha\to\infty$ when
$T_\alpha>0$ (not necessarily from every start: it can have a finite
local minimum), and has a finite stationary point when $T_\alpha<0$,
since the bracket is positive as $\alpha\to0$. For $\Sigma=0$, \eqref{eq:dL_dalpha}
gives $\alpha^\star=H/\lVert\mu\rVert^2$, the type-II form.
\emph{Gauge:} under the rescaling $\mu^\top\psi_n$ and
$\psi_n^\top\Sigma\psi_n$ are unchanged, so every plate term is, while
$\mathcal{N}(\mu/c;0,(\Sigma+\alpha^{-1}I)/c^2)=c^H\mathcal{N}(\mu;0,\Sigma+\alpha^{-1}I)$,
i.e.\ $-\log Z_w\mapsto-\log Z_w-H\log c$.
\emph{Lower bound:}
$-\log Z_w\ge\tfrac12\log\det(2\pi\Sigma)\ge\tfrac{H}{2}\log(2\pi\varepsilon)$
on $\Sigma\succeq\varepsilon I$ for every $\mu,\alpha$; each plate term
is $\ge\tfrac12\log(2\pi V_n)$ with
$V_n\ge\varepsilon\lVert\psi_n\rVert^2\ge\varepsilon\nu$; summing gives
the bound, uniform in $(\psi,\mu,\alpha)$ because the feature scale
enters only through $\nu$. The floor truncates the gauge orbit at
$c^2=\lambda_{\min}(\Sigma)/\varepsilon$ (Appendix~\ref{app:gauge}),
and the prior term at any data-optimal belief is finite, so
$\Delta_\Omega<\infty$ when such a point exists. In $\alpha$, $-\log Z_w$ diverges as
$\alpha\to0$ and tends to $-\log\mathcal{N}(\mu;0,\Sigma)$ as
$\alpha\to\infty$, so its infimum is attained at finite $\alpha$ or in
the limit; joint attainment is not claimed.
\emph{KL form:}
$\mathrm{KL}\big(\mathcal{N}(\mu,\Sigma)\,\|\,\mathcal{N}(0,\alpha^{-1}I)\big)
 =\tfrac12\big[\alpha(\operatorname{tr}\Sigma+\lVert\mu\rVert^2)-H-\log\det\Sigma-H\log\alpha\big]$
is invariant under the rescaling (the $\pm2H\log c$ cancel),
nonnegative, and stationary at
$\alpha^\star=H/(\operatorname{tr}\Sigma+\lVert\mu\rVert^2)$; the
objective is then bounded below by the plate floor alone.
\end{proof}

\begin{proof}[Proof of Proposition~\ref{prop:cavity_invariance} (Shared and leave-one-out scores at a Gaussian last layer)]
\emph{(a)} With sites at the shared mean, plate $n$ is scored by
$\mathcal{N}(y_n;m_n,\sigma_\text{obs}^2+v_{\setminus n})$ at the
leave-one-out cavity and by $\mathcal{N}(y_n;m_n,V_n)$ at the shared
one; the difference of the two negative log densities
is~\eqref{eq:site_loo_gap}. For $g_n\le\tfrac12$,
$\delta_n\le2V_ng_n$, so both terms are $O(g_n)$ whenever
$r_n^2=O(V_n)$, and the criteria differ by $O(\sum_n g_n)=O(H)$
by~\eqref{eq:site_form}.
\emph{(b)} Two criteria within $\varepsilon$ everywhere have minimisers
within $2\varepsilon$ of each other's infimum
(Lemma~\ref{lem:reg_perturbation} with the difference as regulariser);
here $\varepsilon=O(H)$ in total, $O(H/N)$ per plate. In population the
shared criterion is minimised at $V^\star$
(Corollary~\ref{thm:pred_evidence}), and the leave-one-out criterion,
whose plate variance differs by $\delta_n=O(v_ng_n)\to0$ uniformly, has
the same limit.
\end{proof}

\begin{proof}[Proof of Proposition~\ref{prop:routing_selection} (Selection is second-order; binding is first-order)]
\emph{(i)} Fix $\zeta$. With $H$ fixed and $\alpha>0$ the posterior mean
converges to
$\mathbb{E}[xx^\top]^{-1}\mathbb{E}[xy]=\beta$ for every $\alpha>0$ and the posterior
covariance is $O_P(H/N)$, so the predictive conditioned on any growing
subset of the plates converges to
$\mathcal{N}(\beta^\top x,\sigma_\text{obs}^2)$; the leave-one-out
predictives condition on $N{-}1$ plates and the prefix predictives'
Ces\`aro average has the same limit (finitely many short prefixes are
$o(N)$ in total). Both normalised criteria converge to $L(\zeta)$,
uniformly on a compact set; $\alpha$ enters only at $O(H\log N/N)$.
$L$ is flat in $\alpha$ and uniquely minimised, with zero derivative, at
$\sigma_\text{obs}^2=\mathbb{E}[s^2]$, inside the compact set by
hypothesis, so both selected predictives
converge to $\mathcal{N}(\beta^\top x,\mathbb{E}[s^2])$ and
$\Delta_\text{sel}\to0$; by stationarity, to second order in the
selected noises' deviation from $\mathbb{E}[s^2]$, plus the
$O(H\log N/N)$ terms in $\alpha$.
\emph{(ii)} $v_\zeta(x)=x^\top(\Psi^\top\Psi/\sigma_\text{obs}^2+\alpha I)^{-1}x=O_P(H/N)$,
so the bound predictive converges to the constant variance
$\mathbb{E}[s^2]$ selected by (i), whose excess over $V^\star=s^2$ is
$J$; this is the upper end of Theorem~\ref{thm:price}(d), attained for
representable $s^2$ by Corollary~\ref{cor:gap_closed_form}(a). At the
selected variance $\kappa=(s^2-\mathbb{E}s^2)/(\mathbb{E}s^2)^2$
(Corollary~\ref{thm:pred_evidence}), so the $K$ of
Theorem~\ref{thm:price}(b) is
$\mathbb{E}[(s^2-\mathbb{E}s^2)\,xx^\top]/(\mathbb{E}s^2)^2$, which gives
the positivity clause.
\end{proof}

%% file: sections/app_deterministic.tex
\label{app:deterministic}
We derive the reduction of Proposition~\ref{rem:deterministic}, used
throughout Section~\ref{sec:method}: on a deterministic backbone the
intermediate factor terms $-\log Z_a$ in the
decomposition~\eqref{eq:fsc_def} vanish, leaving only the prior terms
$-\log Z_{w_l}$ over probabilistic weights and the observation terms
$-\log Z_n$---the reduced graph drawn in
Figure~\ref{fig:factor_graph}. The per-factor objects of
Section~\ref{sec:general} and Theorem~\ref{prop:binding_forms} are
defined for positive-density factors; a deterministic factor is a
$\delta$, whose product with a belief supported on the same constraint
is not a density, so $\int f_a\,q_a$ is undefined. A deterministic
factor $f_a(u,h)=\delta\big(h-\varphi(u)\big)$, the conditional law of
its output $h$ given its input $u$, is instead scored by its
\emph{directed normalisation} against the belief on its input alone,
$Z_a:=\int\!\int f_a(u,h)\,q(u)\,\mathrm{d}u\,\mathrm{d}h=1$ for every
$q(u)$; deterministic subgraphs are collapsed with this $Z_a$ before
$F_\text{SC}$ and the theorem are applied to the remaining factors.

\begin{proof}[Proof of Proposition~\ref{rem:deterministic}]
\emph{Deterministic layer factor.}
In the feedforward factor graph (Figure~\ref{fig:factor_graph}) each intermediate
factor $a$ is an activation transition $p(h_{n,l}\mid h_{n,l-1},w_l)$.
A \emph{deterministic} layer, its weights a parameter
$w_l^\star\in\theta$ with no prior factor, has
\begin{equation}
  f_a(h_{n,l-1},h_{n,l})
   = \delta\!\big(h_{n,l}-\varphi_l(h_{n,l-1})\big),
  \qquad
  \varphi_l(\cdot) := \varrho\big(w_l^\star\,\cdot\big),
  \label{eq:det_factor}
\end{equation}
where $\varrho$ is the fixed elementwise nonlinearity and the affine map uses the
parameter $w_l^\star$.

\emph{Local normaliser.}
The $\delta$ in~\eqref{eq:det_factor} integrates to one over $h_{n,l}$
for every $h_{n,l-1}$, so against the normalised input belief (on the
forward pass the point mass at $\hat h_{n,l-1}$, pushed forward to
$\hat h_{n,l}=\varphi_l(\hat h_{n,l-1})$) $Z_a=1$ and $-\log Z_a=0$, at
every deterministic layer and plate.

\emph{Consequence.}
Only $\sum_l(-\log Z_{w_l})$ (probabilistic weights) and $\sum_n(-\log Z_n)$
(observations) survive in~\eqref{eq:fsc_def}. With the last layer's
weights $w$ the only probabilistic ones and the collapsed couplings
removed, the latent law is $p_\text{lat}=p(w)$ and the prior terms
reduce to $-\log Z_w$, $Z_w=\int p(w)\,q(w)\,\mathrm{d}w$.
The deterministic parameters $\theta=\{w_l^\star\}$ enter solely through the
forward features $\psi_n=\hat h_{n,L-1}$ inside $Z_n$, and are therefore trained
by ordinary backpropagation through the surviving terms---exactly the reduction
instantiated under the last-layer Gaussian belief in
Section~\ref{sec:method}.
\end{proof}

\paragraph{The soft-factor limit.}
Soften~\eqref{eq:det_factor} to
$f_a^\tau=\mathcal{N}\!\big(h_{n,l};\varphi_l(h_{n,l-1}),\tau^2 I\big)$.
The directed normalisation stays $Z_a=1$ for every $\tau$. The
joint-marginal normaliser $\int f_a^\tau q_a$ does not: at the
consistent belief, with $b_l=\mathcal{N}(h_{n,l};\hat h_{n,l},\tau^2 I)$, it is
$(4\pi\tau^2)^{-d/2}$, $d$ the width of $h_{n,l}$, which diverges as
$\tau\to0$. In $F_\text{Bethe}$ this $\tau$-dependence is matched by
the entropy of the activation belief and cancels, and the limit leaves
no residual term, the standard cancellation of deterministic relations
in the Bethe free energy~\citep{Yedidia2001}. $F_\text{SC}$
carries no entropy, so it scores a collapsed coupling by its directed
normalisation.

%% file: sections/app_setup.tex
\section{Experimental Setup}
\label{app:setup}

Every experiment in the paper is specified here: the data sets and
their splits, the metrics, the training protocols
(Table~\ref{tab:protocols}), the selection rule, the arms and what
each principal gap compares (Table~\ref{tab:arms_estimands}), and the
implementation details of the heads. The results appendices
(\ref{app:screen}--\ref{app:beyond_last_layer}) refer back to this one
and add only what is specific to their measurement.

\paragraph{Data sets: Gaussian likelihood.}
Eight UCI data sets ($n$ examples, $d$ features): Yacht ($308$, $6$),
Boston ($506$, $13$), Energy ($768$, $8$), Concrete ($1030$, $8$),
Wine ($1599$, $11$), Power ($9568$, $4$), Kin8nm ($8192$, $8$), Naval
($11934$, $14$). Three large tabular sets, California housing
($20{,}640$, $8$), Protein ($45{,}730$, $9$; OpenML~44963), Year
($515{,}345$, $90$; OpenML~44027), and two frozen deep-embedding sets:
SICK sentence relatedness ($9{,}840$, target in $[1,5]$), the $768$-d
\texttt{[CLS]} state of a frozen \texttt{bert-base-uncased}
cross-encoder, and UTKFace age ($24{,}102$, ages $1$--$116$), the
$2048$-d pooled penultimate features of a frozen ImageNet ResNet-50.
Embeddings are computed once. The out-of-distribution pairs score a
foreign set, through the same frozen encoder for text and vision: SICK
against $1{,}500$ pairs of AG-News texts (text), UTKFace against
$1{,}500$ CIFAR images (vision), and red against the $4{,}898$ white
wines (tabular, the pair of~\citealp{CalvoOrdonez2026}).

\paragraph{Data sets: count likelihoods.}
Generators: inputs $x\in[-1,1]^5$ uniform, mean
$m(x)=1.5+\sin 2x_1+0.5x_2x_3$. The Poisson-lognormal generator draws
$y\sim\mathrm{Poisson}(e^{m+s(x)\varepsilon})$ with
$s(x)=0.1+1.5\,\sigma(3x_1)$; the negative-binomial generator draws a
gamma rate with input-dependent shape $r(x)=e^{2-4\sigma(3x_1)}$; the
closure control is $y\sim\mathrm{Poisson}(e^{m})$. The profile
constants are chosen so that a constant variance costs about $0.1$ nats at
the true mean on the lognormal generator and the best global $r$ costs
$0.15$ on the negative-binomial one. The six real count sets are listed in
Table~\ref{tab:count_datasets}; the twenty-eight count sets of the
screen are listed with their outcomes in Appendix~\ref{app:screen}.
\input{tables/count_datasets}

\paragraph{Splits and preprocessing.}
Gaussian sets: $60/20/20$ train/validation/test (\texttt{sklearn}, run
seed as \texttt{random\_state}; OOD study: a seeded permutation); near-zero-variance features dropped;
inputs standardised on the training fold, targets centred by the
training mean; the two deep-embedding sets additionally
$\ell_2$-normalised per standardised embedding, except in the OOD
study, where they are standardised only so that the input-norm cue
survives. Generators: training sets of $500$, $2{,}000$, and
$20{,}000$ points, validation one fifth of that, test $5{,}000$. Real
count sets: $70/10/20$ per seed, standardised on the training inputs,
constant columns dropped, exposures as features. The two count time
series, bike and fluview, are evaluated on random splits with time
covariates and, in addition, chronologically: first $70\%$ train, next
$10\%$ validation, last $20\%$ test, the seed varying only the
initialisation. All arms of a (data set, seed) pair share the same
split. The fixed-mean controls refit the variance side on the training
split and, in the held-out variant, on a calibration half of the
validation split with early stopping on the other half.

\paragraph{Metrics.}
Test negative log-likelihood in nats per point (Gaussian NLL with the
$\tfrac12\log2\pi$ constant), the paper's primary metric; RMSE in
original units; calibration error, the mean absolute gap between
nominal and empirical central-interval coverage over $19$ levels in
$[0.05,0.95]$; for counts also the probability-integral-transform
Kolmogorov--Smirnov distance; AUROC for the OOD pairs. The
estimand in every test and table is the mean paired difference over
seeds; medians appear only in the fixed-mean tables, where the paired
gap has heavy seed tails, and are labelled there. Paired two-sided
$t$-tests over seeds, without correction for multiple comparisons.

\paragraph{Architecture and optimisation.}
Table~\ref{tab:protocols} gives the protocol of every experiment
family. Common to all last-layer rows: one deterministic hidden layer
without bias, four backbones $\{$relu, relu+LayerNorm, tanh,
tanh+LayerNorm$\}$, Adam with early stopping on the validation
predictive NLL for every probabilistic arm, the ELBO beliefs included
(validation MSE for the Gaussian MAP-trained arms), and the lowest-validation
checkpoint evaluated,
32-bit precision, $\ell_2$ penalty $0.01$ on the backbone. Where
$F_\text{SC}$ is mini-batched the data sum is rescaled by $N/B$ and
the prior term kept at full weight. For the \FSC belief $\alpha$,
$\sigma_\text{obs}^2$ (or $\xi$), $\Sigma$, and the backbone are
optimised jointly; references use a last-layer regularisation weight
$\lambda_\text{ll}$ instead, on the grid
$\{10^{-4},10^{-3},10^{-2},10^{-1}\}$. $\Sigma$ carries a floor
$\varepsilon I$, $\varepsilon=10^{-4}$, which also floors $v_n$
at $\varepsilon\lVert\psi_n\rVert^2$ against
Corollary~\ref{cor:pathologies}(i). With LayerNorm this is about
$0.005$ in the (unscaled) target units, above naval's residual
variance; naval selects relu or tanh, where the floor reaches that
order on part of the inputs (tanh: at most $0.005$; relu: below
$10^{-5}$ at the median test point, $0.002$--$0.05$ on the upper
tenth), so part of the free belief's deficit there can be the floor's. Fit accounting per data
set and seed: an \FSC variant trains four networks (one per backbone),
a $\lambda$-tuned reference sixteen (four backbones $\times$ the
grid). Under the Gaussian likelihood (for counts every arm uses the
summed loss) the backbone penalty is added to the summed loss of the
belief arms and to the mean loss of the references (at half weight for
VBLL), so it weighs about $N$
times more on the references' backbones; the Gaussian Laplace prior
precision is the selected $\lambda_\text{ll}$ itself; the Deep Ensemble
is five MAP networks with one shared fitted noise, selected on its
first member's validation loss; MVN's log-variance is clamped to
$\pm10$. Validation is checked every $5$ steps to step $100$, every
$50$ to $1{,}000$, and every $200$ after; a patience of $50$ checks
never fires within a $3{,}000$-step budget and, under the
$10{,}000$-step budget, ends early exactly the runs whose best
checkpoint falls before step $750$; every arm keeps its
best-validation checkpoint. For counts $\alpha$ is held at its initial
value $1$ for the first $300$ steps, and on the generators at
$N\le2$k, news, and claims the kept \FSC checkpoint falls inside that
hold on $8$--$19$ of $20$ seeds, so its reported $\alpha$ is $1$ there
(fixed anywhere in $[0.1,10^3]$, $\alpha$ moves the test NLL by at most
$0.02$ except on fluview; Appendix~\ref{app:beyond_gaussian}).
Every frozen-feature control and probe runs on all four backbones and
inherits the base run's selection per seed, except the cavity cell and
the selection sweep (relu+LN) and the pinned-noise control, the
screen, and the depth runs (relu). Below the last layer the
belief covers the top $k$ hidden layers (the moment-matched
construction of Appendix~\ref{app:beyond_last_layer}).
\input{tables/protocols}

\paragraph{Validation selection and seeds.}
Every method is validation-selected per (data set, seed) by its own
validation objective: the \FSC belief picks the backbone, the
references also $\lambda_\text{ll}$ or their knob; the heteroscedastic
predictive methods (\FSC, MVN, VBLL, the count and scale heads) by
validation NLL, the Gaussian MAP-trained methods (MAP, Laplace) by
validation MSE, their observation noise being fit post hoc, so that no
validation NLL exists for them at selection time; for counts every arm,
the Poisson head and Laplace included, is selected by validation NLL. Headline rows and the controls use $20$ seeds ($5$--$24$); the
cavity cell, the selection sweep, the OOD study (VBLL on text and vision: all twenty), and the depth runs use the
first ten, the moment-matching audit the first three (Table~\ref{tab:protocols});
the prospectively specified screen of Appendix~\ref{app:screen} reports the same
twenty (its protocol scored the first ten); the probe seeds $0$--$4$ are never reported.

\paragraph{Arms.}
\emph{Beliefs}, on every likelihood: the $F_\text{SC}$ belief with
Full ($LL^\top$), Diag, and None ($\Sigma=0$) covariance, the prior
precision learned in the pass, fixed on a grid, or dropped (the head
with weight decay), and with $\mathrm{KL}(q\,\|\,p)$ in place of
$-\log Z_w$ (the PPGPR form); the ELBO belief, early-stopped as every
other arm (Full; for counts also Diag, trained to the step budget, and
with the KL term tempered); the Laplace belief~\citep{Daxberger2021}, a
post-hoc Gaussian about the MAP solution (full covariance for the
Gaussian likelihood, Poisson generalised Gauss--Newton at the trained
mean for counts with its prior precision selected on validation over
$\{10^{-3},\dots,10^{3}\}$ and the bias outside the covariance); the mismatch
cell, $F_\text{SC}+C_\text{mis}$ with the last-layer closed form of
Corollary~\ref{cor:cw_form}, whose Gaussian UCI entry in
Table~\ref{tab:binding_gap} is trained, with its paired $F_\text{SC}$
arm, under the count protocol ($H{=}64$, batch $512$, Adam $0.01$,
${\le}3{,}000$ steps) on standardised targets with a mean bias, not
under the Gaussian benchmark protocol; VBLL~\citep{Harrison2024}, the same
last-layer family trained by the ELBO with a fixed prior and its
prior-KL term down-weighted by the tuned $\lambda_\text{ll}$ (the
binding $B$ is untouched; the untempered ELBO arm gives $+0.04$ on the
UCI median); the exact corner, the
evidence-trained neural-linear model; and, for counts, the HMC
posterior at frozen features. \emph{Fixed-mean controls}: features
and mean of $F_\text{SC}$-Full frozen, $\Sigma$ refit free by the
score, closed to the conjugate image~\eqref{eq:R1} or bound by the
ELBO over $(\Sigma,\alpha)$, or replaced by the full conjugate map.
\emph{Heads}, the direct predictive-family counterparts and the
community comparators: for the Gaussian likelihood MAP (deterministic
training with $\ell_2$), the heteroscedastic mean--variance network
MVN with two linear heads on the shared backbone trained by
$\beta$-NLL~\citep{Seitzer2022} with $\beta{=}0.5$ ($\beta{=}0$ is the
Gaussian-NLL net of~\citealp{NixWeigend1994}), $\lambda$-tuned; for
counts the Poisson head, the negative
binomial with a global or a free per-input dispersion, and
Poisson-lognormal with a free per-input dispersion.
\input{tables/mean_baseline_note}

\paragraph{Arms and estimands.}
Table~\ref{tab:arms_estimands} lists, for every principal gap of the
paper, what the two sides optimise and hold fixed, whether the
binding map is imposed, and in which sense the statement holds; the
controlled gaps (fixed mean, fixed features) isolate the binding, the
end-to-end arms are what a practitioner obtains.
\input{tables/arms_estimands}

\paragraph{Heads, positivity, quadrature.}
Every arm shares the backbone and a linear mean head. The belief arms
carry $v_n=\psi_n^\top\Sigma\psi_n+\varepsilon\lVert\psi_n\rVert^2$ with $\Sigma=LL^\top$
(Full), a diagonal of squared entries (Diag), or $0$ (None), positive
by construction. The direct predictive-family counterparts output the
per-input parameter through a second linear head (a softplus for the
Poisson-lognormal variance): $v(x)=\mathrm{softplus}(\psi^\top u+c)$ (initialised
at $c=-2$); the negative-binomial comparators output
$\log r(x)=\psi^\top u+c$ (per-input) or a scalar $\log r$ (global);
these dispersion heads carry no penalty.
The Poisson-lognormal
normaliser uses Gauss--Hermite quadrature with $20$ nodes in training
and $40$ at evaluation, centred at the mode of the integrand by $25$
safeguarded Newton steps (step size clipped at $3$; plain quadrature
fails on large counts and tail points), rates clamped at $e^{30}$; the
distribution function for PIT uses a trapezoid rule on the
standardised latent; the negative-binomial log-mass is evaluated in
double precision with $\log r$ clamped to $\pm20$. Against adaptive
quadrature of the log-scaled integrand on $384$ cases (counts to
$2{\times}10^{4}$, rates from $0.1$ to $2{\times}10^{4}$, belief
variances from $10^{-4}$ to $4$, of which $228$ are tail cases with the
count far from the rate and $-\log Z$ up to $2{\times}10^{5}$ nats), the
implemented normaliser is within $8{\times}10^{-7}$ nats at $40$ nodes
and $3{\times}10^{-5}$ at $20$ in double precision; the runs compute in
single precision, where the per-point error has median
$6{\times}10^{-4}$ and reaches about $2{\times}10^{-2}$ nats (platform-dependent) at counts near
$3{,}000$. For counts, the frozen bound arm refits $\Sigma$ and $\alpha$
by the ELBO at the free belief's features and mean ($2{,}000$ Adam
steps, rate $0.01$), the frozen free arm refits $\Sigma$ by the score
at the trained $\alpha$ with early stopping on validation; both start
from $L=0.05\,I$. In the benchmark and probe tables
no run diverged or produced non-finite values, and every cell is the
full seed count; the screen of Appendix~\ref{app:screen} is not covered
by this statement (its metro set is void: ten seeds, non-finite
outcomes).

\paragraph{The representation gauge.}
\label{app:gauge}
The rescaling of Proposition~\ref{prop:bounded} is closed per
configuration: non-affine LayerNorm fixes $\lVert\psi_n\rVert^2$ for
every non-degenerate sample (the floor $\nu$ remains a hypothesis when
the pre-activation variance vanishes); tanh is bounded and not
positively homogeneous; for relu without LayerNorm the $\ell_2$ penalty
makes the criterion along the direction $-H\log c+\lambda_\theta c^2A$,
coercive. The learned $\alpha$ is interpretable only relative to the
fixed architecture and regularisation convention; the predictive mean
and total variance are gauge-invariant. The dependence is exactly an
entropy: under $w=Au$ with $A$ invertible, densities in $w$ pick up
$|\det A|$, so $\mathrm{KL}(q\,\|\,p)$, the expected log-likelihood,
and every coupling $\mathrm{KL}(q_a\,\|\,b_a)$ are unchanged, while
$-\log Z_w=-\log\int p\,q$ and $H(q)$ both shift by $\log|\det A|$
from $u$ to $w$;
by Theorem~\ref{prop:binding_forms}, $F_\text{SC}-H(q)$ is
invariant. The plate scores, the binding,
and the expressible class $\{\psi^\top\Sigma\psi\}$ are invariant under
invertible linear maps of the features; the bound image~\eqref{eq:R1}
is not, through its isotropic prior, so the binding price is
representation-dependent through the bound side alone.

%% file: tables/count_datasets.tex
\begin{table}[t]
\centering\scriptsize
\setlength{\tabcolsep}{3pt}
\caption{The real count data sets. $N$ is the number of rows used, $d$
the number of features after preprocessing; Var/mean is the marginal
target dispersion. Every set is split $70/10/20$ train/validation/test
per seed with the seed as random state, identically for every arm; the
two time series are in addition split chronologically in the same
proportions (Table~\ref{tab:appendix_probe_counts_chrono});
inputs are standardised on the training fold and constant columns
dropped; exposures enter as features, not offsets.}
\label{tab:count_datasets}
\begin{tabular}{@{}p{1.4cm}p{2.7cm}rrp{2.1cm}p{4.5cm}@{}}
\toprule
data set & source & $N$ & $d$ & target (range; Var/mean) & features, preprocessing, structure \\
\midrule
bike & Bike Sharing Demand, hourly (OpenML 44063) & 17{,}379 & 11 & rentals per hour (1--977; 174) & season, year, month, hour, holiday, working day, weather, temperature, felt temperature, humidity, wind; one hourly series, rows treated as independent under a random split \\
\addlinespace[2pt]
fluview & CDC FluView, national weekly ILI counts & 1{,}343 & 9 & ILI visits per week (165--222{,}652; $4\times10^4$) & two seasonal harmonics of the epiweek, log patients seen (exposure), four lagged log counts; one weekly series, lag windows overlap between rows \\
\addlinespace[2pt]
abalone & OpenML 183 & 4{,}177 & 10 & rings (1--29; 1.05) & seven measurements, sex one-hot; one row per animal \\
\addlinespace[2pt]
news & Online News Popularity (OpenML 4545) & 20{,}000 of 39{,}644 & 58 & article shares (1--652{,}900; $2.9\times10^4$) & the numeric attributes; url and timedelta dropped; one fixed random subsample, the same for every seed; one row per article \\
\addlinespace[2pt]
claims & freMTPL2freq (OpenML 41214) & 30{,}000 of 678{,}013 & 47 & claim count per policy (0--11; 1.21; $95\%$ zeros) & area, brand, fuel, region one-hot; log exposure and log density as features; policy id dropped; one fixed random subsample, the same for every seed; one row per policy \\
\addlinespace[2pt]
length of stay & Microsoft hospital length-of-stay sample & 20{,}000 of 100{,}000 & 28 & days (1--17; 1.39) & clinical and demographic attributes, facility one-hot; encounter id and dates dropped; one fixed random subsample, the same for every seed; one row per encounter \\
\bottomrule
\end{tabular}
\end{table}

%% file: tables/protocols.tex
\begin{table}[htp]
\centering\scriptsize
\setlength{\tabcolsep}{2.5pt}
\caption{Training protocols. One row per experiment family; every trained result in the paper is produced under exactly one row (no network trained: the synthetic curves of Figure~\ref{fig:routing_selection_measured} and the Poisson-GLM calibration study of Appendix~\ref{app:screen}; Figure~\ref{fig:horizon_toy} is one fixed synthetic draw). ``Four backbones'' is validation selection per (data set, seed) over $\{$relu, relu+LN, tanh, tanh+LN$\}$, inherited from the base run in the frozen-feature rows; ``patience'' counts validation checks; seeds are the run seeds (the probe seeds $0$--$4$ are never reported).}
\label{tab:protocols}
\begin{tabular}{@{}>{\raggedright\arraybackslash}p{2.55cm}>{\raggedright\arraybackslash}p{3.05cm}>{\raggedright\arraybackslash}p{2.6cm}>{\raggedright\arraybackslash}p{3.5cm}>{\raggedright\arraybackslash}p{1.55cm}@{}}
\toprule
experiment (where) & data & network, belief & optimisation & seeds \\
\midrule
\multicolumn{5}{l}{\emph{Gaussian likelihood}} \\
benchmarks (Tab.~\ref{tab:headline_nll}, \ref{tab:binding_gap}; App.~\ref{app:gaussian_full}) & eight UCI sets & one hidden layer, $H{=}50$, four backbones; Full, Diag, None & full-batch Adam $0.03$, ${\le}10{,}000$ steps, patience $50$; references over $\lambda_\text{ll}$; the $+C_\text{mis}$ column of Tab.~\ref{tab:binding_gap} under the count row below & $20$ ($5$--$24$) \\
at scale (Tab.~\ref{tab:table_scale_nll}, \ref{tab:binding_gap}) & california, protein, year; SICK, UTKFace embeddings & $H{=}100$, four backbones & Adam $0.03$, batch $1{,}024$, ${\le}3{,}000$ steps ($9{,}000$ year), data sum rescaled by $N/B$ & $20$ \\
fixed-mean controls (Tab.~\ref{tab:headline_nll}, \ref{tab:table_scale_nll}, \ref{tab:binding_gap}; App.~\ref{app:routing_selection}) & the two rows above & base run's features and mean, four backbones & variance side refit ${\le}5{,}000$ steps: free (score), closed~\eqref{eq:R1}, conjugate, evidence; held-out variant on a calibration half of validation & $20$ \\
pinned noise (App.~\ref{app:routing_selection}) & eight UCI sets; two Gaussian generators ($x\in\mathbb{R}^8$, $N{=}2{,}000$, noise variance $0.25$ or $0.05+0.5x_1^2$) & relu, $H{=}50$; base run's features and mean & as the benchmarks; bound arm refit ${\le}5{,}000$ steps with $\sigma_\text{obs}$ held at the free arm's value & $20$ \\
selection sweep (Fig.~\ref{fig:routing_selection_measured}b) & UCI & relu+LN, $H\in\{12,25,50,100,200\}$ & as the benchmarks; variance side refit ${\le}5{,}000$ steps, evidence and leave-one-out criteria in closed form & $10$ \\
prior term (Tab.~\ref{tab:binding_gap}; App.~\ref{app:prior_ablation}) & UCI, scale & headline protocol; prior term learned, on a grid, replaced by weight decay, or $\mathrm{KL}(q\,\|\,p)$ in place of $-\log Z_w$ & as the benchmarks; tuned arms over architecture${\times}$knob & $20$ \\
cavity cell (\S\ref{sec:experiments}; App.~\ref{app:ledger}) & UCI & relu+LN, $H{=}50$, own \FSC-Full base; its features and mean frozen & as the benchmarks; variance side refit ${\le}5{,}000$ steps at the shared and leave-one-out cavity & $10$ \\
OOD (App.~\ref{app:ood}) & SICK/AG-News, UTKFace/CIFAR, red/white wine & frozen encoder; $H{=}100$ ($50$ wine), four backbones${\times}\lambda$ & Adam $0.03$, batch $1{,}024$ (full batch on wine and for the closed rows), ${\le}3{,}000$ steps; embeddings standardised, not $\ell_2$-normalised; conjugate refit by evidence at the trained features & $10$ ($20$ VBLL text, vision) \\
\addlinespace[2pt]
\multicolumn{5}{l}{\emph{Count likelihoods}} \\
generators and real sets (Tab.~\ref{tab:binding_gap}, \ref{tab:table_heads_counts}; App.~\ref{app:beyond_gaussian}) & three count generators at $N{=}500$, $2{,}000$, $20{,}000$; six real count sets; for the Gaussian mismatch cell the eight UCI sets (standardised targets) & $H{=}64$, four backbones; Full, Diag, None; heads over $\lambda$ & Adam $0.01$, batch $512$ ($64$ at $N{=}500$), ${\le}3{,}000$ steps, patience $50$, no stop and $\alpha$ held for the first $300$ steps (the kept checkpoint may fall inside) & $20$ \\
chronological splits (App.~\ref{app:beyond_gaussian}) & bike, influenza & as above & as above; rows in time order, seeds vary the initialisation & $20$ \\
HMC at frozen features (App.~\ref{app:beyond_gaussian}) & count generators ($N{=}500$; the overdispersed ones also $2{,}000$); bike, influenza (diagnostics only) & \FSC-Full or ELBO-run features, retrained in the run, four backbones (full batch at $N{=}500$); bias prior $\mathcal{N}(0,10^2)$ & four chains ${\times}$ $1{,}000$ draws after $800$ warm-up, diagonal mass from the Gauss--Newton precision, reset to the inverse sample variance mid-warm-up, acceptance $0.8$; doubled chains ($1{,}600$ warm-up, $2{,}000$ draws) as a check & $20$ \\
count screen (\S\ref{sec:experiments}; App.~\ref{app:screen}) & twenty-eight count sets outside the count benchmarks (wine quality pools the UCI Wine pair); the closure generator; nineteen semi-synthetic re-simulations & relu, $H{=}64$ ($16$, $256$ in the capacity check) & as the generators row (closure generator at $N{=}500$ full batch; capacity check also $1{,}000$ steps); Poisson plug-in fitted on two halves of the training split & $20$ ($10$ metro) \\
\addlinespace[2pt]
\multicolumn{5}{l}{\emph{Below the last layer (App.~\ref{app:beyond_last_layer})}} \\
belief depth & five UCI sets (audit); count generators, bike, influenza & three hidden layers (four in the full-belief column), relu, $H{=}50$, full batch; diagonal belief on the top $k$ layers, point weights under a learned precision below & Adam $0.01$, ${\le}10{,}000$ steps, patience $50$; for counts the ELBO also to the step budget & $10$ ($3$ audit) \\
\bottomrule
\end{tabular}
\end{table}

%% file: tables/mean_baseline_note.tex
\paragraph{Sanity floor.}
The unconditional train-mean predictor---$\mathcal{N}(\bar y_\text{tr},\widehat{\mathrm{Var}}[y_\text{tr}])$, no inputs used---attains test NLL yacht $4.17$, concrete $4.22$, energy $3.73$, kin8nm $0.08$, naval $-2.80$, power $4.26$, wine $1.20$, boston $3.64$ (20 seeds, same protocol).
Every method and reference in Table~\ref{tab:headline_nll} clears this
floor on every dataset, with two exceptions on near-noiseless naval
(\FSC-Full and VBLL; Corollary~\ref{cor:pathologies}(i) and, for \FSC-Full,
the $\varepsilon I$ floor above); the unranked frozen, $\Sigma$-bound control row
does not clear it on yacht and naval.

%% file: tables/arms_estimands.tex
\begin{table}[htp]
\centering\scriptsize
\setlength{\tabcolsep}{2pt}
\caption{What each principal gap compares: the objective each arm minimises, what it holds fixed, whether the posterior (binding) map is imposed, which prior term it carries, and the sense of the statement. ``Score'' is the plate log score $\sum_n-\log Z_n$; ``features'' the backbone; ``conj.'' the conjugate Gaussian posterior~\eqref{eq:R1}.}
\label{tab:arms_estimands}
\begin{tabular}{@{}p{2.9cm}p{1.8cm}p{1.7cm}p{2.5cm}p{1.5cm}p{2.1cm}@{}}
\toprule
gap / arm & objective & fixed & map imposed & prior term & sense \\
\midrule
\multicolumn{6}{l}{\emph{Theory}} \\
Def~\ref{def:gaps} binding gap & population score & mean, prior scale & bound: conj.\ $\Sigma$ & none (data term) & population \\
Cor~\ref{thm:pred_evidence}(b) & score on the design & mean, design & bound: conj.\ $\Sigma$ & none & finite $N$, design \\
Thm~\ref{cor:price_general} $\Delta_\text{bind}^\infty$ & profile risk $R_\xi$ & mean, prior scale & bound = plug-in ($v{=}0$) & none & flat-leverage limit \\
Cor~\ref{cor:gap_closed_form}(a) & profile risk & mean & bound = plug-in & none & limit, Gaussian \\
\addlinespace[2pt]
\multicolumn{6}{l}{\emph{Gaussian, Table~\ref{tab:binding_gap} and Appendix~\ref{app:gaussian_full}}} \\
exact corner & evidence & -- & conj.\ mean $+\Sigma$ & evidence & finite $N$, end-to-end \\
frozen closed $-$ frozen free & score & features, mean & closed: conj.\ $\Sigma$; free: none & $-\log Z_w$ & finite $N$, fixed mean \\
frozen conjugate / -Ev & score / evidence & features & conj.\ mean $+\Sigma$ & $-\log Z_w$ / evidence & finite $N$ \\
VBLL, ELBO cell & ELBO & -- & by the objective (Corollary~\ref{cor:posterior_seeking}) & KL & finite $N$, end-to-end \\
Laplace & MAP $+$ curvature & -- & post hoc & Gaussian log prior & finite $N$ \\
\addlinespace[2pt]
\multicolumn{6}{l}{\emph{Counts, Table~\ref{tab:binding_gap} and Appendix~\ref{app:beyond_gaussian}}} \\
frozen bound $-$ frozen free & score / ELBO refit & features, mean & bound: ELBO over $(\Sigma,\alpha)$ & $-\log Z_w$ / KL & finite $N$, fixed mean \\
certificate $R(0)-R(\hat v)$ & score & features, mean & plug-in ($v{=}0$) vs free & none & finite $N$, fixed mean \\
ELBO belief (early-stopped / budget) & ELBO & -- & by the objective & KL & finite $N$, end-to-end \\
HMC at frozen features & posterior sampling & features & exact posterior & prior $\alpha$ & finite $N$, fixed features \\
test $\hat p$ (Cor.~\ref{prop:plugin_test}, Def.~\ref{def:screen}) & -- & plug-in & -- & -- & upper end for $\Delta_\text{bind}^\infty$, NB stand-in \\
\bottomrule
\end{tabular}
\end{table}

%% file: sections/app_screen.tex
\section{The Channel Screen, Prospectively}
\label{app:screen}

\paragraph{Protocol.}
The channel screen of Definition~\ref{def:screen} was specified
in writing, and committed to the repository with its hash, before any
model was run on the data sets below (the first batch's files were
downloaded, and their marginal dispersions read, twelve minutes before
that commit, the second batch's seconds before the commit fixing its
price rule, the third batch's after it); the protocol is not public, so
we call the evaluation prospectively specified and not pre-registered.
Twenty-seven of the sets are used nowhere else in the paper; the
twenty-eighth, wine quality, pools the red wines of the UCI Wine
regression set and the white wines of its OOD pair
(Appendix~\ref{app:setup}), with the quality score as the count.
The rule: the Poisson plug-in head is fitted
twice on disjoint halves of the training split, each training point is
scored by the half that did not see it, and from those residuals
$\hat\phi$, $T=\sum_n\hat\kappa_n/\sqrt{\sum_n\hat\kappa_n^2}$, and
$\hat p$ are computed as in
Definition~\ref{def:screen}. $T$ is a self-normalised form of the
classical overdispersion score
statistic~\citep{DeanLawless1989,CameronTrivedi1990}, without a
degrees-of-freedom correction, asymptotically standard normal under
closure with an exact plug-in; by Corollary~\ref{prop:plugin_test}(a) a
fitted plug-in shifts its null, which is why the call uses a margin on
$\hat\phi$ and not the significance of $T$ alone. The half-fits split
rows at random, which assumes exchangeable rows; validation stays for
checkpoint and
$\lambda$ selection, the test split for outcomes. Outcome: the binding
gap, frozen-bound minus frozen-free at fixed features and mean
(Table~\ref{tab:binding_gap}'s quantity), paired over seeds; open when
its median over seeds exceeds $0.05$ nats and the $95\%$ $t$-interval of its mean lies above zero. The call
is scored against this outcome per set by the majority over seeds; an
evenly split call is an abstention, and a set is void when any outcome
is non-finite on any seed.

\paragraph{Changes and deviations.}
A pilot on three seeds checked the pipeline and was not scored. Two
things changed after the pilot: the primary call, which subtracted a
half-fit estimate of the plug-in's own mean error, failed its closure
check (a fitted head's excess is bias, which both halves share), and
the fixed margin $\hat\phi\ge1.35$, specified in advance as the fallback,
was promoted to primary before the scored seeds were scored, the margin
having been set from the paper's closure generator; and points whose
plug-in rate is below $10^{-6}$
(deterministic zeros: non-functioning hours on the Seoul bike set) are
excluded from the screen. The price rule $\hat p$ was derived on the
first batch of thirteen sets after their scoring and fixed before the
second batch of seven; the first batch is therefore the derivation set
for the price and the second its test. A third batch of eight sets,
chosen by name and source to be rich in sets expected to be closed
(goals in two further football leagues, Prussian horse-kick deaths,
baseball runs per team-season), was specified with both rules fixed
and with the commitment to report it whatever it showed; a header
error in one loader (patent counts mislabelled on the mirror) was
fixed before any outcome was scored. The
protocol scored seeds 5--14; seeds 15--24 were added to the first two
batches after that reading, under unchanged rules, and change no call,
outcome, or void there; every table reports twenty seeds, the void
metro row ten. The void criterion was written after the second batch
had run, in place of a withdrawn plug-in-fit precondition, a metro
re-run at smaller learning rates having been tried and not adopted; it
voids metro alone.
Three deviations from the written protocol: set 10 (Fertility2, which
has no count target) was replaced before any run, after the marginal
dispersions had been read, by student grades on the same closed-side
criterion; the log1p transform it names
for heavy-tailed features was not applied (credit reports use log
income), and Table~\ref{tab:appendix_screen_endtoend} compares the
end-to-end arms with RMSE of the median rate $e^{m}$, where the
protocol names the frozen arms and the predictive mean.
Four sets carry a data-quality
flag from an independent audit of the setup (post-outcome covariates
on credit reports and Medicare stays, a panel under a random split on
the German panel and on firm patents); without the three in the first
batch its decision accuracy is $7/8$ (two abstentions) and the rank
correlation $0.89$, and over all batches without the four $20/21$.
The void criterion removes one set of the second
batch, Metro traffic (hourly vehicle counts averaging $3\,000$), where
training is numerically unstable at that count scale: the certificate
is non-finite on four of the ten seeds run, $\hat\phi$ ranges from
$3\times10^{2}$ to $2\times10^{6}$ across seeds, and several arms'
means diverge, while the plug-in's own mean fits (validation
$R^2=0.80$). Its screen calls open on nine of ten seeds and its median
binding gap is $0.62$, but two seeds (gaps $23$ and $8$) widen the
interval of the mean to include zero; neither decision nor price can be scored. A ninth candidate, with missing values in the
target, was dropped before running.

\paragraph{Scope.} A Poisson plug-in, the relu
backbone ($H{=}64$), a Poisson-lognormal free belief, and the
frozen-feature binding gap as the outcome; another backbone or
covariance family changes what is expressible, which the Pearson
dispersion does not see. The screen's target is the binding cost relative to the
fitted plug-in head and representation, not the overdispersion of the
conditional data law (see the mean-misfit check below), and its rates
on these sets are preliminary external evidence, not an estimate of
deployment accuracy.

\paragraph{Results.}
Table~\ref{tab:appendix_screen_batches} lists every set;
Figure~\ref{fig:screen_price} plots the price proxy $\hat p$ against
the certificate, here $R(0)-R(\hat v)$ with $\hat v$ the frozen free
refit at the \FSC-Full mean (Tables~\ref{tab:binding_gap}
and~\ref{tab:appendix_certificates} take the end-to-end \FSC-Full
belief's). On the first batch the decision rule was right on
$10$ of the $11$ sets it did not abstain on (two abstentions), the
underdispersed and the equidispersed set called closed (wine quality,
football goals), one false open on the set with the smallest
counts (mean $0.3$, gap $0.04$); the second batch was chosen to test
the price and carries no closed set. On the third batch the screen was
right on eight of eight (always-open: four), the four sets expected
closed called closed on $85$--$100\%$ of seeds and the solar-flare
set, at the same mean count as the false open, called open correctly;
one open set, post comments, is open only barely (gap $0.06$). Over
the three batches the open sets
split \screenOpenCalledOpen{} called open, \screenOpenCalledClosed{}
closed, and \screenOpenAbstained{} abstained, the closed sets
\screenClosedCalledOpen{} called open and \screenClosedCalledClosed{}
closed: \screenCorrect{} of \screenDecided{} right (Wilson $95\%$
interval~\citep{Wilson1927} \screenAccLo{}--\screenAccHi{}, which treats the data sets as
independent trials and is indicative only; $10/11$ on the derivation
batch, $14/14$ on the two later ones), where an always-open rule
scores \screenOpenDecided{}, so the evidence is the closed side and
the closure generator. The protocol's stability hypothesis (the call
shared by $80\%$ of seeds) holds on $24$ of the $27$ valid sets, not on
boating trips, student grades, and post comments. Table~\ref{tab:appendix_screen_sensitivity} varies the
margin and the open-outcome threshold: at $0.10$ nats the screen
scores $21$ of $25$, its open calls then including sets whose gap lies
between $0.05$ and $0.10$, which $\hat p$ sizes as small except on post
comments ($\hat p=5.68$). On the first batch the statistic
$s=\tfrac12\log\hat\phi$ orders the binding gap at Spearman \screenRhoS{}
but its band $[s,10s]$, specified in advance, covered only $4$ of $10$ open
sets, more than the third of misses at which the protocol stops, every miss below $s$: $\tfrac12\log\hat\phi$ is the large-count
limit of the price and over-predicts at small counts. The count-aware
$\hat p$, the upper end of~\eqref{eq:price_bounds} at the plug-in under
the negative-binomial stand-in, orders the
certificate at \screenRhoPOne{} on that batch and covers
\screenBOneBand{} of \screenBOneOpen{}; fixed before the second batch,
it covered \screenBTwoBand{} of \screenBTwoOpen{} there at
\screenRhoPTwo{}, and \screenTestBand{} of \screenTestOpen{} over the
two later batches (Spearman \screenRhoPTest{}; Wilson $95\%$ lower
bound \screenBandLo{}), the certificate a median \screenTestRatio{} of
$\hat p$ (an observed ordering, not one Theorem~\ref{thm:price}(d)
implies: the certificate is read at the \FSC-Full mean). On Seoul bikes the certificate
($43$ nats) and the finite-$N$ binding gap ($0.4$) part: the certificate
bounds the flat-leverage price, and at the learned prior scale the bound
profile is not flat but absorbs the average overdispersion. The closure generator at the scored seeds, the calibration of the
margin rather than an outcome, is called closed on \screenClosureClosed{}
of \screenClosureSeeds{} seeds
($\tfrac12\log\hat\phi=\screenClosureSLo{}$--$\screenClosureSHi{}$, the
fitted head's own excess, the misses all at $N{=}500$).
On the second batch the binding gap is $0.82$--$0.98$ of the certificate
on every set (Table~\ref{tab:appendix_screen_batches}): at these
counts and prior scales the bound belief pays the full expressible
improvement, Theorem~\ref{cor:price_general}'s regime observed. Table~\ref{tab:appendix_screen_endtoend}
gives what each head delivers end to end: the free belief or the direct
negative-binomial head is best in NLL on every open set, every head is
within $0.07$ on the closed ones (the ELBO belief nominally ahead on
horse-kick deaths, the plug-in on baseball runs), the early-stopped
ELBO belief loses $0.05$--$0.9$ nats where the gap is open ($9$ on
Seoul bikes), and in RMSE of the median rate the ELBO belief is
ahead of the free one on fifteen sets, behind on seven, and tied on
five, by hundredths to tenths except on Seoul bikes, post comments, and
firm patents ($2.8$), the plug-in
near the best on most: the NLL-versus-RMSE trade exists but is small
and set-dependent.

\paragraph{Mean misfit and capacity.}
Two checks outside the prospective protocol. On the nineteen valid
designs of the first two batches the counts were re-simulated from a
gradient-boosted Poisson rate fitted to the real data. Under a Poisson
law the screen still calls nine designs open: the network head misfits
the rough rate, and by Corollary~\ref{prop:plugin_test}(a) that misfit
is dispersion about the fitted head. It is mostly not a false alarm: the
measured binding gap opens on eight designs ($0.10$--$0.84$ nats), seven
of the nine called open (the other two at $0.04$, just under the
threshold), and stays closed where the head fits, the call matching the outcome on
$16$ of $19$ and $\hat p$ covering $7$ of $8$ open designs. The screen
and the binding gap respond to the same quantity, the dispersion about
the fitted head, whether it comes from the data law or from the head's
mean error; with the dispersion re-simulated at $\phi=2$ and $5$ the
screen calls all nineteen designs open. The call is also insensitive to the head's capacity:
widths $16$, $64$, and $256$, and $1{,}000$ or $3{,}000$ steps, give
the same calls on the three closed sets of the first batch and on the
closure generator.
\input{tables/appendix_screen_batches}
\input{tables/appendix_screen_sensitivity}
\begin{figure}[h]
\centering
\includegraphics[width=0.5\textwidth]{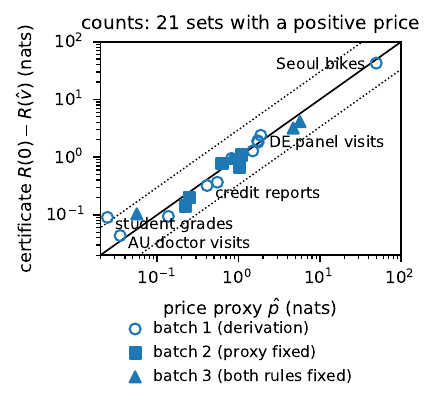}
\caption{The price proxy read at the plug-in against the certificate
measured after the fit: $\hat p$ (the upper end of~\eqref{eq:price_bounds} at the
plug-in under a negative-binomial stand-in) against $R(0)-R(\hat v)$
(its lower end), log axes, on the $21$ screen sets where both are
positive (the void set and six sets with a zero proxy or a negative
certificate omitted); circles the derivation batch, squares the batch run
after the rule was fixed, triangles the third batch; solid line
equality, dotted the band $[\hat p/3,3\hat p]$.}
\label{fig:screen_price}
\end{figure}
\input{tables/appendix_screen_endtoend}

\paragraph{Calibration.}
Table~\ref{tab:appendix_screen_calibration} is the controlled study
behind the margin: on a Poisson GLM truth the one-sided held-out score
test holds its nominal level with an exact plug-in ($2$--$6\%$ at
$5\%$) and, with a plug-in fitted on $5{,}000$ points, has power
$0.85$ at $\tau=0.05$ ($0.88$--$0.90$ at $\tau=0.1$) with $500$
held-out points, but a fitted plug-in fails the null through its own mean
error, since on held-out points
$\mathbb{E}[\hat\kappa_n\mid x]=(\lambda_n-\hat\lambda_n)+(\lambda_n-\hat\lambda_n)^2$
under closure (Corollary~\ref{prop:plugin_test}(a)), and with
$5\,000$ held-out points the closure shape rejects on $59\%$ of
replicates at a relative mean-squared error of $0.011$. Significance alone
therefore cannot decide for a fitted head at large $N$; the margin on
$\hat\phi$ decides, and $\hat p$, not the relative excess
$\hat\tau=\sum\hat\kappa/\sum\hat\lambda^2$ (small on large-count sets
whose price is large), sizes it.
\input{tables/appendix_screen_calibration}

%% file: tables/appendix_screen_batches.tex
\begin{table}[htp]
\centering\scriptsize
\setlength{\tabcolsep}{3pt}
\caption{The channel screen on twenty-eight count data sets, all but wine quality (red and white wines of the UCI Wine set and its OOD pair) used nowhere else in the paper (protocol committed before any model was run on them, its later changes stated in the text; relu backbone; medians over 20 seeds, ten for the void metro row). Screen $\tfrac12\log\hat\phi$ and statistic $T$ from the plug-in's cross-fitted training residuals; call = fraction of seeds calling the channel open ($T>1.645$ and $\tfrac12\log\hat\phi\ge0.15$); price proxy $\hat p$ = the upper end of~\eqref{eq:price_bounds} at the plug-in under a negative-binomial stand-in matched to $\hat\phi$ (Corollary~\ref{prop:plugin_test}(b)); $P_1$ = the certificate $R(0)-R(\hat v)$ at the free belief's mean; gap = frozen-bound minus frozen-free (the quantity of Table~\ref{tab:binding_gap}); outcome open when the gap exceeds $0.05$ with its paired interval above zero; band = $[\hat p/3,\,3\hat p]$ around $P_1$. $^\dagger$ data-quality flag (post-outcome covariates or a panel split), reported as a sensitivity in the text.}
\label{tab:appendix_screen_batches}
\begin{tabular}{@{}lrrrrrrrl@{}}
\toprule
set & mean $y$ & $\tfrac12\log\hat\phi$ & $T$ & call & $\hat p$ & $P_1$ & gap & outcome: verdict \\
\midrule
\multicolumn{9}{l}{\emph{Batch 1 (derivation set for the price rule)}} \\
AU doctor visits & 0.3 & 0.23 & 6.0 & 1.0 & 0.03 & 0.04 & 0.04 & closed: false open \\
credit reports$^\dagger$ & 0.4 & 0.99 & 4.0 & 1.0 & 0.41 & 0.32 & 0.24 & open: correct, in band \\
football goals & 1.4 & 0.00 & $-$0.5 & 0.0 & 0.00 & $-$0.00 & $-$0.00 & closed: correct \\
articles & 1.9 & 0.35 & 3.8 & 1.0 & 0.14 & 0.09 & 0.08 & open: correct, in band \\
boating trips & 2.3 & 1.42 & 1.6 & 0.5 & 1.49 & 1.28 & 0.90 & open: abstains \\
DE visits & 2.6 & 0.79 & 5.2 & 1.0 & 0.82 & 0.95 & 0.88 & open: correct, in band \\
DE panel visits$^\dagger$ & 3.1 & 1.12 & 14.8 & 0.9 & 1.88 & 2.39 & 2.26 & open: correct, in band \\
wine quality & 5.8 & $-$1.24 & $-$66.3 & 0.0 & 0.00 & $-$0.00 & 0.00 & closed: correct \\
NMES visits & 6.2 & 1.01 & 9.6 & 1.0 & 1.74 & 1.82 & 1.74 & open: correct, in band \\
AZ stays & 9.1 & 0.62 & 7.2 & 1.0 & 0.55 & 0.37 & 0.35 & open: correct, in band \\
Medicare stays$^\dagger$ & 10 & 0.96 & 6.0 & 1.0 & 1.72 & 1.92 & 1.78 & open: correct, in band \\
student grades & 12 & 0.15 & 2.6 & 0.5 & 0.02 & 0.09 & 0.07 & open: abstains \\
Seoul bikes & 708 & 2.35 & 26.4 & 1.0 & 49.30 & 42.59 & 0.40 & open: correct, in band \\
\addlinespace[2pt]
\multicolumn{9}{l}{\emph{Batch 2 (price rule fixed before the run)}} \\
doctor visits & 1.5 & 0.94 & 1.9 & 0.9 & 1.02 & 0.89 & 0.73 & open: correct, in band \\
Medicaid visits & 1.8 & 0.92 & 2.5 & 1.0 & 1.03 & 0.65 & 0.59 & open: correct, in band \\
DE visits (bad health) & 2.3 & 0.71 & 4.5 & 1.0 & 0.63 & 0.78 & 0.75 & open: correct, in band \\
RAND HIE visits & 2.8 & 0.90 & 15.2 & 1.0 & 1.10 & 1.09 & 1.05 & open: correct, in band \\
AZ DRG stays & 4.9 & 0.44 & 4.1 & 1.0 & 0.25 & 0.20 & 0.20 & open: correct, in band \\
AZ CABG stays & 7.3 & 0.42 & 5.1 & 1.0 & 0.22 & 0.14 & 0.13 & open: correct, in band \\
metro traffic & 3235 & 4.34 & 55.7 & 0.9 & -- & -- & 0.62 & void (non-finite outcomes) \\
\addlinespace[2pt]
\multicolumn{9}{l}{\emph{Batch 3 (screen and price rule both fixed before the run; enriched with sets expected to be closed)}} \\
solar flares & 0.3 & 0.30 & 3.5 & 0.9 & 0.06 & 0.10 & 0.08 & open: correct, in band \\
horse-kick deaths & 0.7 & 0.09 & 1.0 & 0.1 & 0.01 & $-$0.00 & $-$0.00 & closed: correct \\
La Liga goals & 1.3 & $-$0.02 & $-$1.9 & 0.0 & 0.00 & $-$0.00 & $-$0.00 & closed: correct \\
Bundesliga goals & 1.5 & 0.00 & 0.0 & 0.0 & 0.00 & $-$0.00 & $-$0.00 & closed: correct \\
post comments & 6.4 & 2.07 & 2.0 & 0.6 & 5.68 & 4.12 & 0.06 & open: correct, in band \\
hours absent & 8.1 & 1.50 & 3.6 & 0.9 & 4.69 & 3.19 & 1.21 & open: correct, in band \\
firm patents$^\dagger$ & 37 & 0.76 & 4.1 & 1.0 & 0.83 & 0.95 & 0.67 & open: correct, in band \\
baseball runs & 717 & $-$0.08 & $-$3.6 & 0.0 & 0.00 & $-$0.00 & 0.00 & closed: correct \\
\bottomrule
\end{tabular}
\end{table}

%% file: tables/appendix_screen_sensitivity.tex
\begin{table}[h]
\centering\scriptsize
\setlength{\tabcolsep}{5pt}
\caption{Sensitivity of the channel screen to the dispersion margin (the rule's $1.35$ in bold): decisions on the valid sets scored as in Table~\ref{tab:appendix_screen_batches}, and the rate at which the closure generator is called open over its seeds. Smaller margins look no worse on the real sets, most of which are open, and fail on closure; larger ones begin to miss open sets. Bottom: the same decisions when the binding gap counts as open above $0.02$, $0.05$ or $0.10$ nats.}
\label{tab:appendix_screen_sensitivity}
\begin{tabular}{@{}lrrrrr@{}}
\toprule
margin on $\hat\phi$ & right / decided & abstentions & false open & false closed & closure called open \\
\midrule
1.10 & 25/26 & 1 & 1 & 0 & 0.98 \\
1.20 & 25/26 & 1 & 1 & 0 & 0.38 \\
1.30 & 25/26 & 1 & 1 & 0 & 0.17 \\
\textbf{1.35} & 24/25 & 2 & 1 & 0 & 0.08 \\
1.40 & 24/26 & 1 & 1 & 1 & 0.03 \\
1.50 & 24/26 & 1 & 1 & 1 & 0.02 \\
1.60 & 24/25 & 2 & 0 & 1 & 0.00 \\
\midrule
\multicolumn{6}{@{}l}{\emph{open-outcome threshold on the binding gap, at margin $1.35$ (the protocol's $0.05$ in bold)}} \\
0.02 nats & 25/25 & 2 & 0 & 0 & -- \\
\textbf{0.05} nats & 24/25 & 2 & 1 & 0 & -- \\
0.10 nats & 21/25 & 2 & 4 & 0 & -- \\
\bottomrule
\end{tabular}
\end{table}

%% file: tables/appendix_screen_endtoend.tex
\begin{table}[htp]
\centering\scriptsize
\setlength{\tabcolsep}{3.5pt}
\caption{What each head delivers end to end on the screen's count sets (test split, medians over seeds): NLL of the Poisson plug-in, the score-trained free belief, the early-stopped ELBO belief and the negative-binomial head at its validation-selected $\lambda$ (bold = best), and RMSE / MAE of the median rate $e^{m}$ for plug-in / free / ELBO. The free belief or the NB head is best in NLL on every open set and every head is within $0.07$ on the closed ones; the ELBO belief ties where the screen is closed and loses where it is open and on the two abstentions; the point-prediction differences are small and mixed.}
\label{tab:appendix_screen_endtoend}
\begin{tabular}{@{}lrrrrll@{}}
\toprule
set & plug-in & free & ELBO & NB & RMSE p/f/E & MAE p/f/E \\
\midrule
\multicolumn{7}{l}{\emph{Batch 1}} \\
AU doctor visits & 0.65 & \textbf{0.62} & 0.64 & 0.62 & 0.75/0.74/0.74 & 0.39/0.36/0.39 \\
credit reports & 0.92 & 0.78 & 0.89 & \textbf{0.74} & 1.28/1.39/1.31 & 0.60/0.55/0.61 \\
football goals & \textbf{1.45} & 1.46 & 1.46 & 1.46 & 1.17/1.17/1.17 & 0.91/0.91/0.91 \\
articles & 1.79 & 1.73 & 1.78 & \textbf{1.72} & 1.87/1.89/1.89 & 1.30/1.28/1.29 \\
boating trips & 2.76 & \textbf{1.22} & 2.00 & 1.24 & 7.91/6.98/7.68 & 2.26/2.08/2.17 \\
DE visits & 2.75 & 2.08 & 2.60 & \textbf{2.07} & 3.83/4.02/3.83 & 2.22/2.10/2.24 \\
DE panel visits & 3.71 & 2.16 & 2.82 & \textbf{2.16} & 5.56/6.24/5.89 & 3.15/2.86/3.00 \\
wine quality & \textbf{1.85} & 1.85 & 1.85 & 1.85 & 0.69/0.69/0.70 & 0.54/0.54/0.54 \\
NMES visits & 4.13 & \textbf{2.77} & 3.49 & 2.78 & 6.29/6.61/6.52 & 4.19/3.95/4.24 \\
AZ stays & 3.08 & \textbf{2.76} & 3.03 & 2.79 & 5.38/5.38/5.49 & 3.46/3.30/3.53 \\
Medicare stays & 4.81 & \textbf{3.18} & 4.11 & 3.22 & 8.84/9.33/9.09 & 5.75/5.67/5.78 \\
student grades & 2.88 & \textbf{2.79} & 2.84 & 3.37 & 3.62/3.61/3.51 & 2.66/2.70/2.60 \\
Seoul bikes & 42.75 & \textbf{6.27} & 15.21 & 6.51 & 256.07/263.57/533.92 & 163.63/159.84/399.47 \\
\multicolumn{7}{l}{\emph{Batch 2}} \\
doctor visits & 2.20 & 1.73 & 2.21 & \textbf{1.72} & 2.56/2.70/2.71 & 1.73/1.55/1.78 \\
Medicaid visits & 2.36 & 1.85 & 2.21 & \textbf{1.82} & 2.75/2.68/2.62 & 1.88/1.70/1.84 \\
DE visits (bad health) & 2.59 & 1.99 & 2.56 & \textbf{1.99} & 3.30/3.39/3.26 & 2.11/1.91/2.08 \\
RAND HIE visits & 2.86 & \textbf{2.07} & 2.73 & 2.09 & 4.11/4.26/4.26 & 2.47/2.23/2.55 \\
AZ DRG stays & 2.57 & \textbf{2.36} & 2.54 & 2.38 & 3.38/3.41/3.39 & 2.36/2.27/2.35 \\
AZ CABG stays & 2.69 & \textbf{2.51} & 2.67 & 2.53 & 4.17/4.18/4.14 & 2.61/2.56/2.64 \\
\multicolumn{7}{l}{\emph{Batch 3}} \\
solar flares & 0.70 & 0.63 & 0.68 & \textbf{0.62} & 0.81/0.80/0.79 & 0.43/0.38/0.41 \\
horse-kick deaths & 1.12 & 1.14 & \textbf{1.10} & 1.12 & 0.86/0.88/0.87 & 0.71/0.71/0.69 \\
La Liga goals & 1.40 & 1.40 & \textbf{1.40} & 1.40 & 1.11/1.10/1.11 & 0.86/0.86/0.86 \\
Bundesliga goals & \textbf{1.50} & 1.50 & 1.50 & 1.51 & 1.22/1.22/1.22 & 0.96/0.96/0.96 \\
post comments & 4.68 & \textbf{1.55} & 2.21 & 1.61 & 25.98/60.15/33.89 & 5.09/6.09/6.58 \\
hours absent & 5.66 & \textbf{2.67} & 3.07 & 2.71 & 13.60/12.19/12.06 & 5.32/4.40/5.15 \\
firm patents & 3.69 & \textbf{2.86} & 3.44 & 2.92 & 19.31/22.43/19.67 & 8.02/8.63/8.11 \\
baseball runs & \textbf{4.60} & 4.62 & 4.62 & 4.67 & 23.82/24.02/24.19 & 19.11/19.18/19.11 \\
\bottomrule
\end{tabular}
\end{table}

%% file: tables/appendix_screen_calibration.tex
\begin{table}[h]
\centering\scriptsize
\setlength{\tabcolsep}{4pt}
\caption{Calibration of the held-out score test on a Poisson GLM truth with a lognormal random effect of variance $v^\ast(x)$ (none = closure; const $\tau$; quad $\tau x_1^2$, expressible; latent $\tau z^2$, not expressible), 100 replicates: rejection rate of the one-sided test at the nominal 5\% (one-sided $N(0,1)$), the relative excess-variance estimate $\hat\tau$, the gain in nats a belief with the true profile could buy over the best Poisson plug-in, and the plug-in's relative mean-squared error; the overdispersed shapes (below the rule) use the fitted plug-in at $N_\text{train}=5000$. With an exact plug-in the test holds its level; a fitted plug-in fails the null through its own mean error, which on held-out points is excess variance.}
\label{tab:appendix_screen_calibration}
\begin{tabular}{@{}lrrrrr@{}}
\toprule
shape / plug-in & $N_\text{held-out}$ & rejection & $\hat\tau$ & gain (nats) & plug-in rel.\ MSE \\
\midrule
closure, oracle plug-in & 500 & 0.04 & 0.002 & 0.000 & 0.000 \\
closure, oracle plug-in & 5000 & 0.03 & $-$0.000 & 0.000 & 0.000 \\
closure, fitted plug-in, $N_\text{train}$ 500 & 500 & 0.17 & 0.010 & 0.000 & 0.011 \\
closure, fitted plug-in, $N_\text{train}$ 500 & 5000 & 0.59 & 0.011 & 0.000 & 0.011 \\
closure, fitted plug-in, $N_\text{train}$ 5000 & 500 & 0.06 & 0.003 & 0.000 & 0.001 \\
closure, fitted plug-in, $N_\text{train}$ 5000 & 5000 & 0.04 & 0.001 & 0.000 & 0.001 \\
\midrule
const $\tau=0.05$, fitted plug-in & 500 & 0.85 & 0.052 & 0.012 & 0.001 \\
const $\tau=0.2$, fitted plug-in & 500 & 1.00 & 0.192 & 0.135 & 0.002 \\
const $\tau=0.5$, fitted plug-in & 500 & 0.98 & 0.555 & 0.643 & 0.005 \\
quad $\tau=0.1$, fitted plug-in & 500 & 0.88 & 0.202 & 0.124 & 0.003 \\
latent $\tau=0.1$, fitted plug-in & 500 & 0.90 & 0.115 & 0.090 & 0.002 \\
\bottomrule
\end{tabular}
\end{table}

%% file: sections/app_experiments.tex
\section{The Gaussian Case in Full}
\label{app:gaussian_full}

The Gaussian likelihood is the small-price instance of
Section~\ref{sec:experiments}; the complete benchmark tables and the
controls behind its Gaussian rows follow.

\subsection{Benchmarks}
\label{app:gaussian_benchmarks}
\input{tables/headline_nll}
\input{tables/table_scale_nll}

The \emph{Exact corner} row of Table~\ref{tab:headline_nll} is the
evidence-trained neural-linear corner, exact inference in its own
model at its own learned features; VBLL trains the same last-layer
family on the ELBO, whose optimum restores the binding
(Corollary~\ref{cor:posterior_seeking}), and tracks the corner rather
than the free belief on six of the eight sets (not on energy and
naval). The free belief improves on the corner on every
UCI set except near-noiseless naval, where the corner wins outright
(Corollary~\ref{cor:pathologies}); against the free-variance head MVN,
which shares the score, it ties on four UCI sets (interval containing
zero), leads on three, and trails on naval. Table~\ref{tab:table_scale_nll} repeats the protocol at
scale with mini-batches; its frozen rows feed
Table~\ref{tab:binding_gap}.

\subsection{The Covariance-Binding Gap at Fixed Mean and Features}
\label{app:routing_selection}

The references differ from the \FSC belief in the binding, in the
selection criterion, and, end to end, in the learned features; the
control freezes features and mean so that only the binding differs.

\paragraph{The same-features test.}
\FSC-Full's features and mean are frozen and $(\Sigma,\sigma_\text{obs}^2,\alpha)$
refit under the identical objective, $\Sigma$ free (warm-started) or
bound to~\eqref{eq:R1} (Table~\ref{tab:frozen_routing}); the paired
test-NLL difference is the fixed-mean binding gap of
Corollary~\ref{thm:pred_evidence}; the control runs on all four backbones and inherits the base
run's validation selection per seed. Table~\ref{tab:frozen_routing}:
the free arm reproduces the base run (median difference below $0.03$
nats, except on yacht and naval, where the warm-started refit improves
on it by $0.1$ and $0.9$) and wins $139/160$ pairs with a positive
median gap on every data set. At the mini-batch scale protocol (the
frozen rows of Table~\ref{tab:table_scale_nll}) the free arm wins
$95/100$ pairs with median gaps of $0.01$--$0.23$ nats.
\input{tables/appendix_frozen_routing}

\paragraph{The variance fit on held-out data.}
The variance side is refit on a calibration half of the validation split
($31$ to $1{,}193$ points the mean and features never saw), early
stopping on the other half and scoring on the test split. The gap
survives: the free arm wins $136/160$ pairs, with a positive median on $7/8$ sets, against
the train-bound arm ($130/160$ against the calibration-bound
one), with medians $0.02$--$0.95$ nats; wine, whose in-sample gap was
$0.01$, moves to zero. The median test NLLs of the in-sample and
held-out free fits differ by at most $0.07$ nats on every data set
(paired, per-seed differences have medians up to $0.12$, on yacht), so
the binding gap is a population effect, not optimism.

\paragraph{The full conjugate map at the same features.}
The two refits above keep the free belief's mean and bind only the
covariance. Table~\ref{tab:frozen_conjugate} adds the complete map of
Definition~\ref{def:routing} at the same features: the conjugate
posterior mean and covariance, with $(\sigma_\text{obs}^2,\alpha)$
selected by the shared score or by evidence maximisation, the latter
being the exact neural-linear corner at the free model's features.
Its gap to the free refit is within $0.13$ nats of the covariance-only
gap on every data set except naval under either selection, so
binding the mean as well leaves the covariance-binding gap in place;
naval, the near-noiseless case of
Corollary~\ref{cor:pathologies}(i), is the exception where the
conjugate mean and covariance together beat the free refit by $1.2$
nats while the covariance-only binding still costs $0.36$.
\input{tables/appendix_frozen_conjugate}

\paragraph{The gap is predicted by Corollary~\ref{thm:pred_evidence}.}
Evaluating the binding gap of Definition~\ref{def:gaps}, the
population risk difference, on the two fitted profiles predicts the
measured gap with no access to the held-out
likelihoods (Figure~\ref{fig:routing_selection_measured}a): across
the eight UCI data sets the data-set medians track the measurement
(correlation $0.91$, through-origin slope $1.3$; per-seed log--log
correlation $0.87$ over all thirteen data sets, on the $234$ of $260$
seed pairs with both gaps positive), near the $y{=}x$ line the
synthetic model follows at its known $J$.

\paragraph{The noise parameter absorbs the gap: a pinned-noise control.}
Theorem~\ref{cor:price_general} attributes the small Gaussian price
to $\sigma_\text{obs}$, which absorbs the average dispersion, as the
negative-binomial dispersion does for counts
(Appendix~\ref{app:beyond_gaussian}). The converse control removes the
parameter: at the same frozen features and mean the bound arm is
refit with $\sigma_\text{obs}$ held at the free arm's fitted value, so
the bound belief can no longer raise its noise to the average (relu
backbone, $20$ seeds). The gap grows on every data set and on at least
$95\%$ of seeds: from $0.21$ to $8.5$ nats on the heteroscedastic
generator (two-profile prediction $7.5$), from $0.01$--$0.34$ to
$0.05$--$46$ on wine, power, kin8nm, and boston, and by more than three
orders of magnitude on yacht, energy, concrete, and naval, where the
free belief carries almost all of the variance in
$\psi^\top\Sigma\psi$ and its fitted noise sits near the floor
(Corollary~\ref{cor:pathologies}), so the size there reflects how low
the pin is and not a price a practitioner would meet. The two-profile
prediction of Corollary~\ref{thm:pred_evidence} orders the pinned gaps
at Spearman $1.00$ across the ten sets and is within a factor of $1.6$
of each. With the pin below the true noise the closure generator opens
as well ($0.03$ to $0.20$ nats), the Gaussian counterpart of a Poisson
likelihood whose variance sits below the data's.

\paragraph{Selection is second-order when data-rich.}
Proposition~\ref{prop:routing_selection} separates the bound belief's
two losses: the selection gap $\Delta_\text{sel}$~\eqref{eq:sel_gap},
which vanishes as $N/H$ grows, part~(i), and the binding gap
$\Delta_\text{bind}$ of Definition~\ref{def:gaps}, which does not,
part~(ii). On the frozen features the bound belief is conjugate, so
both criteria of~\eqref{eq:sel_gap} are available in closed form;
sweeping the width $H\in\{12,25,50,100,200\}$ spans $N/H$ from
${\sim}1$ to ${\sim}600$.
Figure~\ref{fig:routing_selection_measured}b: the RMS selection gap
falls from $2.3$ nats at $N/H\le3$ to $8\times10^{-4}$ at $N/H>100$
at the conjugate model's~\eqref{eq:rsel_model} rate, while the
binding gap stays first-order at every width on every set but
near-noiseless naval, as~(ii) states. Below
$N/H\approx10$ the evidence criterion selects worse predictives than
leave-one-out (median $+0.05$ nats at $N/H\le3$): the
$\Delta_\text{sel}\to0$ of~(i) is asymptotic. Each width trains its
own relu+LN backbone (ten seeds).

\begin{figure}[t]
  \centering
  \includegraphics[width=0.98\textwidth]{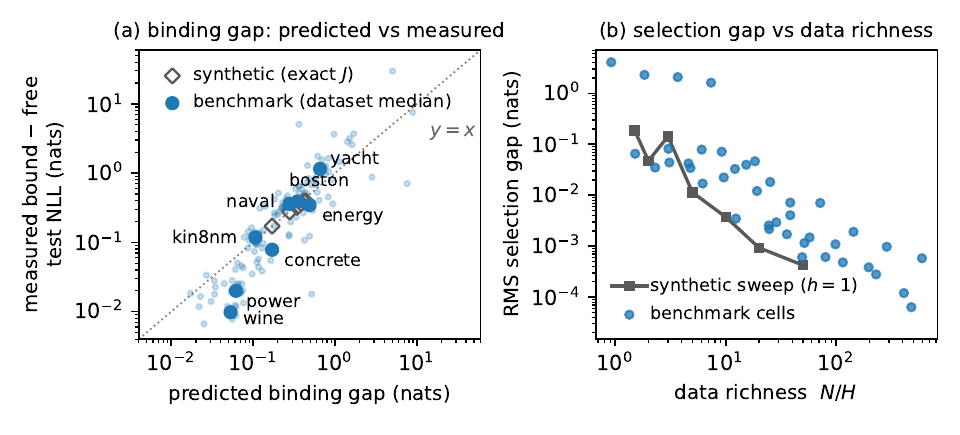}
  \caption{The binding/selection separation of
    Proposition~\ref{prop:routing_selection} on the benchmarks, the
    conjugate model~\eqref{eq:rsel_model} alongside.
    \textbf{(a)} The binding gap predicted by
    Corollary~\ref{thm:pred_evidence} from the fitted variance profiles
    versus the measured bound$-$free test NLL at frozen features and
    mean, eight UCI sets, 20 seeds (log--log; small points: seeds;
    large: data-set medians); diamonds: the conjugate model, its
    bound-minus-true-profile test NLL against its Jensen gap $J$
    (Proposition~\ref{prop:routing_selection}(ii)).
    \textbf{(b)} RMS selection gap $\Delta_\text{sel}$~\eqref{eq:sel_gap}
    versus $N/H$ (relu+LN, 10 seeds): benchmark cells (data set
    $\times$ width) decay at the conjugate model's rate (noise slope
    $h{=}1$), Proposition~\ref{prop:routing_selection}(i), above its level.}
  \label{fig:routing_selection_measured}
\end{figure}

\subsection{The Prior Term}
\label{app:prior_ablation}

The $F_\text{SC}$ belief, at each variance
structure (None, Diag, Full), is trained with $-\log Z_w$ and $\alpha$
learned in the pass, with $-\log Z_w$ at a cross-validated fixed
$\alpha$, or with the term deleted and a cross-validated weight decay
in its place; a head-only arm inherits the belief's structure, so it
ablates the prior term alone, and each tuned arm is
validation-selected over architecture${\times}$knob. The no-prior
Full arm is the exact quadratic-variance counterpart: the identical
$\psi^\top\Sigma\,\psi$ head, optimiser, and protocol as \FSC-Full
(None fixed-$\alpha$ is the no-prior arm at $\lambda{=}\alpha$).
Table~\ref{tab:appendix_prior_ablation} gives the UCI rows; the scale
rows are quoted in the text.

At scale the prior-term axis is flat at every structure: the tuned
arms differ from the learned-$\alpha$ variant by $-0.003$ to $+0.011$
nats, and over three decades of $\lambda$ and five of $\alpha$ the
selected NLL moves by at most $0.009$; at $H/N\le0.02$ the prior term
is the $O(H/N)$ term of Corollary~\ref{cor:pathologies}(i). On UCI
the prior term has content and the learned $\alpha$ lands inside the
tuned scatter:
Full with learned $\alpha$ ties both siblings in aggregate and, at
paired $p<0.05$, is ahead of its fixed-$\alpha$ sibling on
boston, concrete, and power ($0.02$--$0.11$) and of its no-prior
sibling on boston ($0.26$) and power, behind both on naval
($0.12$--$0.14$) and behind fixed $\alpha$ on yacht ($0.35$). Swapping only the prior
factor for the $\mathrm{KL}(q\,\|\,p)$ of the PPGPR
objective~\citep{Jankowiak2020} (the
direct-loss criterion of \citet{Sheth2020} at the last layer) leaves
NLL tied in aggregate and raises calibration error on all eight data
sets, significantly on one (pooled mean $+0.017$).
\input{tables/appendix_prior_ablation}

\subsection{What the Freed Covariance No Longer Means}
\label{app:ood}

The bound belief pins its variance to the leverage
$\psi_n^\top(\Psi^\top\Psi/\sigma_\text{obs}^2+\alpha I)^{-1}\psi_n$, a
function of the inputs alone that grows away from the training
features; the free belief reallocates it onto the residuals
(Corollary~\ref{thm:pred_evidence}), so its $\psi^\top\Sigma\psi$ is no
longer an epistemic score by construction. Three regression tasks face
a foreign data set (text and vision through the same frozen encoder,
wine red against white; Appendix~\ref{app:setup}); each method
scores a point by its epistemic variance and Table~\ref{tab:ood}
reports the AUROC against the in-distribution test set, with Deep
Ensembles~\citep[$M{=}5$;][]{Lakshminarayanan2017} and MC
Dropout~\citep[$p{=}0.1$, $50$ passes;][]{Gal2016} as sampling
references. The bound belief, trained end to end under the binding,
separates shifted inputs at $0.88$, $0.91$, and $0.95$; the free
belief, $0.06$--$1.28$ nats better in distribution, at $0.58$--$0.74$; MAP and \FSC-None hold no covariance
and sit at $0.5$, and the free-variance head MVN is anti-correlated
with shift on wine ($0.365$). Freeing the covariance leaves the signal in the
features, and the signal is not specific to the belief: a conjugate refit on the frozen
features of a trained model, Bayesian linear regression with its
hyperparameters by evidence, scores $0.75$--$0.96$ at the free
belief's, the matched head's, and the MAP network's features (ten
seeds); at the free belief's features it reaches $0.78$, $0.88$, and
$0.94$: freeing the covariance changes the meaning of $\Sigma$, and a
conjugate refit restores most, not all, of the bound belief's separation.
\input{tables/table_ood}

%% file: tables/headline_nll.tex
\begin{table}[ht]
\centering
\scriptsize
\caption{Regression test NLL on eight UCI datasets, validation-selected (the \FSC belief: architecture; references: architecture$\times\lambda$), 20 seeds. Ranked blocks: the \FSC belief, the tuned single-pass references, and the exact evidence corner (sequential cavity, bound belief, trained under its own objective at its own features). Unranked: the same-features control (\FSC-Full, features and mean frozen, $\Sigma$ free vs.\ bound to~\eqref{eq:R1}; Table~\ref{tab:frozen_routing}). Bold = best ranked entry; italic = not separated from the best (paired $t$-test, $p\ge0.05$, uncorrected).}
\label{tab:headline_nll}
\renewcommand{\arraystretch}{0.92}%
\begin{tabular}{lcccccccc}
\toprule
 & yacht & concrete & energy & kin8nm & naval & power & wine & boston \\
\midrule
\FSC-Full & \textbf{2.253} & \textit{3.316} & \textit{0.777} & \textbf{$-$0.704} & $-$2.216 & \textbf{2.840} & \textit{0.959} & 2.679 \\
\FSC-Diag & 3.033 & \textit{3.344} & 0.913 & $-$0.694 & $-$4.075 & \textit{2.842} & \textbf{0.953} & \textbf{2.606} \\
\FSC-None & 3.532 & 3.384 & 0.862 & $-$0.627 & $-$3.169 & \textit{2.852} & 0.978 & 2.903 \\
\midrule
MAP & 3.340 & 3.290 & \textit{0.747} & $-$0.578 & $-$3.001 & \textit{2.852} & 0.984 & 2.833 \\
Laplace-Full & 3.000 & \textbf{3.276} & \textit{0.737} & $-$0.577 & $-$3.001 & \textit{2.852} & 0.979 & 2.760 \\
VBLL & 3.667 & 3.338 & \textbf{0.736} & $-$0.626 & $-$2.770 & \textit{2.851} & 0.998 & 2.901 \\
MVN & \textit{2.259} & 3.453 & 0.865 & $-$0.689 & $-$3.198 & \textit{2.842} & \textit{0.969} & \textit{2.647} \\
\midrule
Exact corner (Full) & 3.680 & 3.336 & 0.822 & $-$0.609 & \textbf{$-$7.324} & \textit{2.845} & 1.008 & 2.728 \\
\midrule
Frozen, $\Sigma$ free & 2.465 & 3.286 & 0.799 & $-$0.708 & $-$3.181 & 2.841 & 0.967 & 2.676 \\
Frozen, $\Sigma$ bound & 5.346 & 3.579 & 1.213 & $-$0.580 & $-$2.783 & 2.864 & 0.975 & 3.540 \\
\bottomrule
\end{tabular}
\end{table}

%% file: tables/table_scale_nll.tex
\begin{table}[ht]
\centering
\scriptsize
\caption{Large-scale tabular (California, Protein, Year) and frozen deep-embedding (SICK text, UTKFace vision) test NLL, protocol and markup as Table~\ref{tab:headline_nll}. \emph{Train}/\emph{Infer}: training runs and inference passes per architecture.}
\label{tab:table_scale_nll}
\renewcommand{\arraystretch}{0.92}%
\begin{tabular}{l cc ccccc}
\toprule
 & \multicolumn{2}{c}{Passes} & \multicolumn{5}{c}{Test NLL} \\
\cmidrule(lr){2-3}\cmidrule(lr){4-8}
Method & Train & Infer & California & Protein & Year & SICK & UTKFace \\
\midrule
\FSC-Full & 1 & 1 & 0.740 & \textbf{2.814} & \textbf{3.361} & \textbf{0.872} & \textbf{3.791} \\
\FSC-Diag & 1 & 1 & \textbf{0.688} & 2.870 & \textit{3.361} & 0.882 & 3.881 \\
\FSC-None & 1 & 1 & 0.840 & 2.918 & 3.590 & 0.881 & 3.970 \\
\midrule
MAP & 4 & 1 & 0.878 & 2.935 & 3.605 & 0.888 & 3.865 \\
Laplace-Full & 4 & 1 & 0.876 & 2.934 & 3.605 & 0.888 & 3.864 \\
VBLL & 4 & 1 & 0.879 & 2.964 & 3.630 & \textit{0.878} & 3.915 \\
MVN & 4 & 1 & 0.764 & 2.906 & 3.443 & 0.896 & 3.825 \\
\midrule
Frozen, $\Sigma$ free & 1 & 1 & 0.726 & 2.793 & 3.360 & 0.885 & 3.841 \\
Frozen, $\Sigma$ bound & 1 & 1 & 0.903 & 2.949 & 3.595 & 0.900 & 3.898 \\
\bottomrule
\end{tabular}
\end{table}

%% file: tables/appendix_frozen_routing.tex
\begin{table}[h]
\centering\scriptsize
\caption{The same-features test, belief free vs.\ bound (20
seeds, backbone validation-selected per seed): \FSC-Full is trained once; with its features
and mean frozen, the variance side $(\Sigma,\sigma_\text{obs}^2,\alpha)$
is refit under the identical objective with $\Sigma$ free over the PSD
cone vs.\ bound to the conjugate image~\eqref{eq:R1} --- Corollary~%
\ref{thm:pred_evidence}'s fixed-mean setting with nothing else moving
(point predictions are identical by construction). Median test NLL per
arm; the measured gap (bound $-$ free) is paired per seed and
heavy-tailed to the right (a single far-out test point under the tighter
bound band dominates the mean), so medians are reported; predicted =
Corollary~\ref{thm:pred_evidence}'s risk difference evaluated on the two
fitted variance profiles.}
\label{tab:frozen_routing}
\begin{tabular}{lrrrrrr}
\toprule
 & full (base) & frozen-free & frozen-bound & gap & predicted &
free wins \\
\midrule
yacht & 2.481 & 2.429 & 3.526 & +1.151 & +0.659 & 19/20 \\
boston & 2.622 & 2.665 & 3.049 & +0.390 & +0.349 & 16/20 \\
naval & $-$2.360 & $-$3.190 & $-$2.851 & +0.360 & +0.277 & 20/20 \\
energy & 0.717 & 0.750 & 1.112 & +0.346 & +0.489 & 18/20 \\
kin8nm & $-$0.711 & $-$0.714 & $-$0.599 & +0.118 & +0.107 & 20/20 \\
concrete & 3.256 & 3.266 & 3.367 & +0.078 & +0.171 & 15/20 \\
power & 2.837 & 2.837 & 2.859 & +0.020 & +0.062 & 19/20 \\
wine & 0.954 & 0.960 & 0.965 & +0.010 & +0.053 & 12/20 \\
\bottomrule
\end{tabular}
\end{table}

%% file: tables/appendix_frozen_conjugate.tex
\begin{table}[h]
\centering\scriptsize
\setlength{\tabcolsep}{4pt}
\caption{The full conjugate map at the free model's features: median over seeds of the paired test-NLL difference to the free refit (positive = the bound arm is worse) for the covariance-only bound arm at the free mean (the fixed-mean covariance-binding gap of Table~\ref{tab:binding_gap}) and for the conjugate posterior mean and covariance at the same features, with $(\sigma_\text{obs}^2,\alpha)$ selected by the shared score or by evidence maximisation. Binding the mean as well does not recover the gap; naval is the near-noiseless exception where the conjugate form beats the free refit. The last column gives the mean paired difference of the covariance-only arm with a bootstrap interval over seeds, the estimand of Corollary~\ref{thm:pred_evidence}.}
\label{tab:frozen_conjugate}
\begin{tabular}{@{}lrrrrr@{}}
\toprule
 & $\Sigma$ bound, free mean & full map, score & full map, evidence & $\Sigma$ bound, mean [95\% CI] & seeds \\
\midrule
yacht & $+1.151$ & $+1.201$ & $+1.094$ & $+2.88$ [$+1.12$, $+5.94$] & 20 \\
boston & $+0.390$ & $+0.283$ & $+0.464$ & $+0.86$ [$+0.26$, $+1.71$] & 20 \\
energy & $+0.346$ & $+0.306$ & $+0.218$ & $+0.41$ [$+0.21$, $+0.69$] & 20 \\
concrete & $+0.078$ & $+0.118$ & $+0.118$ & $+0.29$ [$+0.09$, $+0.64$] & 20 \\
wine & $+0.010$ & $+0.009$ & $+0.008$ & $+0.01$ [$-0.01$, $+0.02$] & 20 \\
power & $+0.020$ & $+0.018$ & $+0.019$ & $+0.02$ [$+0.02$, $+0.03$] & 20 \\
kin8nm & $+0.118$ & $+0.110$ & $+0.110$ & $+0.13$ [$+0.11$, $+0.15$] & 20 \\
naval & $+0.360$ & $-1.178$ & $-1.169$ & $+0.40$ [$+0.34$, $+0.47$] & 20 \\
california & $+0.178$ & $+0.152$ & $+0.152$ & $+0.18$ [$+0.17$, $+0.18$] & 20 \\
protein & $+0.155$ & $+0.148$ & $+0.148$ & $+0.16$ [$+0.15$, $+0.17$] & 20 \\
year & $+0.234$ & $+0.231$ & $+0.230$ & $+0.23$ [$+0.23$, $+0.24$] & 20 \\
sick & $+0.005$ & $+0.010$ & $+0.019$ & $+0.01$ [$+0.01$, $+0.03$] & 20 \\
utkface & $+0.056$ & $+0.076$ & $+0.073$ & $+0.06$ [$+0.04$, $+0.07$] & 20 \\
\bottomrule
\end{tabular}
\end{table}

%% file: tables/appendix_prior_ablation.tex
\begin{table}[ht]
\centering
\scriptsize
\caption{Prior-term ablation, test NLL (20 seeds; per-seed validation selection over architecture$\times$knob, i.e.\ the references' own tuning protocol). Within each covariance family: the shipped variant ($\alpha$ learned in the pass via $-\log Z_w$, nothing tuned), the $-\log Z_w$ form at a cross-validated fixed $\alpha$, and the prior term removed with the references' cross-validated weight decay $\lambda$ in its place. The KL row keeps the identical score and swaps only the prior factor for the $\mathrm{KL}(q\,\|\,\mathcal N(0,\alpha^{-1}I))$ substitute of the PPGPR objective (fixed-$\alpha$ grid and learned-$\alpha$ readings pooled into the selection). Bottom: the free-variance MVN head, the structure-free endpoint. Bold = best; italic = not separated from the best (two-sided paired $t$-test, $p\ge0.05$).}
\label{tab:appendix_prior_ablation}
\setlength{\tabcolsep}{4pt}%
\begin{tabular}{lcccccccc}
\toprule
 & yacht & concrete & energy & kin8nm & naval & power & wine & boston \\
\midrule
\FSC-None (learned $\alpha$) & 3.532 & 3.384 & 0.862 & $-$0.627 & $-$3.169 & \textit{2.852} & 0.978 & 2.903 \\
None, no prior (tuned $\lambda$) & 3.602 & 3.377 & \textit{0.795} & $-$0.626 & $-$3.332 & \textit{2.849} & 0.977 & 2.876 \\
\midrule
\FSC-Diag (learned $\alpha$) & 3.033 & \textit{3.344} & 0.913 & $-$0.694 & $-$4.075 & \textit{2.842} & \textit{0.953} & \textbf{2.606} \\
Diag, fixed $\alpha$ (tuned) & 2.183 & \textit{3.316} & 0.849 & \textit{$-$0.699} & \textbf{$-$4.183} & \textit{2.843} & \textbf{0.951} & \textit{2.630} \\
Diag, no prior (tuned $\lambda$) & 2.546 & \textbf{3.296} & \textit{0.855} & \textit{$-$0.696} & \textit{$-$4.159} & \textit{2.843} & \textit{0.952} & \textit{2.617} \\
\midrule
\FSC-Full (learned $\alpha$) & 2.253 & \textit{3.316} & \textit{0.777} & \textbf{$-$0.704} & $-$2.216 & \textbf{2.840} & \textit{0.959} & 2.679 \\
Full, fixed $\alpha$ (tuned) & \textbf{1.901} & 3.425 & \textit{0.842} & \textit{$-$0.698} & $-$2.352 & 2.858 & \textit{0.960} & 2.754 \\
Full, no prior (tuned $\lambda$) & \textit{2.322} & \textit{3.445} & \textbf{0.756} & \textit{$-$0.703} & $-$2.336 & 2.853 & \textit{0.959} & 2.937 \\
Full, KL regulariser (PPGPR) & \textit{2.254} & 3.449 & \textit{0.774} & \textit{$-$0.701} & $-$2.393 & 2.848 & \textit{0.957} & 2.743 \\
\midrule
MVN (free var., tuned $\lambda$) & 2.259 & 3.453 & 0.865 & $-$0.689 & $-$3.198 & \textit{2.842} & \textit{0.969} & \textit{2.647} \\
\bottomrule
\end{tabular}
\end{table}

%% file: tables/table_ood.tex
\begin{table}[ht]
\centering
\scriptsize
\caption{Out-of-distribution detection (covariate shift through a frozen encoder): AUROC separating in-distribution test points from a foreign set, each scored by its epistemic variance (predictive variance minus observation noise; $\psi^\top\Sigma\psi$ for the last-layer methods; for MVN the clipped $\max(\sigma^2(x)-\bar\sigma^2,0)$, $\bar\sigma$ its mean predicted standard deviation on validation). ID/OOD pairs are SICK/AG-News (text), UTKFace/CIFAR (vision), and wine red/white (tabular). Rows are grouped by the \emph{binding}: the free belief, the bound belief (trained end to end through the closed-form posterior, full batch), and leverage/sampling baselines. Validation-selected exactly as the headline (the \FSC belief: best architecture; references: best architecture$\times\lambda$), by in-distribution validation loss only. \emph{Train}/\emph{Infer} are runs/passes per architecture; both beliefs are single-pass, with no $\lambda$ grid. The bound belief separates shift better than the free belief (it trades OOD for predictive density, Corollary~\ref{thm:pred_evidence}); MAP and \FSC-None carry no epistemic signal (AUROC $\equiv0.5$ by construction). Bold = best; italic = not separated from the best (two-sided paired $t$-test, $p\ge0.05$).}
\label{tab:ood}
\begin{tabular}{l cc ccc}
\toprule
 & \multicolumn{2}{c}{Passes} & \multicolumn{3}{c}{OOD AUROC} \\
\cmidrule(lr){2-3}\cmidrule(lr){4-6}
Method & Train & Infer & Text & Vision & Wine \\
\midrule
\FSC-Full & 1 & 1 & 0.577 & 0.682 & 0.736 \\
\FSC-Diag & 1 & 1 & \textit{0.833} & 0.649 & 0.774 \\
\FSC-None & 1 & 1 & 0.500 & 0.500 & 0.500 \\
\midrule
Full bound & 1 & 1 & \textbf{0.882} & \textit{0.911} & \textbf{0.947} \\
Diag bound & 1 & 1 & 0.708 & 0.661 & \textit{0.927} \\
\midrule
Laplace-Full & 4 & 1 & \textit{0.784} & \textbf{0.914} & \textit{0.944} \\
VBLL & 4 & 1 & 0.744 & 0.736 & \textit{0.932} \\
\midrule
MAP & 4 & 1 & 0.500 & 0.500 & 0.500 \\
MVN & 4 & 1 & 0.672 & 0.598 & 0.365 \\
\midrule
Deep Ensemble & 20 & 5 & \textit{0.869} & 0.861 & 0.829 \\
MC Dropout & 1 & 50 & \textit{0.849} & 0.775 & 0.640 \\
\bottomrule
\end{tabular}
\\[2pt]{\footnotesize 10--20 seeds.}
\end{table}

%% file: sections/app_beyond_gaussian.tex
\section{The Binding Gap Beyond the Gaussian Likelihood}
\label{app:beyond_gaussian}

This appendix carries the full per-data-set results behind
Tables~\ref{tab:binding_gap} and~\ref{tab:table_heads_counts}
(Tables~\ref{tab:appendix_probe_counts_gen}
and~\ref{tab:appendix_probe_counts_real}), the certificates of
Proposition~\ref{prop:overdispersion} and
Theorem~\ref{cor:price_general}, the controls behind the count reading
(dispersion parameter, ordered splits, prior swap, ELBO dynamics, HMC,
backbone), and the Gaussian cavity cell of
Proposition~\ref{prop:cavity_invariance}.

\paragraph{Protocol.}
Generators, real sets, splits, the training protocol ($H{=}64$, Adam
$0.01$, twenty seeds), the arms, and the quadrature and positivity
conventions of the heads are those of Appendix~\ref{app:setup}
(Tables~\ref{tab:count_datasets} and~\ref{tab:protocols}); for the count runs the only
departure is that the HMC runs retrain their \FSC-Full base full batch
at $N{=}500$.

\input{tables/appendix_probe_counts_gen}
\input{tables/appendix_probe_counts_real}

\paragraph{The regulariser.}
The heads' weight decay is immaterial here: at any single $\lambda$ of
the grid, with no selection over it, the counterpart's test NLL is
within $0.002$ nats of its validation-tuned value on eight of the nine
count sets, and the $F_\text{SC}$ belief with its prior precision
fixed anywhere in $[0.1,10^3]$ is within $0.02$ of the trained one;
the exception is fluview, the smallest set, where both move by up to
$0.16$. The Poisson head's test NLL moves by $0.03$--$0.14$ across
$\lambda$ on the overdispersed sets ($2.8$ on fluview) and by at most
$0.002$ on the closed ones. What separates the heads where data are
scarce is the profile class, not the regulariser: at $N{=}500$ the
$F_\text{SC}$ belief trails its Poisson-lognormal counterpart by
$0.06$ on that generator ($0.02$ and under $0.01$ at $N{=}2$k and
$20$k) and per-input negative binomial, the generator's own family, by
$0.03$ on the negative-binomial one at every $N$.

\paragraph{The certificates.}
Two quantities are computed from the saved test predictions
(Table~\ref{tab:appendix_certificates}). The channel test compares the
free belief's fitted variance $\hat v(x)$ with the empirical score
$\hat\kappa/\lambda^2$ at the Poisson head, the quantity the belief
variance tracks under a log link since
$\mathrm{Var}=\lambda+\lambda^2(e^{v}-1)$: its decile means are
monotone in $\hat v$ with rank correlation $0.96$--$0.99$ on bike on
every backbone and $0.80$--$0.91$ on the overdispersed generators on
the relu backbones, against $0.42$--$0.55$ there on tanh, no higher
than the $0.19$--$0.63$ of the
Poisson generator, whose fitted variance is near zero; on the empirical
near-closure sets it has no consistent sign; without a null
distribution this is a diagnostic, not the classical score test. The
gap certificate $R(0)-R(\hat v)$ of Theorem~\ref{cor:price_general},
the plug-in risk minus the risk at the fitted profile on the test set,
is a lower bound on the limiting price $\Delta_\text{bind}^\infty$, since the
fitted profile need not attain the optimum in $\mathcal{V}_\psi$.
Against the measured finite-$N$ frozen gap it is the larger quantity
on every overdispersed count set ($6.0$ against $2.4$ on bike, $9.6$ against $8.3$
on the negative-binomial generator, $15.4$ against $9.4$ on the
lognormal one) and equal to it within $0.01$ on the closure sets, a different
comparison: the finite-$N$ bound belief keeps some leverage and so
beats the limiting plug-in. On fluview and news the certificate is
uninformative, the Poisson score at the free mean being dominated by
the largest counts.
\input{tables/appendix_certificates}

\paragraph{The dispersion-parameter clause within the count family.}
Theorem~\ref{cor:price_general} makes the size of the price a
property of the likelihood's dispersion parameters, and the body tests
that only across families, Gaussian against Poisson.
Table~\ref{tab:appendix_nbbound} tests it within counts on the same
data: at the frozen features and mean of the negative-binomial head
(trained at the fixed $\lambda=0.01$), whose global dispersion $r$ is
kept as trained, a Gaussian belief on
the last layer is refit free by the NB score or bound by the NB ELBO,
the same two refits that price the Poisson binding. Under the
negative-binomial likelihood the bound belief costs $0.03$ and $0.01$
nats on the two overdispersed generators and $0.05$ on bike, a small
residual once the dispersion parameter absorbs the average
overdispersion; under the Poisson likelihood the same refits cost
$9.4$, $8.3$, and $2.4$. The whole-overdispersion price is a property
of the likelihood having no dispersion parameter, not of the data.
\input{tables/appendix_nbbound}

\paragraph{The two time-series sets under chronological splits.}
Bike and fluview are single series; the headline rows use random
splits with time covariates. Table~\ref{tab:appendix_probe_counts_chrono}
repeats every arm with the rows in time order (first $70\%$ train,
next $10\%$ validation, last $20\%$ test; seeds vary the
initialisation). Every posterior-seeking arm still loses, by more than
under random splits except the ELBO on bike and the bound refit on
fluview, which lose slightly less (ELBO $+1.9$
and $+3.5$ nats to the free belief, the covariance-bound refit $+11.6$
and $+8.5$, Laplace $+17.7$ and $+244$); on bike the free belief and
the count heads tie within $0.1$ nats, and on fluview the free
belief's direct counterpart, whose free per-input dispersion
extrapolates badly across the split, trails it by $5.0$. The mismatch
cell costs $0.7$ on bike and $0.2$ on fluview. The ordered rows enter
Table~\ref{tab:binding_gap}.
\input{tables/appendix_probe_counts_chrono}

\paragraph{The prior-term swap for counts.}
The KL form of Theorem~\ref{prop:binding_forms} minus the binding is
the PPGPR and direct-loss objective, which differs from $F_\text{SC}$
only in the prior term. With the Full and Diag beliefs trained under
$\mathrm{KL}(q\,\|\,p)$ in place of
$-\log Z_w$ on every count set, the paired difference to the
$F_\text{SC}$ belief is within $0.01$ nats in NLL and $0.01$ in PIT
distance on every set but one (Diag on fluview, $-0.07$;
Table~\ref{tab:appendix_klprior}, whose Full column is the prior-swap
column of Table~\ref{tab:binding_gap}), and within $0.06$ (standard
errors $0.04$--$0.10$) under the chronological splits, where the two
forms had the best chance to part. This matters for reading the count
control: its bound arm is an ELBO refit and its free arm a score
refit, so it changes the binding and the prior term together; the swap
being inert, the count differences of Table~\ref{tab:binding_gap} are
the binding.
\input{tables/appendix_klprior}

\paragraph{The ELBO belief's dynamics.}
The ELBO arm's median validation-best step is $5$ of $3{,}000$ on the
overdispersed generators at $N{=}2$k and $20$k ($225$ and $325$ at
$N{=}500$) and $77$ on news (Full; $95$--$600$ for
Diag), against $80$--$2{,}900$ for the free belief.
Figure~\ref{fig:elbo_curves} shows why: the belief's validation
predictive worsens from the first checks on, while its training
objective (recorded, not plotted) still decreases overall, so its
best-validation checkpoint sits near the initial belief before
the binding takes hold. Trained to the budget it moves toward the
Poisson corner ($4.1\to7.5$ on the lognormal generator); a tenfold
smaller learning rate delays the same trajectory and ends worse, and
tempering the KL term does not move it, the coupling being the
$\exp(m+v/2)$ of the data term. Binding the mean as well (mean head,
$\Sigma$, and $\alpha$ refit by the ELBO at the free features) lands
$1.9$--$62$ nats above the $F_\text{SC}$ belief, so it does not
recover the gap. On bike and fluview the early-stopped belief stops
late ($2{,}600$) with gaps of $2.1$ and $0.0$ under random splits.
\input{figures/elbo_curves}

\paragraph{The HMC posterior at the free features.}
What a sampled posterior-seeking belief reaches is available on the
generators: Table~\ref{tab:appendix_hmc} samples the last-layer
posterior by Hamiltonian Monte Carlo at the frozen features of the
$F_\text{SC}$-Full run (four chains, $1{,}000$ draws each after $800$
warm-up steps, diagonal mass matrix initialised at the Gauss--Newton
precision and reset to the inverse sample variance halfway through
warm-up, step size adapted
to $0.8$ acceptance), with the prior precision selected on validation.
The posterior predictive sits at the plug-in corner: $12$ nats above
the free belief on the lognormal generators and $5$ on the
negative-binomial ones, within $2$ nats of the Poisson head, and equal
to every other arm on the closure control; the Laplace belief lies
between, and so does the ELBO-bound refit on the lognormal generators
(on the negative-binomial ones it lies above the HMC predictive), so the early-stopped ELBO belief's
small gap is an artefact of stopping. The same sampler at the features
of the ELBO-trained model reaches the same corner ($15.2$--$16.5$ and
$7.3$--$7.8$), while the score refit at those features recovers to
within $0.35$ nats of the $F_\text{SC}$ belief, so the result does not
depend on which objective learned the features. At the
validation-selected prior precision the chains mix (Table~\ref{tab:appendix_hmc_diag}; median maximal split-$\widehat R$
$1.00$--$1.06$, median minimal effective sample sizes from $66$ to over $1{,}500$, Monte Carlo
error of the NLL at most $0.05$ nats), and doubling warm-up and draws
moves the rows by at most $0.07$ nats; at the base run's own, smaller
prior precision (its initial value $1$ on $12$--$20$ of $20$ seeds per
row, the validation-best step falling inside the $300$-step hold on
$\alpha$, Appendix~\ref{app:setup}) some coordinates mix
slowly ($\widehat R$ up to $2.0$), so those rows are a sensitivity analysis, landing at the
same corner within $0.7$ nats. On bike and fluview the chains mix
poorly (median $\widehat R$ $1.1$--$3.7$, single seeds up to $6.9$), so their predictive is left out of
Table~\ref{tab:appendix_hmc} and the claim rests on the generators.
\input{tables/appendix_hmc}
\input{tables/appendix_hmc_diag}

\paragraph{The expressibility clause, tested by the backbone.}
Proposition~\ref{prop:overdispersion} opens the channel only along a
positive direction of $\kappa$ that the belief family can express (the
sentence after it and Theorem~\ref{thm:price}(b)); at
the last layer the profile is $v_n=\psi_n^\top\Sigma\psi_n$, which
varies across inputs through the feature norm and, for a non-isotropic
$\Sigma$, the features' direction. Saturated tanh
features have nearly constant norm, so on the lognormal generator, where
$\kappa>0$ everywhere, the profile carries the average overdispersion but not
its variation across inputs: the tanh belief closes the plug-in gap as relu
does, yet its fitted variance has correlation $0.00$ with the true profile against
$0.84$--$0.87$ for relu, and it trails the best relu backbone by $0.07$--$0.13$
nats while the free per-input dispersion head, which does not go
through the quadratic form, is indifferent to the backbone on the
generators (within $0.02$), less so on bike and fluview ($0.24$ and
$0.29$; Table~\ref{tab:appendix_backbone_profile}). This is the
expressibility condition observed directly: the part of the score
without an expressible direction buys nothing. Validation selection
picks relu or relu+LN on every seed of the overdispersed generators
but four of twenty on the lognormal one at $N{=}500$, so the headline
($N{=}2$k) is unaffected, and on bike, whose profile is mild, three
backbones are within $0.05$ and relu trails by $0.19$.
\input{tables/appendix_backbone_profile}

\paragraph{The cavity cell.}
\label{app:ledger}
Table~\ref{tab:appendix_cavity} holds the leave-one-out cell of
Proposition~\ref{prop:cavity_invariance} at the trained
$F_\text{SC}$-Full features and mean. On the data-rich sets (kin8nm,
power, wine; $N/H\ge19$) the free-site criterion lands within
\priceCavity{} nats of the shared free belief (part of it the prior
term $-\log Z_w$, which the leave-one-out arms are trained without and
the shared arms keep); naval, near-noiseless, is the exception, the
free sites gaining $1.7$ nats (Corollary~\ref{cor:pathologies}); and
the bound belief at
the leave-one-out cavity coincides with the bound belief at the shared
cavity to within $0.001$; on the data-poor sets ($N/H$ between $4$ and
$12$) the free sites overfit their in-sample leave-one-out score, $N$
site parameters against $N$ plates. The statement compares criteria,
and the test scores here are the measurement.
\input{tables/appendix_cavity}

%% file: tables/appendix_probe_counts_gen.tex
\begin{table}[ht]
\centering
\scriptsize
\setlength{\tabcolsep}{2pt}%
\caption{Count likelihood, generators (test NLL, 20 seeds; validation-selected over four backbones, tuned arms also over their knob). PLN: Poisson-lognormal profile; NB: input-dependent negative-binomial dispersion; Poi: pure Poisson, the closure control; the number is the training-set size. Frozen rows: features and mean of \FSC-Full fixed, $\Sigma$ refit free by the score or bound by the ELBO. Last row: the oracle, the generator's own density, unranked. Bold = best; italic = not separated (paired $t$-test, $p\ge0.05$).}
\label{tab:appendix_probe_counts_gen}
\begin{tabular}{lccccccccc}
\toprule
 & PLN 500 & PLN 2k & PLN 20k & NB 500 & NB 2k & NB 20k & Poi 500 & Poi 2k & Poi 20k \\
\midrule
\FSC-Full & 3.068 & 2.987 & 2.960 & 2.487 & 2.455 & 2.439 & \textit{2.242} & 2.199 & 2.175 \\
\FSC-Diag & 3.057 & 2.991 & \textit{2.955} & 2.501 & 2.456 & 2.435 & \textbf{2.236} & \textit{2.198} & \textit{2.172} \\
\FSC-None & 16.878 & 15.110 & 14.070 & 7.755 & 7.560 & 7.319 & \textit{2.245} & 2.201 & 2.175 \\
frozen, $\Sigma$ free & 3.063 & 2.982 & \textit{2.956} & 2.489 & 2.455 & 2.435 & \textit{2.241} & \textit{2.198} & 2.173 \\
frozen, $\Sigma$ bound & 12.401 & 12.375 & 12.181 & 9.001 & 10.768 & 11.216 & 2.248 & 2.200 & 2.175 \\
mismatch cell, Full & 3.284 & 3.343 & 3.304 & 2.621 & 2.569 & 2.563 & \textit{2.240} & \textbf{2.194} & 2.174 \\
mismatch cell, Diag & 3.108 & 3.084 & 3.060 & 2.695 & 2.677 & 2.652 & \textit{2.248} & 2.204 & 2.186 \\
ELBO, Full & 3.936 & 4.134 & 4.491 & 3.213 & 3.565 & 3.749 & \textit{2.246} & 2.200 & 2.173 \\
ELBO, Diag & 3.313 & 3.316 & 3.282 & 2.701 & 2.702 & 2.693 & \textit{2.248} & 2.207 & 2.192 \\
ELBO, Full, to the budget & -- & 7.496 & -- & -- & 5.958 & -- & -- & -- & -- \\
ELBO, Full, tempered & -- & 4.180 & -- & -- & 3.681 & -- & -- & -- & -- \\
Laplace & 9.778 & 13.234 & 13.984 & 5.878 & 7.021 & 7.273 & \textit{2.244} & \textit{2.196} & 2.173 \\
Poisson & 17.303 & 15.222 & 14.073 & 7.740 & 7.587 & 7.314 & 2.256 & 2.199 & 2.174 \\
NB, global dispersion & 3.168 & 3.115 & 3.092 & 2.585 & 2.573 & 2.558 & \textit{2.249} & 2.204 & \textbf{2.171} \\
NB, per-input dispersion & 3.089 & 3.030 & 3.007 & \textbf{2.454} & \textbf{2.423} & \textbf{2.409} & \textit{2.240} & \textit{2.195} & 2.173 \\
PLN, free per-input dispersion & \textbf{3.009} & \textbf{2.969} & \textbf{2.955} & 2.500 & 2.451 & 2.432 & \textit{2.244} & \textit{2.195} & \textit{2.172} \\
\midrule
oracle (generator) & 2.938 & 2.941 & 2.939 & 2.396 & 2.396 & 2.399 & 2.134 & 2.134 & 2.135 \\
\bottomrule
\end{tabular}
\end{table}

%% file: tables/appendix_probe_counts_real.tex
\begin{table}[ht]
\centering
\scriptsize
\setlength{\tabcolsep}{2pt}%
\caption{Count likelihood, real data sets (as the previous table): bike rentals, weekly influenza counts, article shares, abalone rings, hospital length of stay, motor claim counts.}
\label{tab:appendix_probe_counts_real}
\begin{tabular}{lcccccc}
\toprule
 & bike & fluview & news & abalone & stay & claims \\
\midrule
\FSC-Full & 4.881 & \textit{8.643} & 8.750 & \textit{2.250} & 1.629 & 0.216 \\
\FSC-Diag & \textbf{4.833} & \textit{8.673} & \textbf{8.730} & \textit{2.250} & 1.629 & \textit{0.215} \\
\FSC-None & 10.148 & 125.586 & 2849.147 & \textbf{2.249} & \textit{1.615} & 0.217 \\
frozen, $\Sigma$ free & \textit{4.846} & 8.858 & 8.741 & \textit{2.250} & 1.628 & \textbf{0.214} \\
frozen, $\Sigma$ bound & 7.243 & 18.618 & 56.895 & \textit{2.261} & 1.629 & 0.222 \\
mismatch cell, Full & 5.371 & 9.050 & 11.029 & 2.253 & 1.628 & \textit{0.215} \\
mismatch cell, Diag & 5.440 & 9.508 & 8.889 & \textit{2.250} & 1.633 & \textit{0.215} \\
ELBO, Full & 6.984 & 8.615 & 9.357 & \textit{2.250} & 1.628 & 0.217 \\
ELBO, Diag & 5.826 & \textit{8.530} & 8.913 & 2.256 & 1.722 & 0.216 \\
ELBO, Full, to the budget & 7.202 & 8.695 & \textit{16.259} & 2.252 & 1.630 & 0.245 \\
ELBO, Full, tempered & 6.790 & 8.631 & 9.367 & 2.251 & 1.629 & 0.218 \\
Laplace & 9.998 & 83.875 & 2334.491 & \textit{2.251} & 1.617 & 0.217 \\
Poisson & 10.054 & 121.878 & 2836.033 & \textit{2.250} & 1.616 & 0.217 \\
NB, global dispersion & 4.919 & \textbf{8.505} & 9.021 & 2.255 & 1.619 & \textit{0.215} \\
NB, per-input dispersion & \textit{4.851} & 9.010 & 9.016 & 2.251 & \textit{1.615} & \textit{0.215} \\
PLN, free per-input dispersion & 4.884 & 8.687 & \textit{8.734} & 2.251 & \textbf{1.615} & 0.218 \\
\bottomrule
\end{tabular}
\end{table}

%% file: tables/appendix_certificates.tex
\begin{table}[h]
\centering\scriptsize
\setlength{\tabcolsep}{3pt}
\caption{Certificates from the saved test predictions (means over seeds). Predicted and measured are taken as in Table~\ref{tab:binding_gap}, at the backbone selected for \FSC-Full and, for the generators, $N{=}2{,}000$; the other columns pool backbones and training sizes. $\hat\kappa>0$: fraction of test points with a positive empirical score at the Poisson head; monotonicity: rank correlation between the decile means of $\hat\kappa/\lambda^2$ and of the fitted variance $\hat v$, per backbone (relu / relu+LN / tanh / tanh+LN); $\bar{\hat v}$: mean fitted belief variance; predicted: $R(0)-R(\hat v)$ (Theorem~\ref{cor:price_general}); measured: frozen-bound minus frozen-free test NLL.}
\label{tab:appendix_certificates}
\begin{tabular}{@{}lrlrrr@{}}
\toprule
 & $\hat\kappa>0$ & monotonicity & $\bar{\hat v}$ & predicted & measured \\
\midrule
PLN & 0.69 & 0.82 / 0.80 / 0.45 / 0.42 & 1.177 & 15.41 & 9.39 \\
NB & 0.68 & 0.91 / 0.90 / 0.55 / 0.47 & 2.519 & 9.58 & 8.31 \\
Poisson (closure) & 0.35 & 0.30 / 0.19 / 0.63 / 0.60 & 0.045 & 0.00 & 0.00 \\
bike & 0.79 & 0.97 / 0.99 / 0.96 / 0.96 & 0.249 & 5.97 & 2.40 \\
fluview & 0.90 & 0.75 / 0.55 / 0.68 / 0.56 & 0.012 & 128.84 & 9.76 \\
news & 0.99 & 0.49 / 0.31 / 0.31 / 0.26 & 0.827 & 3205.24 & 48.15 \\
abalone & 0.10 & $-$0.23 / 0.21 / $-$0.19 / 0.27 & 0.004 & $-$0.00 & 0.01 \\
stay & 0.02 & $-$0.63 / $-$0.34 / $-$0.47 / 0.58 & 0.006 & $-$0.01 & 0.00 \\
claims & 0.05 & 0.17 / $-$0.06 / 0.12 / 0.07 & 1.112 & 0.01 & 0.01 \\
\bottomrule
\end{tabular}
\end{table}

%% file: tables/appendix_nbbound.tex
\begin{table}[h]
\centering\scriptsize
\setlength{\tabcolsep}{3pt}
\caption{The dispersion-parameter clause of Theorem~\ref{cor:price_general} within the count family. Left: the negative-binomial head with a global dispersion (a likelihood with a dispersion parameter), and at its frozen features and mean a Gaussian belief on the last layer refit free by the NB score or bound by the NB ELBO. Right: the same two refits at the Poisson-lognormal belief's features (the Poisson likelihood, no dispersion parameter). Test NLL, mean over seeds, and the paired bound-minus-free gap: a small residual where the dispersion parameter absorbs the average overdispersion, the whole overdispersion where none does.}
\label{tab:appendix_nbbound}
\begin{tabular}{@{}lrrrrrrrrr@{}}
\toprule
 & \multicolumn{4}{c}{negative binomial, global $r$} & \multicolumn{4}{c@{}}{Poisson (\FSC belief)} & \\
\cmidrule(lr){2-5}\cmidrule(l){6-9}
 & head & $\Sigma$ free & $\Sigma$ bound & gap & \FSC-Full & $\Sigma$ free & $\Sigma$ bound & gap & seeds \\
\midrule
Poisson-lognormal ($N{=}2$k) & 3.11 & 3.08 & 3.11 & $+0.03$ & 2.99 & 2.98 & 12.37 & $+9.39$ & 20 \\
neg.\ binomial ($N{=}2$k) & 2.57 & 2.55 & 2.57 & $+0.01$ & 2.45 & 2.46 & 10.77 & $+8.31$ & 20 \\
bike & 4.92 & 4.87 & 4.91 & $+0.05$ & 4.88 & 4.85 & 7.24 & $+2.40$ & 20 \\
\bottomrule
\end{tabular}
\end{table}

%% file: tables/appendix_probe_counts_chrono.tex
\begin{table}[ht]
\centering
\scriptsize
\setlength{\tabcolsep}{4pt}%
\caption{Count likelihood, the two time-series sets under chronological splits (first 70\% of the rows train, next 10\% validation, last 20\% test; test NLL, 20 seeds varying the initialisation, validation-selected over four backbones, tuned arms also over their knob). Bold = best; italic = not separated (paired $t$-test, $p\ge0.05$).}
\label{tab:appendix_probe_counts_chrono}
\begin{tabular}{lcc}
\toprule
 & bike & fluview \\
\midrule
\FSC-Full & 5.759 & 10.404 \\
\FSC-Diag & \textbf{5.552} & 10.557 \\
\FSC-None & 24.723 & 1491.381 \\
frozen, $\Sigma$ free & 5.751 & 10.421 \\
frozen, $\Sigma$ bound & 17.308 & 18.936 \\
mismatch cell, Full & 6.450 & 10.649 \\
mismatch cell, Diag & 6.471 & 10.436 \\
ELBO, Full & 7.687 & 13.916 \\
ELBO, Diag & 6.424 & \textbf{10.246} \\
ELBO, Full, to the budget & 13.306 & 16.959 \\
ELBO, Full, tempered & 7.716 & 14.526 \\
Laplace & 23.493 & 253.955 \\
Poisson & 24.402 & 1493.718 \\
NB, global dispersion & 5.799 & 11.614 \\
NB, per-input dispersion & \textit{5.658} & 12.455 \\
PLN, free per-input dispersion & \textit{5.673} & 15.397 \\
\bottomrule
\end{tabular}
\end{table}

%% file: tables/appendix_klprior.tex
\begin{table}[h]
\centering\scriptsize
\setlength{\tabcolsep}{4pt}
\caption{The prior-term swap beyond the Gaussian likelihood: the score-trained belief with $\mathrm{KL}(q\,\|\,p)$ in place of $-\log Z_w$ (the PPGPR / direct-loss form, $\alpha$ learned), paired difference to the \FSC belief of the same covariance structure (positive = the KL form is worse; test NLL and PIT distance, mean over 20 seeds, backbone by validation).}
\label{tab:appendix_klprior}
\begin{tabular}{@{}lrrrr@{}}
\toprule
 & \multicolumn{2}{c}{Full} & \multicolumn{2}{c@{}}{Diag} \\
\cmidrule(lr){2-3}\cmidrule(l){4-5}
 & NLL & PIT & NLL & PIT \\
\midrule
PLN 2k & $-0.000$ & $-0.000$ & $-0.004$ & $+0.001$ \\
NB 2k & $-0.003$ & $-0.002$ & $-0.003$ & $-0.002$ \\
Poi 2k (closure) & $-0.003$ & $-0.001$ & $+0.001$ & $-0.001$ \\
bike & $-0.012$ & $-0.001$ & $+0.001$ & $+0.004$ \\
fluview & $+0.007$ & $+0.009$ & $-0.072$ & $+0.014$ \\
news & $+0.000$ & $-0.000$ & $-0.001$ & $-0.001$ \\
abalone & $+0.003$ & $-0.002$ & $+0.003$ & $-0.001$ \\
stay & $-0.000$ & $-0.000$ & $+0.001$ & $+0.001$ \\
claims & $-0.000$ & $+0.000$ & $-0.000$ & $+0.000$ \\
\bottomrule
\end{tabular}
\end{table}

%% file: figures/elbo_curves.tex
\begin{figure}[t]
\centering
\includegraphics[width=\linewidth]{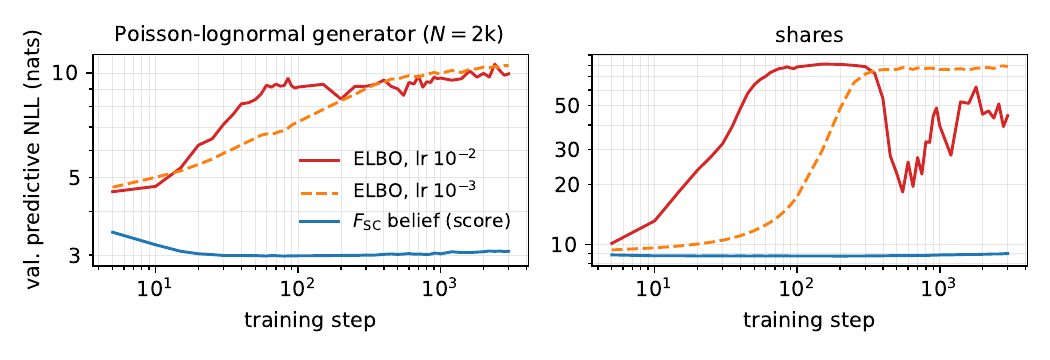}
\caption{\textbf{The ELBO belief moves away from the predictive optimum
as it trains.} Validation predictive NLL against the training step
(median over 20 seeds, relu+LN backbone, log axes) for the ELBO belief trained to the
step budget at the headline learning rate and at a tenfold smaller one,
beside the score-trained $F_\text{SC}$ belief, on the Poisson-lognormal
generator and on shares. The ELBO belief's validation predictive
worsens from the first checks on; the
smaller learning rate delays the same trajectory; on shares the
headline-rate run is unstable after step $300$. Validation early
stopping therefore keeps the ELBO arm of Table~\ref{tab:binding_gap}
near its initial belief.}
\label{fig:elbo_curves}
\end{figure}

%% file: tables/appendix_hmc.tex
\begin{table}[h]
\centering\scriptsize
\setlength{\tabcolsep}{2pt}
\caption{The HMC posterior predictive at the frozen features: Hamiltonian Monte Carlo over the last layer (four chains, $1{,}000$ draws each after $800$ warm-up steps unless stated) at the frozen features of the \FSC-Full run and of the ELBO-trained run, with that run's prior precision (on most generator seeds its initial value $1$, the validation-best step falling inside the $300$-step hold on $\alpha$) or one selected on validation over $\{10^{-2},\dots,10^{2}\}$, beside the free and covariance-bound refits at the same features, the Laplace belief, and the Poisson head (test NLL, mean over seeds, backbone by validation; chain diagnostics in Table~\ref{tab:appendix_hmc_diag}). The sampled posterior-seeking belief sits at the plug-in corner; on the closure control every arm coincides.}
\label{tab:appendix_hmc}
\begin{tabular}{@{}lrrrrrrrrr@{}}
\toprule
 & oracle & \FSC-Full & frozen, $\Sigma$ free & frozen, $\Sigma$ bound & HMC, base run's $\alpha$ & HMC, $\alpha$ on validation & Laplace & Poisson & seeds \\
\midrule
PLN 500 & 2.94 & 3.05 & 3.05 & 12.22 & 16.14 & 15.48 & 9.78 & 17.30 & 20 \\
NB 500 & 2.40 & 2.48 & 2.48 & 9.08 & 7.46 & 7.40 & 5.88 & 7.74 & 20 \\
Poi 500 (closure) & 2.13 & 2.24 & 2.24 & 2.25 & 2.25 & 2.24 & 2.24 & 2.26 & 20 \\
PLN 2k & 2.94 & 2.99 & 2.98 & 12.37 & 15.21 & 14.78 & 13.23 & 15.22 & 20 \\
NB 2k & 2.40 & 2.45 & 2.46 & 10.77 & 7.44 & 7.35 & 7.02 & 7.59 & 20 \\
\midrule
\multicolumn{10}{@{}l}{\emph{at the features of the ELBO-trained belief (its own row in the second column)}}\\
PLN 500 & 2.94 & 4.46 & 3.25 & 13.25 & 16.53 & 15.86 & 9.78 & 17.30 & 20 \\
NB 500 & 2.40 & 3.63 & 2.71 & 13.59 & 7.85 & 7.47 & 5.88 & 7.74 & 20 \\
Poi 500 (closure) & 2.13 & 2.24 & 2.24 & 2.24 & 2.25 & 2.23 & 2.24 & 2.26 & 20 \\
PLN 2k & 2.94 & 4.13 & 3.31 & 16.89 & 15.57 & 15.22 & 13.23 & 15.22 & 20 \\
NB 2k & 2.40 & 3.56 & 2.65 & 9.59 & 7.41 & 7.34 & 7.02 & 7.59 & 20 \\
\midrule
\multicolumn{10}{@{}l}{\emph{doubled chains at the \FSC{} features ($1{,}600$ warm-up, $2{,}000$ draws per chain)}}\\
PLN 500 & 2.94 & 3.05 & 3.05 & 12.22 & 16.08 & 15.41 & 9.78 & 17.30 & 20 \\
NB 500 & 2.40 & 2.48 & 2.48 & 9.08 & 7.43 & 7.37 & 5.88 & 7.74 & 20 \\
Poi 500 (closure) & 2.13 & 2.24 & 2.24 & 2.25 & 2.25 & 2.24 & 2.24 & 2.26 & 20 \\
PLN 2k & 2.94 & 2.99 & 2.98 & 12.37 & 15.18 & 14.74 & 13.23 & 15.22 & 20 \\
NB 2k & 2.40 & 2.45 & 2.46 & 10.77 & 7.43 & 7.35 & 7.02 & 7.59 & 20 \\
\bottomrule
\end{tabular}
\end{table}

%% file: tables/appendix_hmc_diag.tex
\begin{table}[htp]
\centering\scriptsize
\setlength{\tabcolsep}{2.2pt}
\caption{HMC diagnostics for the runs of Table~\ref{tab:appendix_hmc} and for the bike and fluview runs it leaves out (the validation-selected backbone per seed), per arm and data set over seeds: split-$\widehat R$ (maximum over the $H{+}1$ weight coordinates; median over seeds, and the share of seeds with every coordinate below $1.1$), effective sample size of the pooled draws ($4{,}000$; $8{,}000$ for doubled chains; split chains, no rank normalisation; minimum and median over coordinates, median over seeds), the Monte Carlo standard error of the test NLL (between-chain standard deviation over $\sqrt{4}$; median and worst seed), and the acceptance rate.}
\label{tab:appendix_hmc_diag}
\begin{tabular}{@{}llrrrrrrrr@{}}
\toprule
 arm & data & $\widehat R$ med. & $\widehat R{<}1.1$ & ESS min & ESS med. & MCSE med. & MCSE max & acc. & seeds \\
\midrule
\FSC features, base run's $\alpha$ & PLN 500 & 1.054 & 55\% & 85 & 258 & 0.0144 & 0.0641 & 0.82 & 20 \\
\FSC features, base run's $\alpha$ & NB 500 & 1.024 & 60\% & 220 & 525 & 0.0066 & 0.0229 & 0.81 & 20 \\
\FSC features, base run's $\alpha$ & Poi 500 (closure) & 1.030 & 60\% & 150 & 405 & 0.0004 & 0.0020 & 0.80 & 20 \\
\FSC features, base run's $\alpha$ & PLN 2k & 2.029 & 15\% & 13 & 23 & 0.0066 & 0.0134 & 0.82 & 20 \\
\FSC features, base run's $\alpha$ & NB 2k & 1.305 & 50\% & 32 & 92 & 0.0025 & 0.0066 & 0.81 & 20 \\
\FSC features, $\alpha$ on val. & PLN 500 & 1.011 & 60\% & 550 & 1209 & 0.0131 & 0.0288 & 0.82 & 20 \\
\FSC features, $\alpha$ on val. & NB 500 & 1.013 & 70\% & 457 & 770 & 0.0056 & 0.0207 & 0.81 & 20 \\
\FSC features, $\alpha$ on val. & Poi 500 (closure) & 1.019 & 70\% & 361 & 698 & 0.0003 & 0.0022 & 0.81 & 20 \\
\FSC features, $\alpha$ on val. & PLN 2k & 1.021 & 90\% & 222 & 548 & 0.0047 & 0.0217 & 0.80 & 20 \\
\FSC features, $\alpha$ on val. & NB 2k & 1.025 & 75\% & 259 & 638 & 0.0022 & 0.0044 & 0.81 & 20 \\
ELBO features, base run's $\alpha$ & PLN 500 & 1.505 & 0\% & 16 & 31 & 0.0167 & 0.0556 & 0.80 & 20 \\
ELBO features, base run's $\alpha$ & NB 500 & 1.201 & 10\% & 28 & 58 & 0.0082 & 0.0354 & 0.82 & 20 \\
ELBO features, base run's $\alpha$ & Poi 500 (closure) & 1.067 & 50\% & 114 & 308 & 0.0004 & 0.0014 & 0.81 & 20 \\
ELBO features, base run's $\alpha$ & PLN 2k & 2.008 & 0\% & 13 & 19 & 0.0055 & 0.0562 & 0.81 & 20 \\
ELBO features, base run's $\alpha$ & NB 2k & 1.617 & 0\% & 15 & 20 & 0.0027 & 0.0038 & 0.80 & 20 \\
ELBO features, $\alpha$ on val. & PLN 500 & 1.060 & 55\% & 66 & 172 & 0.0137 & 0.0436 & 0.80 & 20 \\
ELBO features, $\alpha$ on val. & NB 500 & 1.004 & 70\% & 1539 & 2357 & 0.0059 & 0.0314 & 0.80 & 20 \\
ELBO features, $\alpha$ on val. & Poi 500 (closure) & 1.003 & 85\% & 1335 & 2839 & 0.0002 & 0.0017 & 0.82 & 20 \\
ELBO features, $\alpha$ on val. & PLN 2k & 1.029 & 85\% & 167 & 429 & 0.0047 & 0.0187 & 0.81 & 20 \\
ELBO features, $\alpha$ on val. & NB 2k & 1.028 & 55\% & 221 & 523 & 0.0026 & 0.0090 & 0.80 & 20 \\
\FSC feat., $\alpha$ on val., $2\times$ chains & PLN 500 & 1.005 & 60\% & 1137 & 2376 & 0.0126 & 0.0461 & 0.82 & 20 \\
\FSC feat., $\alpha$ on val., $2\times$ chains & NB 500 & 1.008 & 80\% & 900 & 1300 & 0.0035 & 0.0188 & 0.80 & 20 \\
\FSC feat., $\alpha$ on val., $2\times$ chains & Poi 500 (closure) & 1.007 & 80\% & 853 & 1384 & 0.0002 & 0.0035 & 0.81 & 20 \\
\FSC feat., $\alpha$ on val., $2\times$ chains & PLN 2k & 1.013 & 95\% & 441 & 1017 & 0.0038 & 0.0180 & 0.83 & 20 \\
\FSC feat., $\alpha$ on val., $2\times$ chains & NB 2k & 1.027 & 75\% & 186 & 391 & 0.0015 & 0.0058 & 0.80 & 20 \\
\FSC features, base run's $\alpha$ & bike & 1.383 & 5\% & 17 & 58 & 0.0004 & 0.0007 & 0.82 & 20 \\
\FSC features, $\alpha$ on val. & bike & 3.748 & 0\% & 11 & 12 & 0.0007 & 0.0051 & 0.82 & 20 \\
\FSC features, base run's $\alpha$ & bike, ordered split & 1.102 & 45\% & 48 & 165 & 0.0142 & 0.0392 & 0.81 & 20 \\
\FSC features, $\alpha$ on val. & bike, ordered split & 1.701 & 15\% & 14 & 28 & 0.0298 & 0.1317 & 0.82 & 20 \\
\FSC features, base run's $\alpha$ & fluview & 1.191 & 15\% & 30 & 61 & 0.0840 & 0.2527 & 0.81 & 20 \\
\FSC features, $\alpha$ on val. & fluview & 1.193 & 15\% & 30 & 60 & 0.0782 & 0.4029 & 0.80 & 20 \\
\FSC features, base run's $\alpha$ & fluview, ordered split & 1.150 & 30\% & 36 & 79 & 3.2530 & 6.2940 & 0.81 & 20 \\
\FSC features, $\alpha$ on val. & fluview, ordered split & 1.134 & 30\% & 36 & 95 & 3.0958 & 18.9241 & 0.80 & 20 \\
\bottomrule
\end{tabular}
\end{table}

%% file: tables/appendix_backbone_profile.tex
\begin{table}[h]
\centering\scriptsize
\setlength{\tabcolsep}{3pt}
\caption{The count profile per backbone (test NLL, mean over 20 seeds, no architecture selection): the shared quadratic form $\psi^\top\Sigma\psi$ needs feature variation across inputs, which saturated tanh features do not provide, while a free per-input dispersion head is indifferent to the backbone on the generators. Validation selection picks relu or relu+LN on every seed of the overdispersed generators except four of twenty on the lognormal one at $N{=}500$.}
\label{tab:appendix_backbone_profile}
\begin{tabular}{@{}llrrrr@{}}
\toprule
head & data & relu & relu+LN & tanh & tanh+LN \\
\midrule
\FSC-Full & PLN 20k & 2.975 & 2.960 & 3.036 & 3.034 \\
\FSC-Full & NB 20k & 2.444 & 2.439 & 2.570 & 2.567 \\
\FSC-Full & bike & 5.099 & 4.905 & 4.950 & 4.922 \\
\FSC-Full & fluview & 8.642 & 8.362 & 8.496 & 8.354 \\
PLN, free per-input dispersion & PLN 20k & 2.956 & 2.957 & 2.954 & 2.957 \\
PLN, free per-input dispersion & NB 20k & 2.432 & 2.436 & 2.446 & 2.448 \\
PLN, free per-input dispersion & bike & 5.093 & 5.139 & 4.898 & 4.947 \\
PLN, free per-input dispersion & fluview & 8.931 & 8.699 & 8.643 & 8.673 \\
\bottomrule
\end{tabular}
\end{table}

%% file: tables/appendix_cavity.tex
\begin{table}[ht]
\centering
\scriptsize
\setlength{\tabcolsep}{2pt}%
\caption{The cavity cell at frozen features and fixed mean (Gaussian likelihood, UCI, 10 seeds, relu+LN): the same trained \FSC-Full features and mean, the variance side refit at the shared cavity ($\Sigma$ free or bound) and at the leave-one-out cavity (bound; free per-plate sites; sites tied through the head's quadratic form, i.e.\ the exact leave-one-out inference of a heteroscedastic Gaussian model), plus the same heteroscedastic conjugate image trained by the shared score. The leave-one-out rows are trained without the prior term $-\log Z_w$, which the shared rows keep, so part of their difference is the prior term. Test NLL; bold = best; italic = not separated (paired $t$-test, $p\ge0.05$).}
\label{tab:appendix_cavity}
\begin{tabular}{lcccccccc}
\toprule
 & yacht & concrete & energy & kin8nm & naval & power & wine & boston \\
\midrule
\FSC-Full (base) & 2.594 & \textbf{3.272} & \textit{1.087} & \textit{$-$0.714} & $-$1.695 & \textbf{2.986} & \textbf{0.985} & \textit{2.778} \\
shared cavity, $\Sigma$ free & \textit{2.427} & \textit{3.276} & \textbf{1.020} & \textbf{$-$0.715} & $-$1.688 & \textit{2.990} & \textit{0.989} & \textit{2.734} \\
shared cavity, $\Sigma$ bound & \textit{6.760} & 3.507 & 2.051 & $-$0.597 & $-$3.267 & 3.033 & 1.008 & \textit{3.682} \\
leave-one-out cavity, bound & \textit{9.585} & 3.547 & 2.058 & $-$0.598 & $-$3.268 & 3.033 & 1.008 & \textit{3.758} \\
leave-one-out cavity, free sites & \textit{10.725} & 3.614 & 1.846 & $-$0.668 & \textbf{$-$3.409} & 3.007 & 0.999 & \textit{5.923} \\
leave-one-out cavity, tied sites (het.\ likelihood) & \textit{2.517} & \textit{3.313} & 1.659 & $-$0.695 & $-$2.198 & 3.004 & \textit{0.993} & \textit{2.694} \\
shared cavity, heteroscedastic conjugate image & \textbf{2.035} & \textit{3.318} & 1.613 & $-$0.690 & $-$2.173 & 2.999 & \textit{0.993} & \textbf{2.684} \\
\bottomrule
\end{tabular}
\end{table}

%% file: sections/app_beyond_last_layer.tex
\section{Beyond the Last Layer}
\label{app:beyond_last_layer}

The identity and the first-order channel hold on any directed factor
graph; the body measures on the reduced last-layer graph, where the
binding is the whole observation-dependent remainder beyond the prior
KL, and quotes its depth results from here. Below the last layer,
Theorem~\ref{prop:binding_forms} carries two more pieces, the
prior couplings and the observation-free intermediate couplings
$\mathrm{KL}(q_a\,\|\,b_a)$ of the activation factors, zero on the
reduced graph (Proposition~\ref{rem:deterministic}) and not zero
once the hidden weights carry a belief. This appendix repeats the
count experiment with the belief below the last layer, under the
practical construction that makes $F_\text{SC}$ computable there; it is
exploratory. Gaussian-likelihood runs of the same construction serve
here only to audit the moment matching
(Table~\ref{tab:appendix_depth_mm}); their results are not reported.

\paragraph{The practical construction.}
Let the belief cover the top $k$ hidden layers as well as the head,
diagonal Gaussian per layer, with the weights below at a point mass
(in the count runs with a fixed variance of $10^{-4}$ as a numerical
floor; the moment-matching audit uses exact point masses) under their
own learned prior precision, a weight-decay term: unlike
the reduced graph of Proposition~\ref{rem:deterministic}, these layers
keep a prior factor. The plate marginal $q(f_n)$ is then no
longer Gaussian; it is replaced by its moment-matched Gaussian,
propagated layer by layer through closed-form relu moments with
cross-unit correlations dropped, as in probabilistic
backpropagation~\citep{HernandezLobato2015}. $F_\text{SC}$ at depth is
the local-normaliser sum on these marginals: the plate term is the
same convolution of the likelihood with $\mathcal{N}(m_n,v_n)$ as at
the last layer, so any likelihood of Section~\ref{sec:method} applies
unchanged, and the prior term is one $-\log Z_{w_l}$ per belief layer
with its own $\alpha_l$ learned in the pass. Three objectives are
trained on the identical family and forward pass: the score-trained
belief (S); the deep PPGPR form (P), the score's plate term with the
ELBO's per-layer $\mathrm{KL}(q_l\,\|\,p_l)$, so that E$-$P is the
binding at depth and P$-$S the prior-term swap with the intermediate
couplings (below); and the mean-field
ELBO (E), the expected log-likelihood on the same marginals with the
same prior KL, early-stopped on the validation score like every other
arm, and for counts also trained to the step budget (E$_b$). Ten seeds
for the count runs, three for the audit of
Table~\ref{tab:appendix_depth_mm}; three-hidden-layer relu networks
(four in the full-belief $D{=}4$ column of
Table~\ref{tab:appendix_depth_counts}), $H{=}50$, the count runs in
double precision and full batch. E$-$P is the
plate couplings within the construction, exact by~\eqref{eq:coupling}
whatever produced $q(f_n)$, while by the KL form of
Theorem~\ref{prop:binding_forms} P$-$S carries the prior and
intermediate couplings together; what the construction approximates is
$q(f_n)$ itself, so the deep numbers are the construction's and not an exact deep price (exact
deep marginals are open; the nearest instance is the deep sigma-point
process~\citep{Jankowiak2020dspp}).

\paragraph{The price of the construction.}
Table~\ref{tab:appendix_depth_mm} measures what the moment matching
costs: the mean absolute per-plate difference between the
moment-matched score and the exact mixture score at the same trained
parameters, the hidden-layer weights sampled and the head and noise
marginalised analytically. It is zero at $k{=}0$ by construction,
below $0.1$ nats at $k{=}1$, and at $k{=}2,3$ $0.1$--$0.5$ on average
over the five sets ($0.02$--$0.8$ per set); the ELBO belief, whose
weights carry a smaller share of the predictive variance ($0.09$--$0.16$
against $0.29$--$0.55$), pays a third to a half of what the
score-trained belief pays, so the objectives are not compared at
equal approximation error.
\input{tables/appendix_depth_mm}

\paragraph{The count channel at depth.}
Table~\ref{tab:appendix_depth_counts} runs the count generators and
the two real count sets of Section~\ref{sec:experiments} with the
belief below the last layer; its $k{=}0$ cell is this construction's
diagonal belief, not the last-layer head of
Section~\ref{sec:experiments}. The exact-mixture audit of
Table~\ref{tab:appendix_depth_mm} covers the Gaussian sets only and none
is run for the count construction, so at $k\ge1$ early-stopped
differences of a few tenths are within what that approximation could
move; at $k{=}0$ the marginal is Gaussian up to the $10^{-4}$ floor. On the overdispersed generators the
score-trained belief sits within $0.07$ nats of the oracle at every
depth, the PPGPR form within $0.01$ of it, the early-stopped ELBO
belief $0.1$--$0.4$ above, and the same belief trained to the budget
$2.8$--$10$ above (at the last layer that drift ends at the Poisson
corner, Appendix~\ref{app:beyond_gaussian}; no Poisson reference is
run at depth); on the closure control every arm, against the
score-trained belief, is within $0.07$ except the budget-trained belief at $k{=}0$,
$N{=}500$ ($+0.24$). On bike the early-stopped ELBO belief trails by
$0.5$--$0.8$ nats under random splits and $0.7$--$1.1$ under the
chronological split, the budget-trained one by $1.2$--$1.5$ and
$17$--$26$; on fluview the random split is the near-closure it is at
the last layer (every early-stopped arm within $0.2$) and the
chronological split opens the gap to $1.8$--$7.9$ (E$-$P) for the early-stopped
belief and $28$--$56$ to the budget. P$-$S is within
$0.07$ on the generators and bike's random split and moves the
chronological rows by up to $1.9$ nats, the KL
prior ahead on fluview and the $-\log Z_w$ term by $0.5$ on bike at
$k{=}1$.
\input{tables/appendix_depth_counts}

%% file: tables/appendix_depth_mm.tex
\begin{table}[h]
\centering\scriptsize
\setlength{\tabcolsep}{4pt}
\caption{The price of the construction: mean absolute per-plate difference (nats) between the moment-matched score and the exact mixture score at the same trained parameters (hidden-layer weights sampled, head and noise marginalised analytically, $4{,}096$ antithetic samples), $H{=}50$, per objective and belief depth, mean over 3 seeds; at $k{=}0$ the point layers are deterministic and the difference is zero by construction.}
\label{tab:appendix_depth_mm}
\begin{tabular}{@{}lrrrrrrrrr@{}}
\toprule
 & \multicolumn{3}{c}{$k{=}1$} & \multicolumn{3}{c}{$k{=}2$} & \multicolumn{3}{c}{$k{=}3$} \\
 & S & P & E & S & P & E & S & P & E \\
\midrule
yacht & 0.01 & 0.04 & 0.04 & 0.25 & 0.11 & 0.13 & 0.51 & 0.17 & 0.12 \\
energy & 0.00 & 0.01 & 0.00 & 0.23 & 0.38 & 0.09 & 0.29 & 0.62 & 0.17 \\
concrete & 0.00 & 0.01 & 0.00 & 0.41 & 0.52 & 0.22 & 0.49 & 0.58 & 0.23 \\
power & 0.01 & 0.01 & 0.00 & 0.25 & 0.34 & 0.02 & 0.08 & 0.37 & 0.04 \\
kin8nm & 0.09 & 0.05 & 0.00 & 0.43 & 0.63 & 0.11 & 0.43 & 0.81 & 0.22 \\
\midrule
mean & 0.02 & 0.02 & 0.01 & 0.31 & 0.40 & 0.11 & 0.36 & 0.51 & 0.16 \\
\bottomrule
\end{tabular}
\end{table}

%% file: tables/appendix_depth_counts.tex
\begin{table}[h]
\centering\scriptsize
\setlength{\tabcolsep}{1.4pt}
\caption{The count channel at depth: Poisson likelihood, a three-hidden-layer relu network ($H{=}50$) with a diagonal Gaussian belief on the top $k$ hidden layers and point weights with a learned prior precision below, plus the full four-layer belief, moment-matched forward pass and the adaptive-quadrature plate term of the count head. The score-trained belief (S), the deep PPGPR form (P: the score's plate term with the ELBO's per-layer prior KL), the mean-field ELBO early-stopped on validation (E) and trained to the step budget (E$_b$); test NLL, median over seeds; the generators' oracle is the NLL under the true model. The $k{=}2$ cell is omitted for width: it lies between its neighbours on every generator row and for S on every row, while six E, E$_b$, or P cells of bike and fluview fall outside by up to $0.24$.}
\label{tab:appendix_depth_counts}
\begin{tabular}{@{}lrrrrrrrrrrrrrrrrrr@{}}
\toprule
 & oracle & \multicolumn{4}{c}{$k{=}0$} & \multicolumn{4}{c}{$k{=}1$} & \multicolumn{4}{c}{$k{=}3$} & \multicolumn{4}{c}{$D{=}4$ full} & seeds \\
 & & S & P & E & E$_b$ & S & P & E & E$_b$ & S & P & E & E$_b$ & S & P & E & E$_b$ & \\
\midrule
PLN 500 & 2.93 & 2.96 & 2.96 & 3.10 & 7.41 & 2.96 & 2.97 & 3.09 & 6.56 & 2.96 & 2.97 & 3.06 & 10.03 & 2.97 & 2.97 & 3.10 & 10.29 & 10 \\
PLN 2k & 2.94 & 2.95 & 2.96 & 3.08 & 7.74 & 2.96 & 2.96 & 3.08 & 7.91 & 2.96 & 2.96 & 3.08 & 11.83 & 2.96 & 2.96 & 3.06 & 12.55 & 10 \\
NB 500 & 2.40 & 2.46 & 2.46 & 2.74 & 5.43 & 2.46 & 2.46 & 2.71 & 5.19 & 2.46 & 2.45 & 2.69 & 7.47 & 2.46 & 2.45 & 2.79 & 7.94 & 10 \\
NB 2k & 2.41 & 2.45 & 2.45 & 2.73 & 5.89 & 2.45 & 2.45 & 2.70 & 5.93 & 2.45 & 2.45 & 2.69 & 8.64 & 2.45 & 2.45 & 2.79 & 9.66 & 10 \\
Poi 500 (closure) & 2.14 & 2.18 & 2.18 & 2.18 & 2.42 & 2.18 & 2.19 & 2.18 & 2.25 & 2.18 & 2.24 & 2.18 & 2.19 & 2.18 & 2.24 & 2.19 & 2.20 & 10 \\
Poi 2k (closure) & 2.14 & 2.16 & 2.16 & 2.16 & 2.21 & 2.16 & 2.16 & 2.16 & 2.17 & 2.16 & 2.16 & 2.15 & 2.15 & 2.16 & 2.17 & 2.15 & 2.16 & 10 \\
bike & -- & 4.60 & 4.61 & 5.33 & 5.77 & 4.63 & 4.60 & 5.35 & 5.81 & 4.63 & 4.58 & 5.31 & 5.85 & 4.64 & 4.58 & 5.13 & 6.06 & 10 \\
bike, ordered split & -- & 5.90 & 5.91 & 6.57 & 25.30 & 5.45 & 5.92 & 6.59 & 31.35 & 5.72 & 5.57 & 6.45 & 23.64 & 5.68 & 5.65 & 6.32 & 22.49 & 10 \\
fluview & -- & 8.39 & 8.41 & 8.44 & 8.85 & 8.39 & 8.46 & 8.51 & 8.93 & 8.36 & 8.33 & 8.52 & 8.96 & 8.46 & 8.31 & 8.35 & 8.92 & 10 \\
fluview, ordered split & -- & 11.61 & 11.37 & 16.11 & 67.26 & 12.00 & 10.14 & 18.07 & 37.96 & 11.52 & 10.21 & 12.74 & 51.42 & 11.34 & 10.18 & 11.99 & 51.22 & 10 \\
\bottomrule
\end{tabular}
\end{table}